\documentclass[10pt,journal,compsoc]{IEEEtran}
\ifCLASSOPTIONcompsoc
  \usepackage[nocompress]{cite}
\else
  \usepackage{cite}
\fi
\usepackage{epsfig}
\usepackage{graphicx}
\usepackage{amsmath}
\usepackage{amssymb}
\usepackage{amsthm}
\usepackage{subfigure}
\usepackage{multirow}
\usepackage{color}
\usepackage{tabu}
\usepackage{array}
\usepackage[ruled,vlined]{algorithm2e}
\usepackage{threeparttable}

\usepackage[english]{babel}

\newtheorem{remark}{Remark}
\newtheorem{definition}{Definition}
\newtheorem{theorem}{Theorem}

\newcommand{\vecb}[1]{\mbox{\boldmath{$#1$}}}
\newcommand{\matb}[1]{\mbox{\boldmath{$#1$}}}

\ifCLASSINFOpdf
\else
\fi
\begin{document}
\title{Beyond Low Rank: A Data-Adaptive Tensor Completion Method}
%
%
\author{Lei~Zhang,~\IEEEmembership{Student Member,~IEEE,}
        Wei~Wei,~\IEEEmembership{Member,~IEEE,}
        Qinfeng~Shi,
        Chunhua~Shen,
        Anton~van~den~Hengel,
        and Yanning~Zhang,~\IEEEmembership{Senior Member,~IEEE,}
        
}

\IEEEtitleabstractindextext{%
\begin{abstract}
Low rank tensor representation underpins much of recent progress in tensor completion. In real applications, however, this approach is confronted with two challenging problems, namely (1) tensor rank determination; (2) handling real tensor data which only approximately fulfils the low-rank requirement. To address these two issues, we develop a data-adaptive tensor completion model which explicitly represents both the low-rank and non-low-rank structures in a latent tensor. Representing the non-low-rank structure separately from the low-rank one allows priors which capture the important distinctions between the two, thus enabling more accurate modelling, and ultimately, completion. Through defining a new tensor rank, we develop a sparsity induced prior for the low-rank structure, with which the tensor rank can be automatically determined. The prior for the non-low-rank structure is established based on a mixture of Gaussians which is shown to be flexible enough, and powerful enough, to inform the completion process for a variety of real tensor data. With these two priors, we develop a Bayesian minimum mean squared error estimate (MMSE) framework for inference which provides the posterior mean of missing entries as well as their uncertainty. Compared with the state-of-the-art methods in various applications, the proposed model produces more accurate completion results.
\end{abstract}


}

\maketitle

\IEEEdisplaynontitleabstractindextext

%
\IEEEpeerreviewmaketitle

\IEEEraisesectionheading{\section{Introduction}\label{sec:introduction}}
\IEEEPARstart{R}{epresentation} of multi-dimensional data becomes increasingly crucial in various computer vision tasks, especially that involving videos, hyperspectral images and deep neural networks, etc. Tensors provide a concise and effective way to represent them without loss of structural characteristics, e.g., a color video can be viewed as a 4-mode tensor with dimensionality {\emph{height $\times$ width $\times$ channel $\times$ time}}. In practice, a wide range of application domains (e.g., social networks~\cite{papalexakis2014spotting}, recommender systems~\cite{yao2015context}, computer vision~\cite{guo2015generalized}, etc.), however, often produce incomplete tensor data where partial entires are missing, e.g., incomplete social relations, unknown user-iterm correlation or corrupted videos, etc. Missing entires often cause the performance of related applications to drop greatly, particularly when missing ratio is high. Thus, plenty of efforts~\cite{liu2013tensor,chen2014simultaneous,zhao2015bayesian} have been made to estimate the missing entires on the basis of the available ones by exploiting the intrinsic structural relations within the tensor~\cite{liu2013tensor}. This is often termed tensor completion.

\begin{figure}
\setlength{\abovecaptionskip}{0pt}
\begin{center}
\includegraphics[height=0.9in,width=0.9in,angle=0]{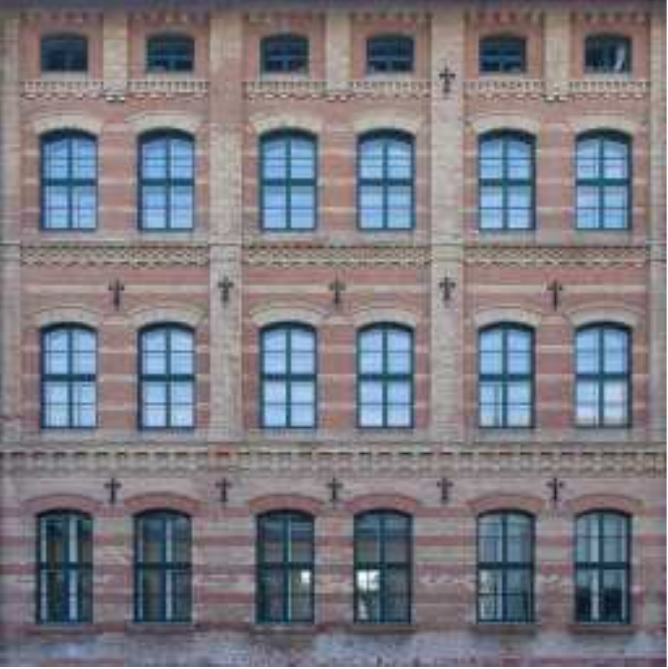}
\hspace{-0.15cm}
\includegraphics[height=0.9in,width=0.9in,angle=0]{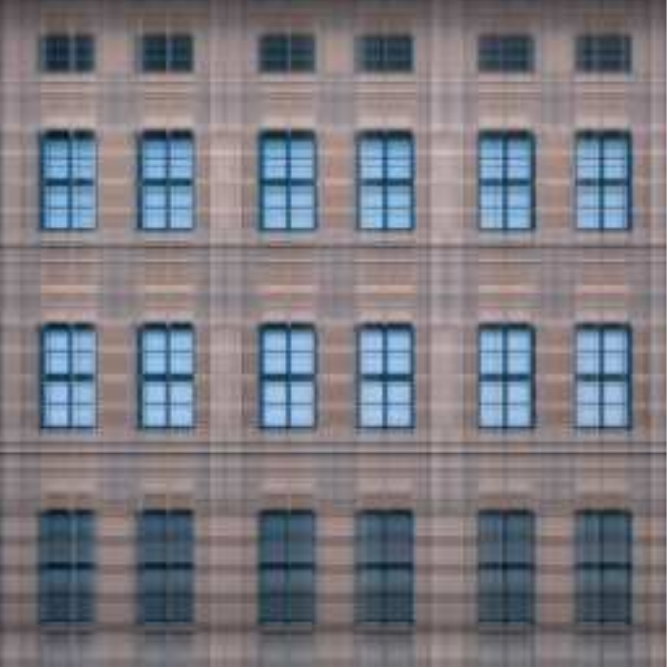}
\includegraphics[height=0.9in,width=1.4in,angle=0]{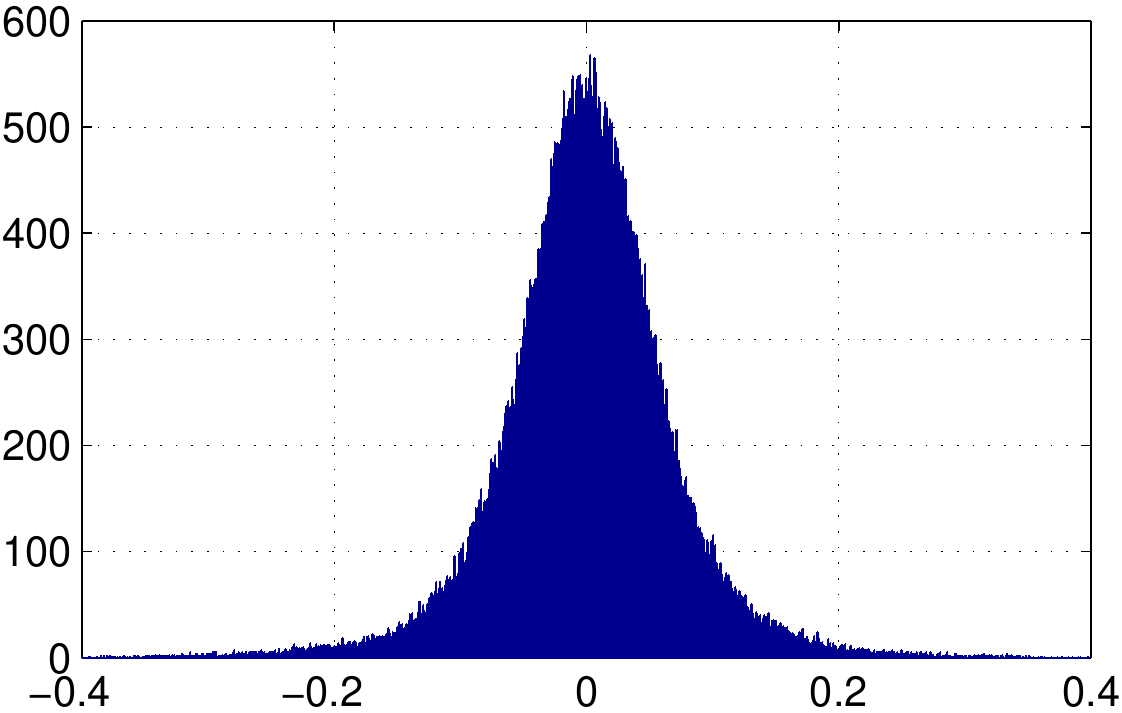}\\
\vspace{-0.15cm}
\subfigure[]{\includegraphics[height=0.9in,width=0.9in,angle=0]{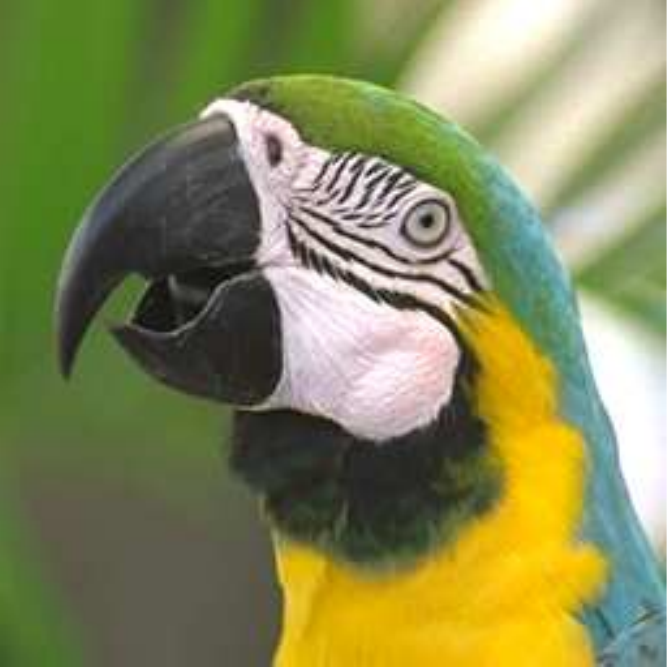}}
\hspace{-0.15cm}
\subfigure[]{\includegraphics[height=0.9in,width=0.9in,angle=0]{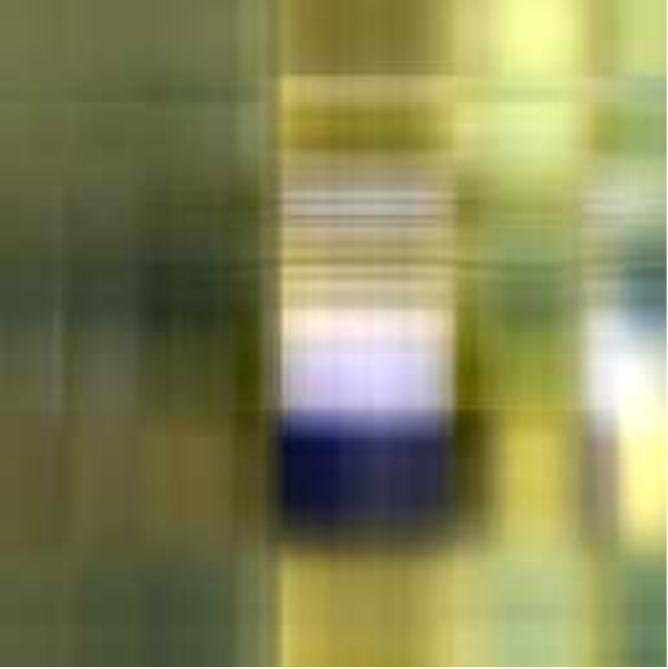}}
\subfigure[]{\includegraphics[height=0.9in,width=1.4in,angle=0]{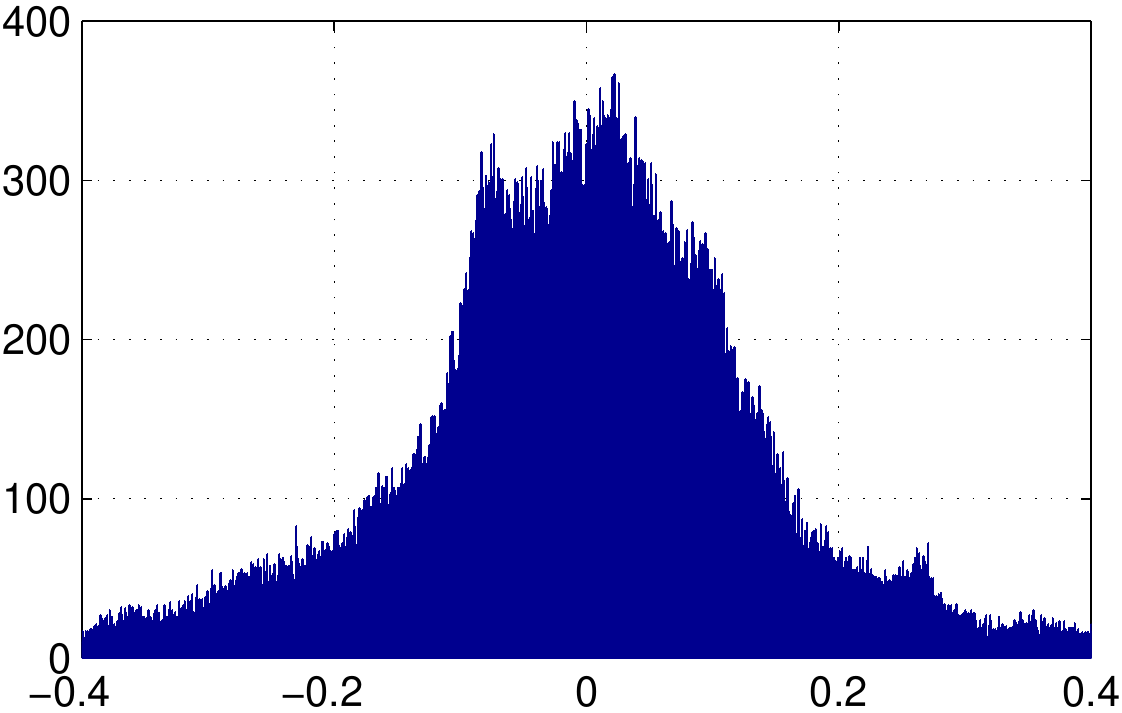}}
\end{center}
\vspace{-0.3cm}
\caption{Low-rank and non-low-rank structures in two color images (top: 'facade', bottom: 'parrot'). (a) Images. (b) Low-rank structure obtained by employing the CANDECOMP/PARAFAC (CP) factorization~\cite{TTB_Software} on each image with {\rm{rank}} = $3$. (c) Histograms of entries in the non-low-rank structure. For the highly structured 'facade', entries in the non-low-rank structure fit an approximately sparse distribution, while that in 'parrot' is more complex (e.g., heavy-tailed and multimodal).}
\vspace{-0.2cm}
\label{fig:dis}
\end{figure}

A promising way for tensor completion is to adopt the low rank representation model~\cite{gu2014robust}, which assumes the latent tensor to be of low-rank and thus recovers the missing entries by exploiting the low rank structure. Specifically, for a latent $K$-mode tensor $\mathcal{L} \in {\mathbb{R}^{n_1\times\cdots\times n_K}}$, $y_{\mbox{\boldmath{$i$}}}$ represents the observation of the ${\mbox{\boldmath{$i$}}}=(i_1,...,i_K)$-th entry of $\mathcal{L}$. Let $\mathcal{Y}_{\Omega} = \{y_{\mbox{\boldmath{$i$}}}\}_{{\mbox{\boldmath{$i$}}} \in \Omega}$ be the set of all such observations, where $\Omega$ collects the indices of all observations. The observation model can be formulated as
\begin{equation}\label{eq:eq2}
\begin{aligned}
\mathcal{Y}_{\Omega} = \mathcal{L}_{\Omega} + \mathcal{M}_{\Omega} 
\end{aligned}
\end{equation}
where $ \mathcal{M}_{\Omega}$ is the noise corruption. In the low rank representation model, $\mathcal{L}$ can be recovered via the maximum a posteriori (MAP) estimate as
\begin{equation}\label{eq:eq0}
\begin{aligned}
\hat{\mathcal{L}} = \arg\max\limits_{\mathcal{L}} p(\mathcal{L}|\mathcal{Y}_{\Omega})\propto p(\mathcal{Y}_{\Omega}|\mathcal{L})p(\mathcal{L}|\Theta)
\end{aligned}
\end{equation}
where $p(\mathcal{Y}_{\Omega}|\mathcal{L})$ denotes the likelihood induced by Eq.~\eqref{eq:eq2} and $p(\mathcal{L}|\Theta)$ represents a $\Theta$-parametrized low-rank prior. Various low-rank priors have been proposed in~\cite{liu2013tensor,chen2014simultaneous,zhang2014novel,zhao2015novel,zhao2016bayesian}. A brief review can be found in Section~\ref{sec: relatedwork}. Two challenges remain when applying this model to real data, however. 
{\emph{(1) Tensor rank determination.}}
Similar as matrix case, the key for low rank representation model is to determine the unknown tensor rank. In most previous works~\cite{liu2013tensor,chen2014simultaneous,zhang2014novel}, tensor rank is determined by unfolding the tensor into a collection of matrices and then minimizing matrix rank norms or set manually. However, matrix rank norms cannot capture the multi-dimensional structure of the latent tensor, thus misleading the rank determination. Moreover, the resulting erroneous rank estimate can causes over-fitting in tensor completion~\cite{zhao2016bayesian}.
{\emph{(2) Applicability to real data.}}
In practical, real tensors only ever approximately comply with the low-rank requirement in most cases, viz., $\mathcal{L}$ contains low-rank structure as well as the non-low-rank structure, shown as Fig.~\ref{fig:dis}. Inspired by this, $\mathcal{L}$ can be factorized as
\begin{equation}\label{eq:eq1}
\begin{aligned}
 \mathcal{L} = \mathcal{X} + \mathcal{E}.
\end{aligned}
\end{equation}
where $\mathcal{X}$ and $\mathcal{E}$ represents the low-rank and non-low-rank structures, respectively. Note that $\mathcal{E}$ here is not the observation noise which is captured by $\mathcal{M}$ in Equation~\eqref{eq:eq2}. Most previous methods assume $\mathcal{L}$ to be of low-rank in its entirety~\cite{liu2013tensor,chen2014simultaneous,zhang2014novel,goldfarb2014robust,gu2014robust,zhao2016bayesian}. They thus implicitly assume $\mathcal{E}$ to be zero~\cite{liu2013tensor,chen2014simultaneous,zhang2014novel}, or explicitly model a sparse $\mathcal{E}$~\cite{goldfarb2014robust,gu2014robust,zhao2016bayesian}. This assumption is explained as resulting from highly structured data (e.g., 'facade' in Figure~\ref{fig:dis} with extensive repeated textures), but it neglects the fact that vanishingly few real data actually exhibit the regularity of structure to support such a model, i.e., most real data is beyond the low rank assumption and shows a more complex $\mathcal{E}$, e.g., 'parrot' in Figure~\ref{fig:dis}. \emph{Accurate tensor completion requires explicit modelling of all aspects of the latent $\mathcal{L}$ which do, or do not conform the low-rank prior.} Therefore, it is crucial to model the complex $\mathcal{E}$ in the latent tensor representation.

To jointly address these two problems above, we present a novel data-adaptive tensor completion model. Firstly, we define a new tensor rank based on the CANDECOMP/PARAFAC (CP) factorization, with which the low-rank structure $\mathcal{X}$ is depicted by a sparsity induced low-rank prior model. Through exploiting the sparsity in the factorization weights probabilistically, the proposed model is able to automatically determine the tensor rank. Then, we model the non-low-rank structure $\mathcal{E}$ with a mixture of Gaussians (MOG). The powerful ability of MOG to fit a wide range of $\mathcal{E}$s (e.g., zero, sparse or mixture one) enables the proposed model to be adaptive to a variety of real tensor data, which is even beyond the low-rank assumption with complex non-low-rank structure. Both merits are both beneficial for robustly modelling the latent tensor. To harness them in a principled way, we adopt the Bayesian minimum mean squared error estimate (MMSE) framework for inference over the proposed model. In contrast to most previous works which only produce a point estimation on each missing entry with the MAP estimate, we infer the posterior mean of missing entries as well as their uncertainty with Gibbs samplers. Experimental results on synthetic tensor data sufficiently demonstrate the capacity of the proposed model in tensor rank determination as well as fitting various non-low-rank structures, the effectiveness of each ingredient, convergence and recovery performance. Additionally, applications in image inpainting, video completion and facial image synthesis, are conducted to further validate the superiority of the proposed model over other state-of-the-art tensor completion methods in terms of recovery accuracy.

In summary, this study mainly contributes in the following five aspects: 1) different from previous methods that only focuses on the low-rank structure, we propose a general tensor completion model which can recover both the low-rank and non-low-rank structures within tensor data; 2) the proposed model is able to automatically determine the unknown tensor rank, even with high missing ratios; 3) the proposed model can data-adaptively fit the complex non-low-rank structure in the latent tensor; 4) spatial coherence is considered for visual tensor data to further improve the recovery accuracy; 5) we present state-of-the-art tensor completion results in various real-world applications.

\section{Related work}\label{sec: relatedwork}

{\textbf{Tensor rank determination.}} According to the rank determination schemes, existing low rank tensor representation models can be roughly divided into two categories. {\emph{1) Completion models:} These models minimize the tensor rank by unfolding a tensor into a collection of matrices and then solving a convex optimization on those matirces with matrix rank norms. For example, Liu et al.~\cite{liu2013tensor} propose to minimize the trace norm of a tensor, which is defined as the summation of the nuclear norm on the unfolding matrix along each mode of the tensor. Zhao et al.~\cite{zhao2015novel} define a new tensor rank norm as the product of rank norms on all unfolding matrices of the tensor. Xu et al.~\cite{xu2013parallel} turn to factorize each unfolding matrix optimally. These models intrinsically exploit the low-rank structure in the unfolding matrices, which, however, cannot fully represent the multi-dimensional structure of tensors~\cite{zhao2015bayesian}. {\emph{2) Factorization models:} They decompose the latent tensor into multiple factors with a fixed rank which is often set manually, and then infer those factors intead. In~\cite{chen2014simultaneous, xu2012infinite}, various factor priors are proposed to regularize the Tucker factorization of tensors. A novel weighted CP factorization scheme is proposed in~\cite{chen2016robust}. However, it is difficult to set correct tensor rank by chance. Moreover, incorrect tensor rank can induce over-fitting~\cite{zhao2015bayesian}. In this study, we propose a new tensor rank definition based on the CP factorization, and the tensor rank can be automatically determined by exploiting the sparsity in factorization weights. Although a similar idea is also adopted in~\cite{zhao2015bayesian,zhao2016bayesian}, the proposed model is more general, flexible and powerful. The detailed comparison can be found in Section~\ref{subsubsection:sparsitylowrank}.

{\textbf{Applicability to real data.}} In most cases, real tensor data $\mathcal{L}$ only ever approximately comply with the low-rank assumption and contains both the low-rank structure $\mathcal{X}$ and the non-low-rank one $\mathcal{E}$. Neverthless, most previous works adopt the low rank tensor representation model~\cite{liu2013tensor,chen2014simultaneous,zhang2014novel} and totally neglect the non-low-rank $\mathcal{E}$ (i.e., $\mathcal{E} = 0$). In~\cite{goldfarb2014robust,gu2014robust,zhao2016bayesian}, $\mathcal{E}$ is depicted by a sparse model. However, $\mathcal{E}$ can be complicated (e.g., heavy-tailed or multimodal) in real tensor data, which necessitates a flexible model to describe $\mathcal{E}$. MOG has shown powerful ability to represent the complex distributed data in various applications, e.g., denoising~\cite{zhao2014robust,chen2016robust}, deblur~\cite{schmidt2014shrinkage} and image representation~\cite{liu2014encoding}, etc. Inspired by this, we leverage MOG to fit the complex non-low-rank structure $\mathcal{E}$ as part of the latent tensor $\mathcal{L}$ for a general tensor completion model. Although MOG has been utilized in several low rank models~\cite{zhao2014robust,cao2015low,chen2016robust}, they are clearly different from this work. In~\cite{zhao2014robust,cao2015low,chen2016robust}, the latent tensor $\mathcal{L}$ is assumed to be of low-rank in its entirety, and MOG is utilized to model the mixed noise corruption (e.g., $\mathcal{M}$ in Eq.~\eqref{eq:eq2}), which is expected to be eliminated from the latent tensor $\mathcal{L}$. In contrast, we assume $\mathcal{L}$ is beyond the low-rank assumption and leverage MOG to fit the complex non-low-rank structure $\mathcal{E}$ belonging to $\mathcal{L}$. In addition, the MOG is inferred from a fully-observed matrix for robust principle analysis in ~\cite{zhao2014robust,cao2015low}, while we infer it from a handful of observed entries for tensor completion. In~\cite{chen2016robust}, EM algorithm is utilized to give a point estimation for parameters of the MOG, while we follow the MMSE framework to infer the posterior mean of these parameters with Gibbs samplers.

It is noticeable that, to our knowledge, this study is the first attempt to jointly address the automatic tensor rank determination and modelling the complex non-low-rank structure in tensor completion.


\section{Proposed tensor completion model}
Given the observation model as Eq.~\eqref{eq:eq2}, we assume that entries in $\mathcal{Y}_{\Omega}$ are independent and identically distributed and $\mathcal{M}_{\Omega}$ is the Gaussian white noise with precision $\tau_0$. Thus, we have the following likelihood
\begin{equation}\label{eq:eq3}
\begin{aligned}
p(\mathcal{Y}_{\Omega}|\mathcal{X}, \mathcal{E}) = \prod\nolimits_{\mbox{\boldmath{$i$}} } \mathcal{N}(y_{\mbox{\boldmath{$i$}}} | x_{\mbox{\boldmath{$i$}}} + e_{\mbox{\boldmath{$i$}}}, \tau^{-1}_0)^{o_{\mbox{\boldmath{$i$}}}}
\end{aligned}
\end{equation}
where $\mathcal{O}$ is an indicator tensor with entries $o_{\mbox{\boldmath{$i$}}} = 1$ if ${\mbox{\boldmath{$i$}}} \in \Omega$. $x_{\mbox{\boldmath{$i$}}}$ and $e_{\mbox{\boldmath{$i$}}}$ are the entries in $\mathcal{X}$ and $\mathcal{E}$. In this study, we propose to recover the latent $\mathcal{L}$ from $\mathcal{Y}_{\Omega}$ by exploiting the low-rank structure $\mathcal{X}$ as well as fitting the complex non-low-rank structure $\mathcal{E}$. To this end, we specifically design priors for $\mathcal{X}$ and $\mathcal{E}$ as follows.

\subsection{Low-rank structure modelling}
In this study, we define a new tensor rank based on CP factorization, with which a sparsity induced low-rank model is proposed to represent $\mathcal{X}$ and the corresponding tensor rank of $\mathcal{X}$ then can be automatically determined by exploiting the sparsity in the CP factorization weights.

\subsubsection{CP factorization}
In the CP factorization, tensor $\mathcal{X}$ is factorized as a sum of $R$ rank-one tensors as
\begin{equation}\label{eq:eq7}
\begin{aligned}
\mathcal{X} = \sum\nolimits^{R}_{r = 1} \lambda_r\mbox{\boldmath{$u$}}^{(1)}_r \circ \cdots \circ \mbox{\boldmath{$u$}}^{(K)}_r = [\kern-0.17em[ \mbox{\boldmath{$\lambda$}}; {\mbox{\boldmath{$U$}}}^{(1)},..., {\mbox{\boldmath{$U$}}}^{(K)} ]\kern-0.17em],
\end{aligned}
\end{equation}
where $\mbox{\boldmath{$u$}}^{(k)}_r \in \mathbb{R}^{n_k}$ is the factor vector in $k$-th mode and $k=1,...,K$. $\circ$ denotes the outer product. For simplicity, the CP factorization can be concisely represented as the right part of Eq.~\eqref{eq:eq7}. $\mbox{\boldmath{$\lambda$}} = [\lambda_1,..., \lambda_R]^T$ is the weights vector. ${\mbox{\boldmath{$U$}}}^{(k)} = [\mbox{\boldmath{$u$}}^{(k)}_1,...,\mbox{\boldmath{$u$}}^{(k)}_R] \in \mathbb{R}^{n_k \times R}$ denotes the $k$-th factor matrix.

\subsubsection{CP rank vs. sparsity induced rank}
\begin{definition}[CP rank]~\label{def:def1}
The rank of tensor $\mathcal{X}$, denoted by $\rm{rank}(\mathcal{X})$, is defined as the smallest number of rank-one tensors in the CP factorization of $\mathcal{X}$~\cite{kolda2009tensor}.
\end{definition}
CP rank is a specialized tensor rank. It degenerates to the matrix rank when $K=2$. However, since CP factorization is ill-posed~\cite{kolda2009tensor}, the determination of CP rank is NP-hard. 

To address this problem, we propose a new tensor rank in this study. Specifically, given a tensor $\mathcal{X}$ of ${\rm{rank}}(\mathcal{X})=R_0$, suppose we have all its CP factorizations with $R \gg R_0$, and we will be able to find one with the most sparse weight vector $\mbox{\boldmath{$\lambda$}}$. According to Definition~\ref{def:def1}, only $R_0$ weights in ${\mbox{\boldmath{$\lambda$}}}$ will be non-zero, viz., ${\rm{rank}}(\mathcal{X})=\|{\mbox{\boldmath{$\lambda$}}}\|_0$. Inspired by this, we give a sparsity induced tensor rank definition as
\begin{definition}[Sparsity induced rank]~\label{def:def2}
Given the CP factorization of $\mathcal{X}$ with $R \gg {\rm{rank}}(\mathcal{X})$ and the most sparse weight vector $\mbox{\boldmath{$\lambda$}}$, $\rm{rank}(\mathcal{X}) = \|\mbox{\boldmath{$\lambda$}}\|_0$.
\end{definition}
Apparently, there are two problems for Definition~\ref{def:def2}, including 1) how to seek the CP factorization with the most sparse weight vector $\vecb{\lambda}$; 2) whether the determination of the sparsity induced rank is NP-hard as CP rank or not.

For an ill-posed problem (e.g., CP factorization), it has been extensively demonstrated that an appropriate prior is necessitated to locate the unique solution expected~\cite{zhang2016exploring}. According to Definition~\ref{def:def2}, we can impose a sparsity prior on the weight vector $\vecb{\lambda}$ during CP factorization. This casts the problem into a sparse representation based optimization where the expected CP factorization can be reached by exploiting the sparsity in $\vecb{\lambda}$. Although exploiting the sparsity (e.g., minimizing $\ell_0$ norm) is also NP-hard, many techniques (e.g., $\ell_1$ norm) have been proposed to produce a satisfactory solution via solving a convex optimization problem. Therefore, determining the sparsity induced rank in Definition~\ref{def:def2} is not NP-hard as the CP rank. In the following part, we will discuss how to build a low rank prior model for $\mathcal{X}$ based on Definition~\ref{def:def2}.

\subsubsection{Sparsity induced low rank model}\label{subsubsection:sparsitylowrank}
Definition~\ref{def:def2} indicates that tensor rank can be automatically determined by exploiting the sparsity of $\mbox{\boldmath{$\lambda$}}$. Thus, we propose to model the low-rank structure $\mathcal{X}$ by modelling the sparsity of weight vector $\vecb{\lambda}$ in its CP factorization. Such a low rank model is termed {\emph{sparsity induced low rank model}} in this study. With a small tensor rank (i.e., $R \gg {\rm{rank}}(\mathcal{X})$), this model amounts to representing $\mathcal{X}$ on a tensor dictionary with sparse coefficients $\mbox{\boldmath{$\lambda$}}$, where each dictionary atom is a rank-one tensor as $\mbox{\boldmath{$u$}}^{(1)}_r \circ \cdots \circ \mbox{\boldmath{$u$}}^{(K)}_r$ in Eq.~\ref{eq:eq7}. These atoms can well preserve the multi-dimensional structure of tensors, viz., the new tensor rank intrinsically depends on the multi-dimensional structure of tensors, which is totally different from the rank norms defined on unfolding matrices of tensor~\cite{liu2013tensor,zhang2014novel}, since only the two-dimensional structure is depicted in unfolding matrices. 

Although a similar sparsity based tensor rank is also reported in~\cite{zhao2015bayesian,zhao2016bayesian}, it is essentially different from ours in three aspects. (i) In~\cite{zhao2015bayesian,zhao2016bayesian}, factorization weights $\mbox{\boldmath{$\lambda$}}$ is absorbed into the factor matrices $\mbox{\boldmath{$U$}}^{(k)}$s (i.e., $\mbox{\boldmath{$\lambda$}}$ and $\mbox{\boldmath{$U$}}^{(k)}$s are coupled together), and thus the tensor rank depends on the row sparsity of $\mbox{\boldmath{$U$}}^{(k)}$s. While we model them separately to analyze the effect of various factors on the tensor more flexibly. (ii) All $\mbox{\boldmath{$U$}}^{(k)}$s follow the same distribution to give a consistent rank in~\cite{zhao2015bayesian,zhao2016bayesian}, while we adopt various distributions to separately model each $\mbox{\boldmath{$U$}}^{(k)}$ (see Section~\ref{subsubsection:regularization}). (iii) The sparsity is depicted by a hierarchical student-t distribution in~\cite{zhao2015bayesian,zhao2016bayesian}, while we employ a more powerful reweighed Laplace prior (see Remark~\ref{remark:remark2}), which determines the tensor rank more accurately (see Section~\ref{subsection:syntetic}).

To model the sparsity of $\mbox{\boldmath{$\lambda$}}$, we adopt a two-level reweighted Laplace prior~\cite{zhang2016exploring} as
\begin{equation}\label{eq:eq9}
\begin{aligned}
&\mbox{\boldmath{$\lambda$}} \sim \mathcal{N}(\mbox{\boldmath{$\lambda$}}|\mathbf{0},\mathbf{diag}(\mbox{\boldmath{$\gamma$}})), {\kern 2pt}
\mbox{\boldmath{$\gamma$}} \sim \prod\nolimits^R_{r=1} {\rm{Ga}}(\gamma_r|1,\frac{\kappa_r}{2}) 
\end{aligned}
\end{equation}
where $\mbox{\boldmath{$\lambda$}}$ and $\mbox{\boldmath{$\gamma$}}=[\gamma_1,...,\gamma_R]^T$ follow a zero-mean Gaussian distribution and a product of Gamma distributions, respectively. $\gamma_r$ and $\kappa_r$ are the respective $r$-th entries of $\mbox{\boldmath{$\gamma$}}$ and $\mbox{\boldmath{$\kappa$}} = [\kappa_1,..., \kappa_R]^T$. The motivation here is twofold. First, the reweighed Laplace prior performs better than traditional sparsity priors (e.g., Laplace prior) on preserving the significant structure in tensor when applied to the sparsity based low rank model (see Remark~\ref{remark:remark2}). Second, the two-level structure makes the inference in Section~\ref{section:inference} tractable.

\begin{remark}\label{remark:remark2}
Sparsity induced low rank model is capable of preserving the significant structure in tensor.
\end{remark}

When exploiting the sparsity of $\mbox{\boldmath{$\lambda$}}$ with appropriate priors (e.g., Laplace prior or $\ell_1$ norm),  weights in $\mbox{\boldmath{$\lambda$}}$ are intrinsically shrunk by a soft-thresholding operator~\cite{gu2014weighted}. In the CP factorization, large weights in $\mbox{\boldmath{$\lambda$}}$ are associated with the significant structure in the tensor, which is similar as large singular values associated with the significant structure in a matrix~\cite{gu2014weighted}. Thus, larger weights should be shrunk less than those smaller ones to preserve the significant structure during exploiting sparsity. However, when the sparsity of $\mbox{\boldmath{$\lambda$}}$ is depicted by Laplace prior or $\ell_1$ norm as
 \begin{equation}\label{eq:eq10}
\begin{aligned}
p(\mbox{\boldmath{$\lambda$}}|w) \propto \exp\left(-w\|\mbox{\boldmath{$\lambda$}}\|_1\right),
\end{aligned}
\end{equation}
all weights in $\mbox{\boldmath{$\lambda$}}$ are shrunk with the same amount $w$~\cite{gu2014weighted}. In contrast, the two-level reweighed Laplace prior amounts to~\cite{zhang2016exploring}
 \begin{equation}\label{eq:eq11}
\begin{aligned}
p(\mbox{\boldmath{$\lambda$}}|\mbox{\boldmath{$K$}}) \propto \exp\left(-\|\mbox{\boldmath{$K$}}\mbox{\boldmath{$\lambda$}}\|_1\right),
\end{aligned}
\end{equation}
where $\mbox{\boldmath{$K$}} = \mathbf{diag}\left([\sqrt{\kappa_1},...,\sqrt{\kappa_R}]^T\right)$. Weights in $\mbox{\boldmath{$\lambda$}}$ are shrunk with various amount $\sqrt{\kappa_r}$s. Moreover, it will be proved in Theorem~\ref{theorem:theorem3} that each $\kappa_r$ is an decreasing function over $|\lambda_r|$ in the Bayesian inference, which is similar as~\cite{zhang2015reweighted}. Therefore, larger weights in $\mbox{\boldmath{$\lambda$}}$ are shrunk less than smaller ones. In other words, the proposed sparsity induced low rank model is capable of adaptively preserving the significant structure in the tensor during rank reduction.

\subsubsection{Regularized factor matrix}\label{subsubsection:regularization}
In this study, the weight vector $\vecb{\lambda}$ and factor matrices $\matb{U}^{(k)}$s from the CP factorization of $\mathcal{X}$ are modelled separately. Although $\vecb{\lambda}$ has been well regularized as Eq.~\eqref{eq:eq9}, there still exists infinite solutions for the CP factorization of $\mathcal{X}$, e.g., $\mathcal{X} = [\kern-0.17em[ \mbox{\boldmath{$\lambda$}}; c^{-1}{\mbox{\boldmath{$U$}}}^{(1)}, {\mbox{\boldmath{$U$}}}^{(2)},..., c{\mbox{\boldmath{$U$}}}^{(K)} ]\kern-0.17em]$ with any scalar $c \neq 0$. To address this problem, we further assume each entry $u^{(k)}_{ir}$ of the factor matrix $\mbox{\boldmath{$U$}}^{(k)}$ follows a Gaussian distribution independently and identically as
\begin{equation}\label{eq:eq12}
\begin{aligned}
u^{(k)}_{ir} \sim \mathcal{N}(u^{(k)}_{ir}|\mu^{(k)}, \tau^{(k)-1}).
\end{aligned}
\end{equation}
which amounts to regularizing each entry with the $\ell_2$ norm. For each $\mbox{\boldmath{$U$}}^{(k)}$, we adopt an individual Gaussian distribution to capture its characteristics. To complete the Bayesian model, we further introduce conjugate priors over the parameters of Gaussian distributions, $\mu^{(k)}$s and $\sigma^{(k)}$s as
\begin{equation}\label{eq:eq13}
\begin{aligned}
\mu^{(k)}, \tau^{(k)} &\sim \mathcal{N}(\mu^{(k)}|\mu_0, (\beta_0\tau^{(k)})^{-1}){\rm{Ga}}(\tau^{(k)}|a_0,b_0).
\end{aligned}
\end{equation}
where $\mu^{(k)}$ and $\tau^{(k)}$ jointly follow a Gaussian-gamma distribution parametrized by $\mu_0$, $\beta_0$, $a_0$ and $b_0$.

In summary, the low-rank structure $\mathcal{X}$ is modelled by exploiting the sparsity in the weight vector $\vecb{\lambda}$ and regularizing the factor matrices $\matb{U}^{(k)}$ in its CP factorization.

\begin{figure*}
\setlength{\abovecaptionskip}{0pt}
\begin{center}
\includegraphics[height=1.6in,width=6.3in,angle=0]{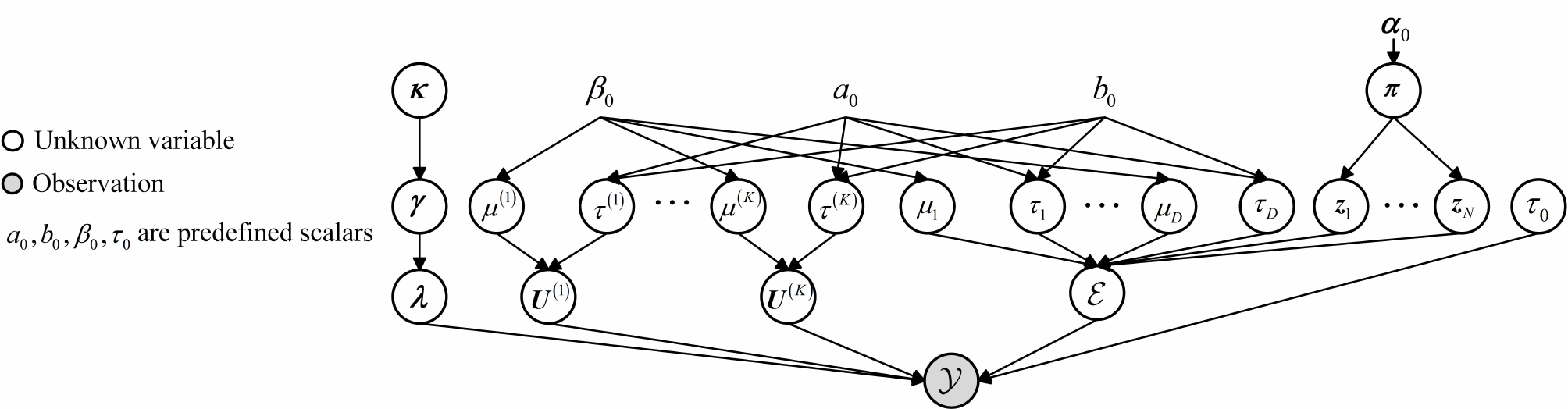}
\end{center}
\vspace{-0.1cm}
\caption{The probabilistic graphical structure of the proposed model.}
\vspace{-0.1cm}
\label{fig:PGM}
\end{figure*}

\subsection{Mixture of Gaussians for non-low-rank structure}
According to Fig.~\ref{fig:dis}, the distribution of entries in $\mathcal{E}$ can be complex (e.g., heavy-tailed or multimodal) in practical. To impose a suitable prior on the complex $\mathcal{E}$, we assume that each entry $e_{\mbox{\boldmath{$i$}}}$ comes from a mixture of $D$ Gaussian as
\vspace{-0.2cm}
\begin{equation}\label{eq:eq4}
\begin{aligned}
e_{\mbox{\boldmath{$i$}}}\sim \sum\nolimits^D_{d = 1} \pi_{d}\mathcal{N}(e_{\mbox{\boldmath{$i$}}}|\mu_d, \tau^{-1}_d),
\end{aligned}
\vspace{-0.2cm}
\end{equation}
where $\pi_d \geq 0$ is the mixing proportion with $\sum\nolimits^{D}_{d=1}\pi_d = 1$. $\mathcal{N}(e_{\mbox{\boldmath{$i$}}}|\mu_d, \tau^{-1}_d)$ denotes the $d$-th Gaussian component with the mean $\mu_d$ and the precision $\tau_d$. Introducing $D$ indicator variables $z^d_{\mbox{\boldmath{$i$}}}$s for $d=1,...,D$, Eq.~\eqref{eq:eq4} can be equivalently represented as a two-level generative model~\cite{bishop2006pattern} as
\begin{equation}\label{eq:eq5}
\begin{aligned}
e_{\mbox{\boldmath{$i$}}} &\sim \prod\nolimits^{D}_{d = 1}\mathcal{N}(e_{\mbox{\boldmath{$i$}}}|\mu_d, \tau^{-1}_d)^{z^{d}_{\mbox{\boldmath{$i$}}}},{\kern 2pt}
\mbox{\boldmath{$z$}}_{\mbox{\boldmath{$i$}}}  \sim {\rm{Multinomial}}(\mbox{\boldmath{$z$}}_{\mbox{\boldmath{$i$}}} |\mbox{\boldmath{$\pi$}}),
\end{aligned}
\end{equation}
where $\mbox{\boldmath{$z$}}_{\mbox{\boldmath{$i$}}} = (z^{1}_{\mbox{\boldmath{$i$}}},..., z^D_{\mbox{\boldmath{$i$}}}) \in \{0,1\}^D$ with $\sum\nolimits^D_{d=1}z^d_{\mbox{\boldmath{$i$}}} = 1$ follows a multinomial distribution parametrized by $\mbox{\boldmath{$\pi$}} = (\pi_1,..., \pi_D)$. To model $\mathcal{E}$ flexibly, we further impose conjugate priors over the parameters of $\mu_d$s, $\tau_d$s, and $\mbox{\boldmath{$\pi$}}$ as
\begin{equation}\label{eq:eq6}
\begin{aligned}
\mu_d, \tau_d &\sim \mathcal{N}(\mu_d|\mu_0, (\beta_0\tau_d)^{-1}){\rm{Ga}}(\tau_d|a_0,b_0)\\
{\mbox{\boldmath{$\pi$}}} &\sim {\rm{Dir}}({\mbox{\boldmath{$\pi$}}}|{\mbox{\boldmath{$\alpha$}}}_0),
\end{aligned}
\end{equation}
where ${\mbox{\boldmath{$\pi$}}} $ follows a Dirichlet distribution ${\rm{Dir}}({\mbox{\boldmath{$\pi$}}}|{\mbox{\boldmath{$\alpha$}}}_0)$ with parameter ${\mbox{\boldmath{$\alpha$}}}_0 = (\alpha_{01}, ..., \alpha_{0D})$. The mixture of Gaussians has obvious advantage on representing the complex $\mathcal{E}$ as
\vspace{-0.05cm}
\begin{remark}\label{remark:remark1}
Mixture of Gaussians is capable of fitting a wide range of non-low-rank structures $\mathcal{E}$s.
\end{remark}
It has been demonstrated that mixture of Gaussians is of the universal ability to approximate any continuous distributions~\cite{bishop2006pattern}. In addition, it has been proved~\cite{zhao2014robust} that both the Dirac delta distribution for zero variable (i.e., $\delta(0)$) and the spike-and-slab sparsity prior~\cite{ishwaran2005spike} are also special cases of the mixture of Gaussians. Therefore, mixture of Gaussians is able to fit a wide range of $\mathcal{E}$s (i.e., zero, sparse or more complex) in this study, which leads to a more robust representation for the real tensor data and ultimately improved completion performance (see more evidences in the experiments in Section~\ref{subsection:syntetic}).

\section{Inference}\label{section:inference}
According to the above likelihood and priors, we give the probabilistic graphical structure of the proposed model in Figure~\ref{fig:PGM}. Most previous works infer the latent tensor $\mathcal{L}$ with MAP estimate as Eq.~\eqref{eq:eq0}. However, generative models can often perform poorly in the context of MAP~\cite{schmidt2010generative}. In this study, we adopt the Bayesian minimum mean squared error estimate (MMSE) suggested in~\cite{schmidt2010generative} for tensor completion as
\begin{equation}\label{eq:eq14}
\vspace{-0.1cm}
\begin{aligned}
\hat{\mathcal{L}} = \arg \min\limits_{\tilde{\mathcal{L}} } \int \|\tilde{\mathcal{L}}  - \mathcal{L} \|^2_Fp(\mathcal{L} | \mathcal{Y}_{\Omega}) d \mathcal{L} = \mathbb{E}[\mathcal{L}|\mathcal{Y}_{\Omega}],
\end{aligned}
\vspace{-0.1cm}
\end{equation}
where $\hat{\mathcal{L}}$ equals the expectation $\mathbb{E}[\mathcal{L}|\mathcal{Y}_{\Omega}]$ of $\mathcal{L}$ under the posterior distribution $p(\mathcal{L} | \mathcal{Y}_{\Omega})$. $\|\mathcal{A}\|_F$ denotes the Frobenius norm on the tensor $\mathcal{A}$. In contrast to MAP that only produces a point estimation, MMSE exploits the uncertainty of $\mathcal{L}$ and adopts the probabilistic mean as solution which often shows better robustness and performance. However, it is difficult to conduct the expectation in Eq.~\eqref{eq:eq14}. To circumvent this problem, we can adopt the mean of samples drawn from the posterior of $\mathcal{L}$ as an unbiased estimate for $\hat{\mathcal{L}}$~\cite{schmidt2010generative}. Since we model the non-low-rank structure $\mathcal{E}$ and the CP factorization of the low-rank structure $\mathcal{X}$ instead of directly modelling $\mathcal{L}$ in this study, we turn to draw samples from the corresponding posteriors for $\vecb{\lambda}$,  $\matb{U}^{(k)}$ and $\mathcal{E}$. Given their sample means, $\hat{\mathcal{L}}$ then can be obtained from Eqs.~\eqref{eq:eq7}~\eqref{eq:eq1}.
 
Specifically, based on Eqs.~\eqref{eq:eq3}~\eqref{eq:eq9}~\eqref{eq:eq12}~\eqref{eq:eq13}~\eqref{eq:eq5}~\eqref{eq:eq6}, we can obtain the posterior of all involved variables as
\vspace{-0.1cm}
\begin{equation}\label{eq:eq15}
\begin{aligned}
p(\mbox{\boldmath{$\lambda$}}, \mathcal{U}, \mathcal{Z}, \mbox{\boldmath{$\pi$}}, \mbox{\boldmath{$\mu$}}, \mbox{\boldmath{$\tau$}}, \mbox{\boldmath{$\mu$}}_e, \mbox{\boldmath{$\tau$}}_e|\mathcal{Y}_{\Omega}),
\end{aligned}
\vspace{-0.1cm}
\end{equation}
where $\mathcal{U} = \{U^{(k)}\}$, $\mathcal{Z}= \{\mbox{\boldmath{$z$}}_{\mbox{\boldmath{$i$}}}\}$, $\mbox{\boldmath{$\mu$}} = \{\mu^{(k)}\}$, ${\mbox{\boldmath{$\tau$}}} = \{\tau^{(k)}\}$, $\mbox{\boldmath{$\mu$}}_e = \{\mu_d\}$ and ${\mbox{\boldmath{$\tau$}}}_e = \{\tau_d\}$ are introduced for simplicity. Then, Gibbs sampling is employed on Eq.~\eqref{eq:eq15} to sample each variable as follows.

\subsection{Gibbs samplers for low-rank structure}\label{subsec:gbs_low}
The low-rank structure $\mathcal{X}$ is decomposed into the weight vector $\vecb{\lambda}$ and factor matrices $\matb{U}^{(k)}$s.
\subsubsection{Sampler for weight vector $\mbox{\boldmath{$\lambda$}}$}
According to Figure~\ref{fig:PGM}, the posterior for $\vecb{\lambda}$ can be inferred through receiving two kinds of messages. The first one comes from its parent and is represented as the prior in Eq.~\eqref{eq:eq9}, while the latter one sent by the observation $\mathcal{Y}$ and its parents (i.e., factor matrices $\matb{U}^{(k)}$, non-low-rank structure $\mathcal{E}$ and $\tau_0$) is expressed as the likelihood in Eq.~\eqref{eq:eq3}. Thus, given $\matb{U}^{(k)}$s, $\mathcal{E}$, $\vecb{\gamma}$ and $\tau_0$, we have the posterior over $\mbox{\boldmath{$\lambda$}}$ as
\begin{equation}\label{eq:eq13-1}
\begin{aligned}
p(\mbox{\boldmath{$\lambda$}}|-) \propto \mathcal{N}(\mbox{\boldmath{$\lambda$}}|\mathbf{0},\Gamma)\prod\limits_{\mbox{\boldmath{$i$}}} \mathcal{N}(y_{\mbox{\boldmath{$i$}}} | e_{\mbox{\boldmath{$i$}}} + \sum\limits^{R}_{r=1}\lambda_r\prod\limits^{K}_{k=1}u^{(k)}_{i_kr}, \tau^{-1}_0)^{o_{\mbox{\boldmath{$i$}}}}.
\end{aligned}
\end{equation}
Since all entries in $\vecb{\lambda}$ are dependent on each other, we have to sample each entry alternatively. To this end, let $\tilde{y}^r_{\mbox{\boldmath{$i$}}} = y_{\mbox{\boldmath{$i$}}} - e_{\mbox{\boldmath{$i$}}} - \sum\nolimits^{R}_{t \neq r}\lambda_t\prod\nolimits^{K}_{k=1}u^{(k)}_{i_kt}$ and $\tilde{b}^r_{\mbox{\boldmath{$i$}}} = \prod\nolimits^{K}_{k=1}u^{(k)}_{i_kr}$, the posterior for entry $\lambda_r$ then can be given as
\begin{equation}\label{eq:eq14-1}
\begin{aligned}
p(\lambda_r|-) &\propto \exp\left(-\frac{\lambda^2_r}{2\gamma_r}\right)\prod\limits_{\mbox{\boldmath{$i$}}} \exp\left[-\frac{{o_{\mbox{\boldmath{$i$}}}}\left(\tilde{y}^r_{\mbox{\boldmath{$i$}}} - \lambda_r\tilde{b}^r_{\mbox{\boldmath{$i$}}}\right)^2}{2\tau^{-1}_0}\right].
\end{aligned}
\end{equation}
where the messages from other entries $\lambda_t$ with $t\neq r$ is further introduced. It can be shown that $\lambda_r$ can be drawn from a Gaussian distribution as
\begin{equation}\label{eq:eq15-1}
\begin{aligned}
\lambda_r \sim \mathcal{N}(\tilde{\mu}_{\lambda_r}, \tilde{\tau}^{-1}_{\lambda_r}),
\end{aligned}
\end{equation}
where the posterior parameters can be updated as
\begin{equation}\label{eq:eq16}
\begin{aligned}
\tilde{\tau}_{\lambda_r} = \gamma^{-1}_r + \tau_0\sum\limits_{\mbox{\boldmath{$i$}}}{o_{\mbox{\boldmath{$i$}}}}\tilde{b}^{r2}_{\mbox{\boldmath{$i$}}},{\kern 4pt}
\tilde{\mu}_{\lambda_r} = \tau^{-1}_{\lambda_r}\tau_0\sum\limits_{\mbox{\boldmath{$i$}}}{o_{\mbox{\boldmath{$i$}}}}\tilde{b}^r_{\mbox{\boldmath{$i$}}}\tilde{y}^r_{\mbox{\boldmath{$i$}}}.
\end{aligned}
\end{equation}

\subsubsection{Sampler for hyperperameter $\mbox{\boldmath{$\gamma$}}$}
The inference of $\vecb{\gamma}$ can be performed as follows by integrating the messages from the weight vector $\vecb{\lambda}$ and its hyperprior, both of which are expressed in Eq.~\eqref{eq:eq9}. Given $\vecb{\lambda}$ and $\vecb{\kappa}$, the posterior of $\vecb{\gamma}$ thus can be formulated as
\begin{equation}\label{eq:eq17}
\begin{aligned}
p(\mbox{\boldmath{$\gamma$}}|\sim) \propto \mathcal{N}(\mbox{\boldmath{$\lambda$}}|\mathbf{0},\Gamma)\prod\limits^R_{r=1} {\rm{Ga}}(\gamma_r|1,\frac{2}{\kappa_r}),
\end{aligned}
\end{equation}
where each entry $\lambda_r$ is independently distributed as
\begin{equation}\label{eq:eq18}
\begin{aligned}
p(\gamma_r|\sim) \propto \gamma^{-1/2}_r\exp\left(-\frac{\lambda^2_r + \kappa_r\gamma^2_r}{2\gamma_r}\right).
\end{aligned}
\end{equation}
Thus, $\gamma_r$ can be drawn from a generalized inverse Gaussian distribution as
\begin{equation}\label{eq:eq19}
\begin{aligned}
\gamma_r \sim \mathcal{GIG}(\gamma_r|\kappa_r, \lambda^2_r, 1/2).
\end{aligned}
\end{equation}
Given the $r$-th entry $\lambda_r$ and $\kappa_r$, the posterior expectation of $\gamma_r$ can be given as
\begin{equation}\label{eq:eq19-0}
\begin{aligned}
\mathbb{E}[\gamma_r] = \frac{|\lambda_r|\mathcal{K}_{3/2}(|\lambda_r|\sqrt{\kappa_r})}{\sqrt{\kappa_r}\mathcal{K}_{1/2}(|\lambda_r|\sqrt{\kappa_r})}
\end{aligned}
\end{equation}
which linearly depends on ${\lambda_r}$ and performs as a decreasing function of $\kappa_r$. To demonstrate this point, we introduce the following results.
\begin{theorem}\label{theorem:theorem1}
Given the modified Bessel function of second kind $K_{v}(z)$, we define a function $f_v(z) = zK_{v + 1}(z)/K_{v}(z)$, then, for any variable $x > 0$, $f_{1/2}(x) = x + 1$.
\end{theorem}

\begin{proof}
In~\cite{kreh2012bessel}, it has been proved that 
\begin{equation}\label{eq:eq19-1}
\begin{aligned}
&K_{v+1}(z) - K_{v-1}(z) = \frac{2v}{x}K_{n}(z);\\
&K_{v+\frac{1}{2}}(z) = K_{-v-\frac{1}{2}}(z),{\kern 4pt}\forall v\in{\mathbb{R}}.
\end{aligned}
\end{equation}
When $v= 1/2$, we have
\begin{equation}\label{eq:eq19-2}
\begin{aligned}
\frac{K_{3/2}(z)}{K_{1/2}(z)} = \frac{1}{x} + \frac{K_{-1/2}(z)}{K_{1/2}(z)} = \frac{1}{x} + 1.\\
\end{aligned}
\end{equation}
Thus, $f_{1/2}(x) = x(1/x + 1) = x + 1$.
\end{proof}
\begin{theorem}\label{theorem:theorem2}
Given $\lambda_r$ and $\kappa_r$, $\mathbb{E}[\gamma_r] = \frac{|\lambda_r|}{\sqrt{\kappa_r}} + \frac{1}{\kappa_r}$.
\end{theorem}

\begin{proof}
In Eq.~\eqref{eq:eq19-0}, $\mathbb{E}[\gamma_r] = \frac{f_{1/2}{(|\lambda_r|\sqrt{\kappa_r})}}{\kappa_r}$. According to Theorem~\ref{theorem:theorem1}, $\mathbb{E}[\gamma_r] = \frac{|\lambda_r|}{\sqrt{\kappa_r}} + \frac{1}{\kappa_r}$.
\end{proof}

\subsubsection{Sampler for hyperparameter $\mbox{\boldmath{$\kappa$}}$}
The posterior over $\vecb{\kappa}$ can be inferred by receiving the messages from hyperparameter $\vecb{\gamma}$, shown as Fig.~\ref{fig:PGM}, and the posterior over each entry $\kappa_r$ can be formulated as
\begin{equation}\label{eq:eq20}
\begin{aligned}
p(\kappa_r|-) &\propto {\rm{Ga}}(\gamma_r|1,\frac{2}{\kappa_r}) \propto \kappa_r\exp\left(-\frac{\gamma_r\kappa_r}{2}\right).
\end{aligned}
\end{equation}
Thus, we can draw $\kappa_r$ from a Gamma distribution as
\begin{equation}\label{eq:eq21}
\begin{aligned}
\kappa_r \sim {\rm{Ga}}(\kappa_r|2,\frac{2}{\gamma_r}).
\end{aligned}
\end{equation}

\begin{theorem}\label{theorem:theorem3}
Given $\lambda_r$ and the corresponding $\mathbb{E}[\gamma_r]$ for Eq.~\eqref{eq:eq21}, there exists a decreasing function $\rho$ s.t. $\mathbb{E}[\kappa_r] = \rho(|\lambda_r|)$.
\end{theorem}

\begin{proof}
Given $\lambda_r$ and the corresponding $\mathbb{E}[\lambda_r]$, the posterior for $\kappa_r$ can be reformulated as ${\rm{Ga}}(\kappa_r|2,\mathbb{E}[\lambda_r])$, and the resulted posterior expectation $\mathbb{E}[\kappa_r] = 4 / {\mathbb{E}[\lambda_r]}$. Thus, according to Theorem~\ref{theorem:theorem2}, we have 
\begin{equation}\label{eq:eq21-1}
\begin{aligned}
\mathbb{E}[\kappa_r] =  \rho(|\lambda_r|) \propto ({|\lambda_r| + {\rm{const}}})^{-1}
\end{aligned}
\end{equation}
where ${\rm{const}}$ denotes the variable relative to $\lambda_r$. Thus, $\rho(|\lambda_r|)$ is an decreasing function over $\lambda_r$.
\end{proof}
With Eq.~\eqref{eq:eq11}, weight $\lambda_r$ is shrunk with $\sqrt{\kappa_r}$ during exploiting sparsity of $\vecb{\lambda}$. Based on Theorem~\ref{theorem:theorem3}, $\mathbb{E}[\kappa_r] \propto ({|\lambda_r| + {\rm{const}}})^{-1}$, thus $\lambda_r$ will be shrunk less than $\lambda_j$ when $|\lambda_r| > |\lambda_j|$, which is consistent with the illustration of Remark~\ref{remark:remark2} in Section~\ref{subsubsection:sparsitylowrank}.

\subsubsection{Sampler for factor matrix $\vecb{U}^{(k)}$}
Similarly, the posterior for $\matb{U}^{(k)}$ is jointly determined by observation $\mathcal{Y}$, other factor matrices $\matb{U}^{(j)}$ with $j\neq k$, and hyperparameters $\mu^{(k)}$, $\tau^{(k)}$ in its prior. Since entries in each column of $\matb{U}^{(k)}$ are correlated with each other, we sample each each entry $u^{(k)}_{ir}$ with the posterior
\begin{equation}\label{eq:eq22}
\begin{aligned}
p(u^{(k)}_{ir}|-) \propto &\prod\nolimits_{{\mbox{\boldmath{$i$}}}: i_k = i} \left(\mathcal{N}(y_{\mbox{\boldmath{$i$}}} |e_{\mbox{\boldmath{$i$}}} + \sum\nolimits^{R}_{t=1}\lambda_t\prod\nolimits^{K}_{s=1}u^{(s)}_{i_st}, \tau^{-1}_0)^{o_{\mbox{\boldmath{$i$}}}}\right. \\
& \left. \mathcal{N}(u^{(k)}_{ir}|\mu^{(k)}, \tau^{(k)-1})\right).
\end{aligned}
\end{equation}
Thus, $u^{(k)}_{ir}$ can be drawn from a Gaussian distribution as
\begin{equation}\label{eq:eq23}
\begin{aligned}
u^{(k)}_{ir} \sim \mathcal{N}(\tilde{\mu}_{u^{(k)}_{ir}}, \tilde{\tau}_{u^{(k)}_{ir}}),
\end{aligned}
\end{equation}
where let $\tilde{c}^{rk}_{\mbox{\boldmath{$i$}}} = \lambda_r\prod\nolimits^K_{s\neq k}u^{(s)}_{i_sr}$ and the corresponding parameters are
\begin{equation}\label{eq:eq24}
\begin{aligned}
&\tilde{\tau}_{u^{(k)}_{ir}} = \sum\nolimits_{{\mbox{\boldmath{$i$}}}: i_k = i}\tau_0{o_{\mbox{\boldmath{$i$}}}}\tilde{c}^{rk2}_{\mbox{\boldmath{$i$}}}+ \tau^{(k)},\\
&\tilde{\mu}_{u^{(k)}_{ir}} = \tau^{-1}_{u^{(k)}_{ir}}\left(\sum\nolimits_{{\mbox{\boldmath{$i$}}}: i_k = i}\tau_0{o_{\mbox{\boldmath{$i$}}}}\tilde{y}^r_{\mbox{\boldmath{$i$}}}\tilde{c}^{rk}_{\mbox{\boldmath{$i$}}} + \tau^{(k)}\mu^{(k)}\right).
\end{aligned}
\end{equation}  

\subsubsection{Sampler for hyperparameter $\mu^{(k)}$ and $\tau^{(k)}$}
Receiving the messages from $\matb{U}^{(k)}$, $\beta_0$, $a_0$ and $b_0$, the posterior over $\mu^{(k)}$ can be given as
\begin{equation}\label{eq:eq26}
\begin{aligned}
p(\mu^{(k)}|-) \propto \prod\limits^{n_k}_{i=1}\prod\limits^{R}_{r=1}\mathcal{N}(u^{(k)}_{ir}|\mu^{(k)}, \tau^{(k)-1})\mathcal{N}(\mu^{(k)}|\mu_0, (\beta_0\tau^{(k)})^{-1}).
\end{aligned}
\end{equation}
Thus, $\mu^{(k)}$ can be drawn from a Gaussian distribution as
\begin{equation}\label{eq:eq27}
\begin{aligned}
\mu^{(k)} \sim \mathcal{N}(\tilde{\mu}_{\mu^{(k)}},\tilde{\tau}^{-1}_{\mu^{(k)}}),
\end{aligned}
\end{equation}
where the posterior parameters are
\begin{equation}\label{eq:eq27-1}
\begin{aligned}
\tilde{\tau}_{\mu^{(k)}} = \tau^{(k)}\left(\beta_0 + n_kR\right), \tilde{\mu}_{\mu^{(k)}} = \frac{\tau^{(k)}}{\tau_{\mu^{(k)}}}\left(\sum\limits^{n_k}_{i=1}\sum\limits^{R}_{r=1}u^{(k)}_{ir} + \beta_0\mu_0\right).
\end{aligned}
\end{equation}

Similarly, we have the posterior over $\tau^{(k)}$ as
\begin{equation}\label{eq:eq29}
\begin{aligned}
p(\tau^{(k)}|-) &\propto \prod\limits^{n_k}_{i=1}\prod\limits^{R}_{r=1}\left(\mathcal{N}(u^{(k)}_{ir}|\mu^{(k)}, \tau^{(k)-1})\right. \\
&\left. \mathcal{N}(\mu^{(k)}|\mu_0, (\beta_0\tau^{(k)})^{-1}){\rm{Ga}}(\tau^{(k)}|c_0,e_0)\right).
\end{aligned}
\end{equation}
It can be seen that $\tau^{(k)}_{d}$ can be drawn from the following Gamma distribution
\begin{equation}\label{eq:eq30}
\begin{aligned}
\tau^{(k)} \sim {\rm{Ga}}(\tilde{a}^{(k)}, \tilde{b}^{(k)}),
\end{aligned}
\end{equation}
where the parameters are
\begin{equation}\label{eq:eq30-1}
\begin{aligned}
&\tilde{a}^{(k)}  = a_0 + \left(n_kR + 1\right)/2;\\
&\tilde{b}^{(k)} = b_0 +  \left[\sum\limits^{n_k}_{i=1}\sum\limits^{R}_{r=1}\left(u^{(k)}_{ir} - \mu^{(k)}\right)^2 + \beta_0\left(\mu^{(k)} - \mu_0\right)^2\right] / 2.
\end{aligned}
\end{equation}

\subsection{Gibbs sampler for non-low-rank structure}
\subsubsection{Sampler for non-low-rank structure $\mathcal{E}$}
Figure~\ref{fig:PGM} shows that inferring the posterior of $\mathcal{E}$ requires the messages from observation $\mathcal{Y}$, factor matrices $\matb{U}^{(k)}$, weight vector $\vecb{\lambda}$ and parameters $\mu_d$s, $\tau_d$s in prior~\eqref{eq:eq5}. Since each entry $e_{\mbox{\boldmath{$i$}}}$ is assumed to be independent to others, we have the following posterior
\begin{equation}\label{eq:eq32}
\begin{aligned}
p(e_{\mbox{\boldmath{$i$}}}|-) \propto \mathcal{N}(y_{\mbox{\boldmath{$i$}}} |e_{\mbox{\boldmath{$i$}}} + x_{\mbox{\boldmath{$i$}}}, \tau^{-1}_0)^{o_{\mbox{\boldmath{$i$}}}}\prod\nolimits^{D}_{d = 1}\mathcal{N}(e_{\mbox{\boldmath{$i$}}}|\mu_d, \tau^{-1}_d)^{z^{d}_{\mbox{\boldmath{$i$}}}},
\end{aligned}
\end{equation}
and $e_{\mbox{\boldmath{$i$}}}$ can be drawn from a Gaussian distribution as
\begin{equation}\label{eq:eq33}
\begin{aligned}
e_{\mbox{\boldmath{$i$}}} \sim \mathcal{N}\left(\tilde{\mu}_{e_{\mbox{\boldmath{$i$}}}}, \tilde{\tau}^{-1}_{e_{\mbox{\boldmath{$i$}}}}\right),
\end{aligned}
\end{equation}
with parameters
\begin{equation}\label{eq:eq33-1}
\begin{aligned}
&\tilde{\tau}_{e_{\mbox{\boldmath{$i$}}}} = {o_{\mbox{\boldmath{$i$}}}}\tau_0 + \sum\nolimits^D_{d=1}\tau_dz^{d}_{\mbox{\boldmath{$i$}}},\\
&\tilde{\mu}_{e_{\mbox{\boldmath{$i$}}}} = \tilde{\tau}^{-1}_{e_{\mbox{\boldmath{$i$}}}}\left[\tau_0{o_{\mbox{\boldmath{$i$}}}}(y_{\mbox{\boldmath{$i$}}} - x_{\mbox{\boldmath{$i$}}}) + \sum\nolimits^D_{d=1}\tau_dz^{d}_{\mbox{\boldmath{$i$}}}\mu_d\right].
\end{aligned}
\end{equation}

\subsubsection{Sampler for hyperparameter $\mu_d$ and $\tau_d$}
With messages from $\mathcal{E}$ and parameters $\beta_0$, $a_0$, $b_0$ in hyperprior, the posterior over $\mu_d$ can be given as
\begin{equation}\label{eq:eq35}
\begin{aligned}
p(\mu_d|-) \propto \prod\nolimits_{\mbox{\boldmath{$i$}}} \mathcal{N}(e_{\mbox{\boldmath{$i$}}}|\mu_d, \tau^{-1}_d)^{z^{d}_{\mbox{\boldmath{$i$}}}}\mathcal{N}(\mu_d|\mu_0, (\beta_0\tau_d)^{-1}).
\end{aligned}
\end{equation}
Thus, $\mu_d$ can be drawn from a Gaussian distribution as 
\begin{equation}\label{eq:eq36}
\begin{aligned}
\mu_d \sim \mathcal{N}\left(\tilde{\mu}_{\mu_d}, \tilde{\tau}^{-1}_{\mu_d}\right),
\end{aligned}
\end{equation}
with parameters
\begin{equation}\label{eq:eq36-1}
\begin{aligned}
\tilde{\tau}_{\mu_d} &= \tau_d\left(\sum\nolimits_{\mbox{\boldmath{$i$}}} z^d_{\mbox{\boldmath{$i$}}} + \beta_0\right), {\kern 4pt} \tilde{\mu}_{\mu_d} = \tilde{\tau}^{-1}_{\mu_d}\tau_d\left(\sum\nolimits_{\mbox{\boldmath{$i$}}} z^d_{\mbox{\boldmath{$i$}}}e_{\mbox{\boldmath{$i$}}} + \beta_0\mu_0\right).
\end{aligned}
\end{equation}

Similarly, we have the posterior over $\tau_d$ as
\begin{equation}\label{eq:eq38}
\begin{aligned}
p(\tau_d|-)  &\propto \prod\limits_{\mbox{\boldmath{$i$}}}\mathcal{N}(z_{\mbox{\boldmath{$i$}}}|\mu_d, \tau^{-1}_d)^{z^{d}_{\mbox{\boldmath{$i$}}}}\mathcal{N}(\mu_d|\mu_0, (\beta_0\tau_d)^{-1}){\rm{Ga}}(\tau_d|a_0,b_0).
\end{aligned}
\end{equation}
$\tau_{d}$ thus can be drawn from a Gamma distribution as
\begin{equation}\label{eq:eq39}
\begin{aligned}
\tau_{d} \sim {\rm{Ga}}(\tilde{a}_d, \tilde{b}_d),
\end{aligned}
\end{equation}
with parameters
\begin{equation}\label{eq:eq39-1}
\begin{aligned}
&\tilde{a}_d  = a_0 + \left(\sum\nolimits_{\mbox{\boldmath{$i$}}} z^d_{\mbox{\boldmath{$i$}}} + 1\right)/2,\\
&\tilde{b}_d  = b_0 +  \left[\sum\nolimits_{\mbox{\boldmath{$i$}}} z^d_{\mbox{\boldmath{$i$}}}e_{\mbox{\boldmath{$i$}}} + \beta_0\left(\mu_d - \mu_0\right)^2\right] / 2;
\end{aligned}
\end{equation}

\subsubsection{Sampler for $\mbox{\boldmath{$z$}}_{\mbox{\boldmath{$i$}}}$}
Similarly as previous, we have the posterior over $\mbox{\boldmath{$z$}}_{\mbox{\boldmath{$i$}}}$ as
\begin{equation}\label{eq:eq41}
\begin{aligned}
p(\mbox{\boldmath{$z$}}_{\mbox{\boldmath{$i$}}}|-) \propto {\rm{Multinomial}}(\mbox{\boldmath{$z$}}_{\mbox{\boldmath{$i$}}}|\mbox{\boldmath{$\pi$}}).
\end{aligned}
\end{equation}
$\mbox{\boldmath{$z$}}_{\mbox{\boldmath{$i$}}}$ thus can be drawn from a multinomial distribution as
\begin{equation}\label{eq:eq42}
\begin{aligned}
\mbox{\boldmath{$z$}}_{\mbox{\boldmath{$i$}}} \sim {\rm{Multinomial}}(\mbox{\boldmath{$z$}}_{\mbox{\boldmath{$i$}}}|\tilde{\mbox{\boldmath{$\pi$}}}),
\end{aligned}
\end{equation}
with $\tilde{\mbox{\boldmath{$\pi$}}} = (\tilde{\pi}_1,..., \tilde{\pi}_{D})$ and each entry $\tilde{\pi}_d$ is
\begin{equation}\label{eq:eq43}
\begin{aligned}
\tilde{\pi}_d = {\pi_d\mathcal{N}(e_{\mbox{\boldmath{$i$}}}|\mu_d, \tau^{-1}_d)}/{\sum\nolimits^D_{t = 1}\pi_{t}\mathcal{N}(e_{\mbox{\boldmath{$i$}}}|\mu_t, \tau^{-1}_t)}.
\end{aligned}
\end{equation}

\subsubsection{Sampler for hyperparameter $\mbox{\boldmath{$\pi$}}$}
Integrating messages from $\mbox{\boldmath{$z$}}_{\mbox{\boldmath{$i$}}}$ and ${\mbox{\boldmath{$\alpha$}}}_0$ in hyperprior, we have the posterior over $\mbox{\boldmath{$\pi$}}$ as
\begin{equation}\label{eq:eq44}
\begin{aligned}
p(\mbox{\boldmath{$\pi$}}|\sim) &\propto \prod\limits_{{\mbox{\boldmath{$i$}}}} {\rm{Multinomial}}(\mbox{\boldmath{$z$}}_{\mbox{\boldmath{$i$}}}|\mbox{\boldmath{$\pi$}}){\rm{Dir}}({\mbox{\boldmath{$\pi$}}}|{\mbox{\boldmath{$\alpha$}}}_0).
\end{aligned}
\end{equation}
Thus, $\mbox{\boldmath{$\pi$}}$ can be drawn from a Dirichlet distribution as
\begin{equation}\label{eq:eq45}
\begin{aligned}
\mbox{\boldmath{$\pi$}} \sim  {\rm{Dir}}({\mbox{\boldmath{$\pi$}}}|\tilde{\mbox{\boldmath{$\alpha$}}}),
\end{aligned}
\end{equation}
with $\tilde{\mbox{\boldmath{$\alpha$}}} = (\tilde{\alpha}_{1},..., \tilde{\alpha}_D)$ and each entry $\tilde{\alpha}_d$ is
\begin{equation}\label{eq:eq46}
\begin{aligned}
\tilde{\alpha}_d = \sum\nolimits_{{\mbox{\boldmath{$i$}}}}z^d_{\mbox{\boldmath{$i$}}} + \alpha_{0d}.
\end{aligned}
\end{equation}

\subsection{Algorithm, complexity and convergence}
With above samplers, the entire Gibbs sampling flow is summarized in Algorithm~\ref{alg:SSFHCS} where $N_b,N_s$ are the iteration number for burn-in and samples collection.
\begin{algorithm}
\caption{Data-adaptive tensor completion}
\label{alg:SSFHCS}
\KwIn{Observation $\mathcal{Y}_{\Omega}$, parameters $a_0,b_0, \mu_0, \beta_0, \mbox{\boldmath{$\alpha$}}_0, \tau_0$.}
{\bf{Initialization}}: $\mathcal{U}, {\mbox{\boldmath{$\lambda$}}}, {\mbox{\boldmath{$\tau$}}}, {\mbox{\boldmath{$\tau$}}}_e$ are initialized by ${\bf{1}}$ with proper size,  ${\mbox{\boldmath{$\mu$}}}, {\mbox{\boldmath{$\mu$}}}_e$ are initialized by $\bf{0}$, ${\mbox{\boldmath{$\pi$}}}=(D^{-1},...,D^{-1})^T$ and ${\mbox{\boldmath{$z$}}}_{\mbox{\boldmath{$i$}}} \sim {\rm{Multinomial}}(\mbox{\boldmath{$z$}}_{\mbox{\boldmath{$i$}}}|\mbox{\boldmath{$\pi$}})$.\\
\For {$t \leftarrow 1$ \textbf{to} $N_{b} + N_{s}$}
{
\textbf{Sample the low-rank structure:}\\
$1.$ Sample $\lambda_r$,  as Eq.~\eqref{eq:eq15-1};\\
$2.$ Sample $\gamma_r$ as Eq.~\eqref{eq:eq19};\\
$3.$ Sample $\kappa_r$ as Eq.~\eqref{eq:eq21};\\
$4.$ Sample $u^{(k)}_{ir}$ as Eq.~\eqref{eq:eq23};\\
$5.$ Sample $\mu^{(k)}$ as Eq.~\eqref{eq:eq27};\\
$6.$ Sample $\tau^{(k)}_{d}$ as Eq.~\eqref{eq:eq30};\\
\textbf{Sample the non-low-rank structure:}\\
$1.$ Sample $e_{\mbox{\boldmath{$i$}}}$ as Eq.~\eqref{eq:eq33};\\
$2.$ Sample $\mu_d$as Eq.~\eqref{eq:eq36};\\
$3.$ Sample $\tau_{d}$ as Eq.~\eqref{eq:eq39};\\
$4.$ Sample $\mbox{\boldmath{$z$}}_{\mbox{\boldmath{$i$}}}$ as Eq.~\eqref{eq:eq42};\\
$5.$ Sample $\mbox{\boldmath{$\pi$}}$ as Eq.~\eqref{eq:eq45};\\
\If{$t > N_{b}$}{ Collect samples for ${\mbox{\boldmath{$\lambda$}}}, \mathcal{U}, \mathcal{E}$;}
}
$4.$ Get the sample mean $\bar{\mbox{\boldmath{$\lambda$}}}, \bar{\mathcal{U}}, \bar{\mathcal{E}}$ across $N_{s}$ samples;\\
$5.$ Complete the latent tensor $\hat{\mathcal{L}}$ with $\bar{\mbox{\boldmath{$\lambda$}}}, \bar{\mathcal{U}}, \bar{\mathcal{E}}$ as  Eqs.~\eqref{eq:eq7}~\eqref{eq:eq1}.
\end{algorithm}

In each iteration of Gibbs sampling, we draw samples for all unknown variables with corresponding samplers. It can be found that per-iteration sampling complexity is dominated by $O(RKN)$, which is only linear in the number $N$ of entries in the latent tensor $\mathcal{L}$. Moreover, it has been shown that MMSE often converges well~\cite{schmidt2010generative} and more evidence will be provided in Section~\ref{subsection:syntetic}.


\section{Experiments}
In this study, we compreshensively evaluate the proposed model with extensive experiments on synthetic and real visual tensor data. Specifically, experiments on synthetic data are conducted to validate the following five aspects of the proposed model: i) automatic tensor rank determination; ii) the ability of fitting a wide range of non-low-rank structures $\mathcal{E}$s; iii) the effectiveness of the data-adaptive tensor model; iv) convergence; v) recovery performance. In addition, other three practical applications including image inpainting, video completion and facial image synthesis, are conducted to further evaluate the recovery performance of the proposed model. To this end, the proposed model is compared with $8$ state-of-the-art low rank tensor completion methods, including FaLRTC~\cite{liu2013tensor}, HaLRTC~\cite{liu2013tensor}, RPTC$_{\rm{scad}}$~\cite{zhao2015novel}, TMac~\cite{xu2013parallel}, STDC~\cite{chen2014simultaneous}, t-SVD~\cite{zhang2014novel}, FBCP~\cite{zhao2015bayesian} and BRTF~\cite{zhao2016bayesian}. It is noticeable that only BRTF considers a sparse non-low-rank structure $\mathcal{E}$, while the others adopt a zero $\mathcal{E}$.

In the proposed model\footnote{Code will be made available at publication time.}, we set all hyperparameters in a non-informative manner to reduce their influence on the posterior as much as possible~\cite{bishop2006pattern,zhao2014robust}. Throughout our experiments, $\mu_0 = 0$, and $\alpha_{01}, ..., \alpha_{0D}$, $\beta_0$, $a_0$ and $b_0$ are $10^{-6}$. The number $D$ of Gaussians in the mixture can be automatically determined by a simple tuning scheme in~\cite{zhao2014robust} where the maximum $D$ is set $6$. The initialized tensor rank $R$ is set $20$ and $100$ for synthetic data and real visual data, respectively. Other competitors are implemented by the codes of authors with tuned parameters for best performance. For simplicity, we term the proposed model 'Ours' in the following tables and figures.

\begin{figure}
\setlength{\abovecaptionskip}{0pt}
\begin{center}
\subfigure[Rank = 5]{\includegraphics[height=1.55in,width=1.7in,angle=0]{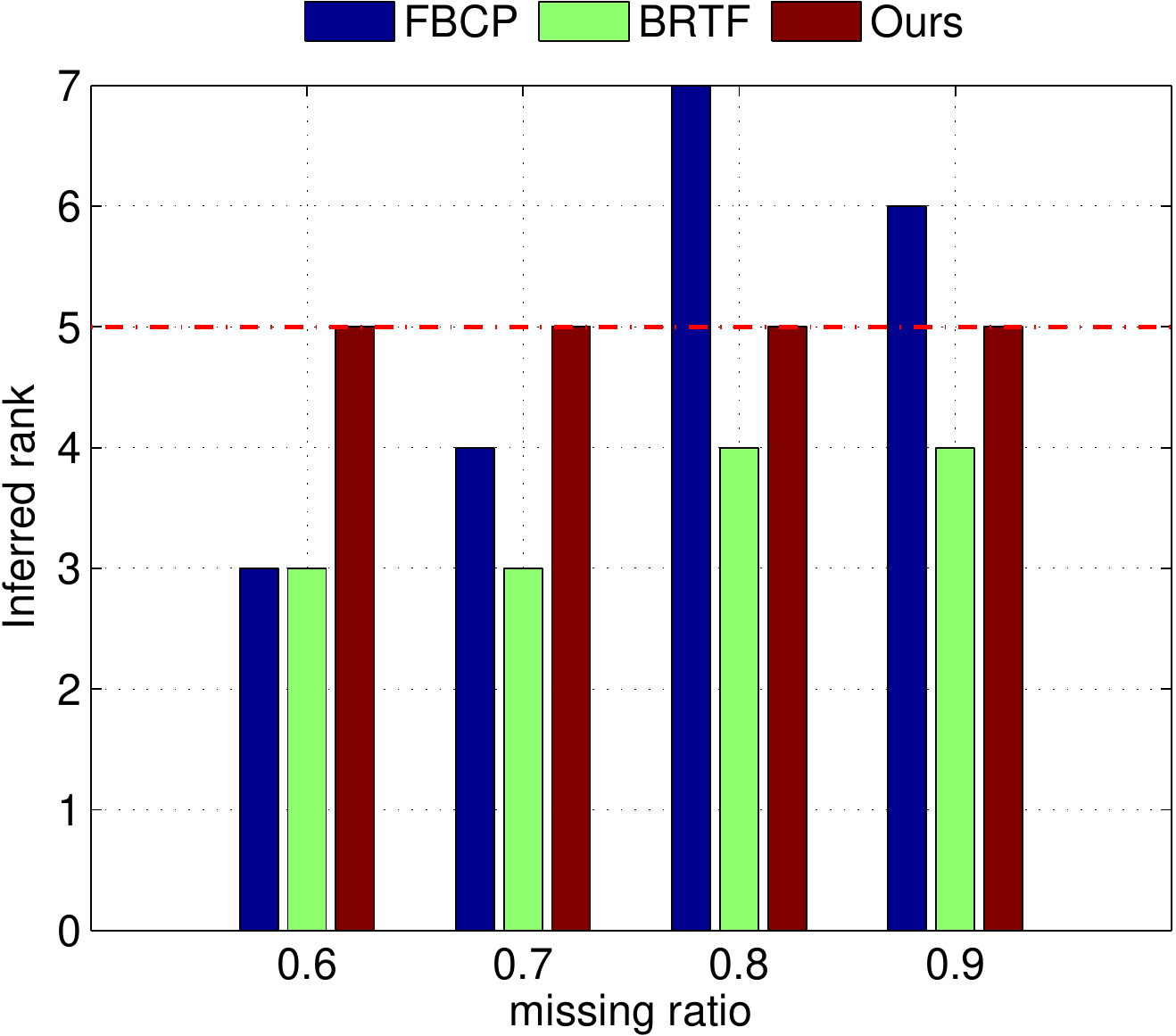}}
\hspace{-0.15cm}
\subfigure[Rank = 10]{\includegraphics[height=1.55in,width=1.7in,angle=0]{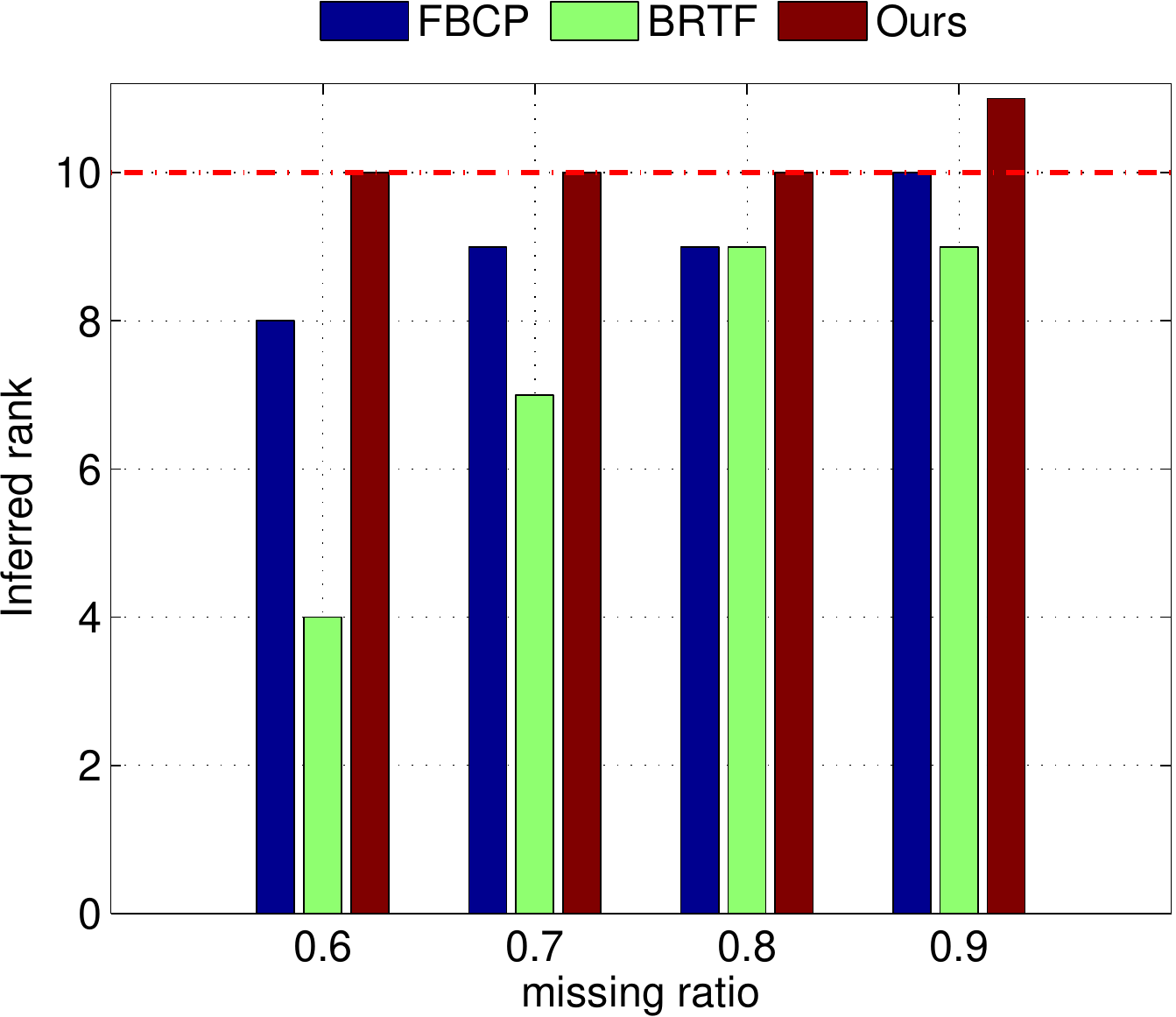}}
\end{center}
\vspace{-0.2cm}
\caption{Tensor rank determined by different methods on the synthetic data. The vertical color bars denote the inferred tensor rank, while the horizontal red dash lines indicate the true tensor rank.}
\vspace{-0.2cm}
\label{fig:rank}
\end{figure}

\begin{table*}\footnotesize
\caption{Quantitative comparison between parameters in true PDFs of $\mathcal{E}$ (denoted 'True') and that in estimated ones (denote 'Est.').}
\vspace{-0.2cm}
\renewcommand{\arraystretch}{1.1}
\begin{center}
\begin{tabu} to 0.9\textwidth {X[c]|X[c]|X[c]|X[c]|X[c]|X[c]|X[c]|X[c]|X[c]|X[c]|X[c]|X[c]}
\hline
\multicolumn{2}{c|}{\multirow{2}{*}{rank = $5$}} & Zero & \multicolumn{2}{c|}{Sparse} & Gaussian & \multicolumn{3}{c|}{Mixture (zero mean)} & \multicolumn{3}{c}{Mixture (non-zero mean)}\\
\cline{3-12}
 \multicolumn{2}{c|}{} & Comp.$1$ & Comp.$1$ & Comp.$2$ & Comp.$1$ & Comp.$1$ & Comp.$2$ & Comp.$3$ & Comp.$1$ & Comp.$2$ & Comp.$3$\\
\hline
\multirow{2}{*}{$\pi_d$} & True & - & $0.1$ & $0.9$ & - & $0.1$ & $0.3$ & $0.6$ & $0.1$ & $0.2$ & $0.7$\\
\cline{2-12}
& Est. & - & $0.102$ & $0.898$ & - & $0.150$ & $0.282$ & $0.568$ & $0.065$ & $0.250$ & $0.685$\\
\hline
\multirow{2}{*}{$\mu_d$} & True & $0$ & $0$ & $0$ & $0.1$ & $0$ & $0$ & $0$ & $1.5$ & $0.1$ & -$0.1$\\
\cline{2-12}
& Est. & -$2$e-$4$ & -$0.015$ & $4$e-$4$ & $0.101$ & -$0.017$ & -$0.012$ & $1$e-$4$ & $2.429$ & $0.050$ & -$0.1$\\
\hline
\multirow{2}{*}{$\tau_d$} & True & $1$e$12$ & $0.75$ & $1$e$12$ & $100$ & $0.75$ & $10$ & $200$ & $0.48$ & $10$ & $300$\\
\cline{2-12}
& Est. & $6.64$e$4$ & $0.7081$ & $4.13$e$5$ & $100.1$ & $1.072$ & $13.582$ & $220.8$ & $1.050$ & $7.086$ & $314.6$\\
\hline
\hline
\multicolumn{2}{c|}{\multirow{2}{*}{rank = $10$}} & Zero & \multicolumn{2}{c|}{Sparse} & Gaussian & \multicolumn{3}{c|}{Mixture (zero mean)} & \multicolumn{3}{c}{Mixture (non-zero mean)}\\
\cline{3-12}
 \multicolumn{2}{c|}{} & Comp.$1$ & Comp.$1$ & Comp.$2$ & Comp.$1$ & Comp.$1$ & Comp.$2$ & Comp.$3$ & Comp.$1$ & Comp.$2$ & Comp.$3$\\
\hline
\multirow{2}{*}{$\pi_d$} & True & - & $0.1$ & $0.9$ & - & $0.1$ & $0.3$ & $0.6$ & $0.1$ & $0.2$ & $0.7$\\
\cline{2-12}
& Est. & - & $0.101$ & $0.899$ & - & $0.149$ & $0.281$ & $0.570$ & $0.062$ & $0.239$ & $0.699$\\
\hline
\multirow{2}{*}{$\mu_d$} & True & $0$ & $0$ & $0$ & $0.1$ & $0$ & $0$ & $0$ & $1.5$ & $0.1$ & -$0.1$\\
\cline{2-12}
& Est. & -$2$e-$4$ & $0.032$ & $0.000$ & $0.100$ & $0.020$ & $4$e-$4$ & -$6$e-$4$ & $2.611$ & $0.094$ & -$0.100$\\
\hline
\multirow{2}{*}{$\tau_d$} & True & $1$e$12$ & $0.75$ & $1$e$12$ & $100$ & $0.75$ & $10$ & $200$ & $0.48$ & $10$ & $300$\\
\cline{2-12}
& Est. & $6.16$e$5$ & $0.728$ & $9.12$e$4$ & $98.76$ & $1.005$ & $15.70$ & $220.6$ & $1.391$ & $6.392$ & $313.9$\\
\hline
\end{tabu}
\end{center}
\label{table:Fittness}
\vspace{-0.2cm}
\end{table*}

\begin{figure*}
\setlength{\abovecaptionskip}{0pt}
\begin{center}
\includegraphics[height=0.65in,width=1.37in,angle=0]{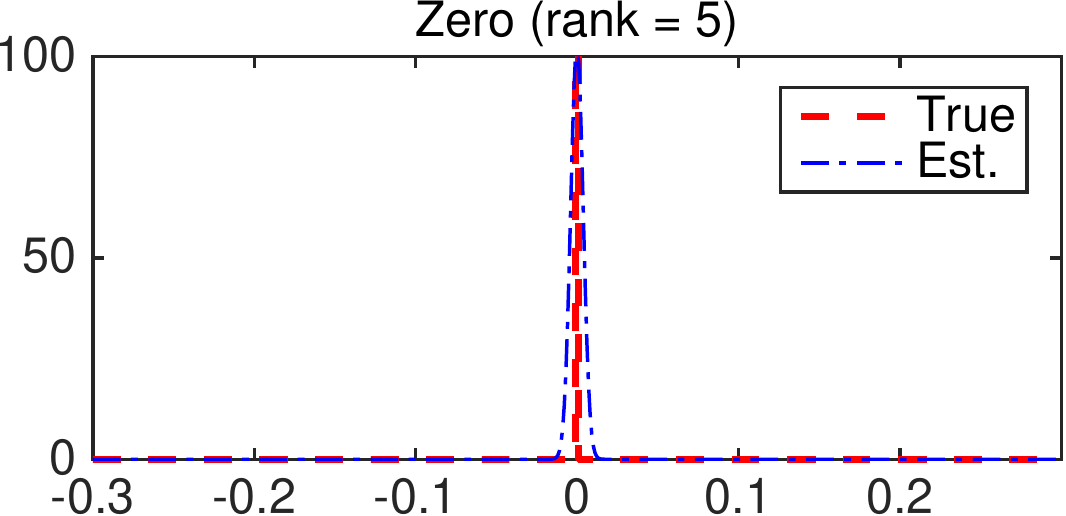}
\hspace{-0.15cm}
\includegraphics[height=0.65in,width=1.37in,angle=0]{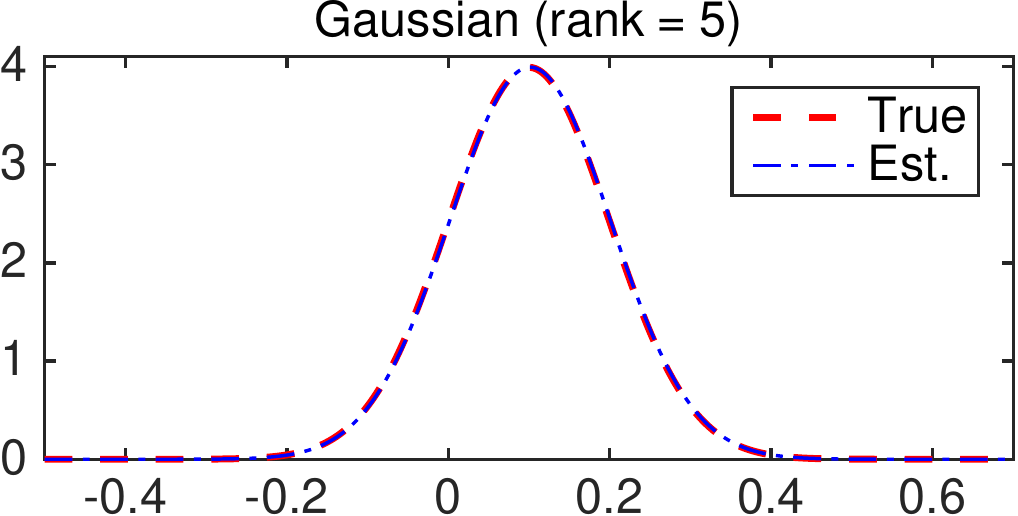}
\hspace{-0.15cm}
\includegraphics[height=0.65in,width=1.37in,angle=0]{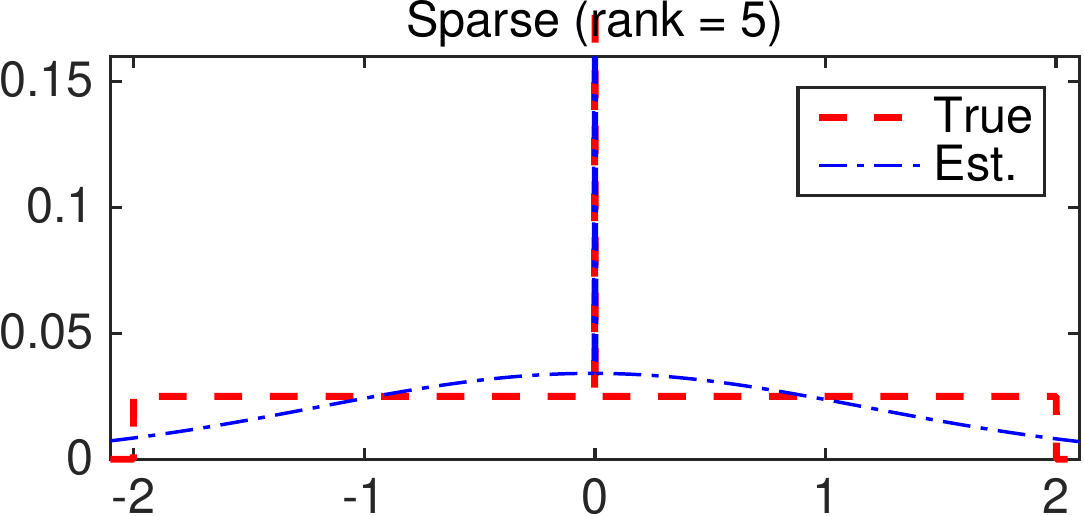}
\hspace{-0.15cm}
\includegraphics[height=0.65in,width=1.37in,angle=0]{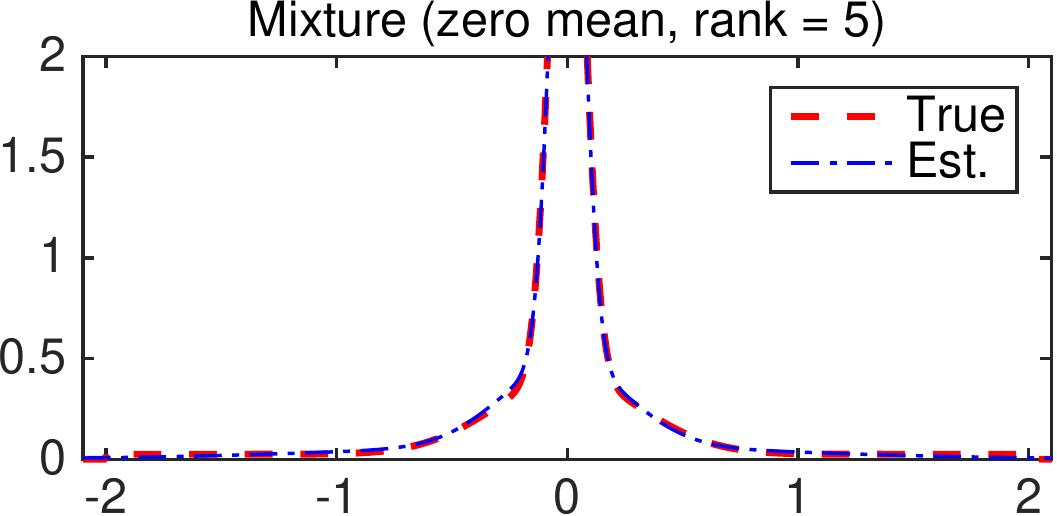}
\hspace{-0.15cm}
\includegraphics[height=0.65in,width=1.37in,angle=0]{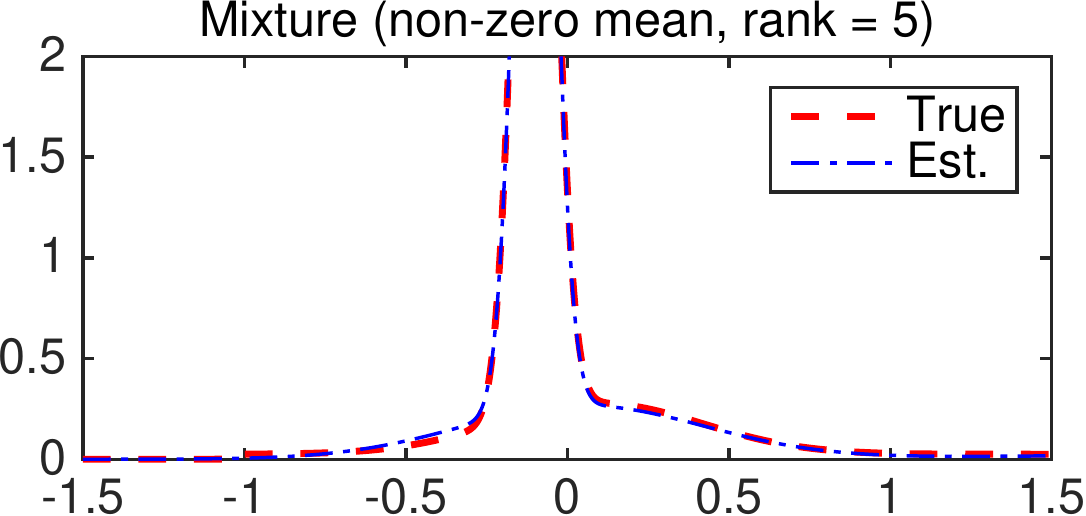}\\
\includegraphics[height=0.65in,width=1.37in,angle=0]{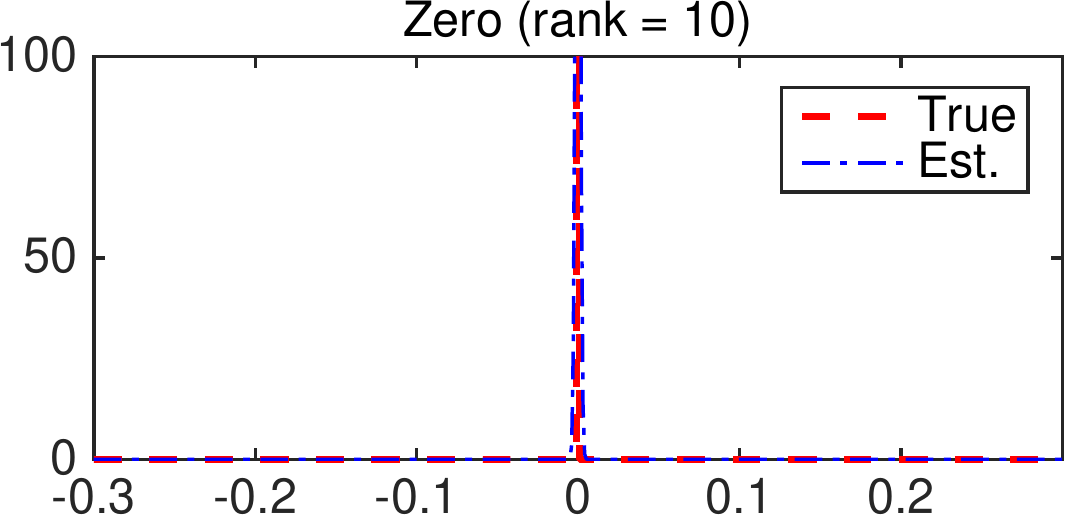}
\hspace{-0.15cm}
\includegraphics[height=0.65in,width=1.37in,angle=0]{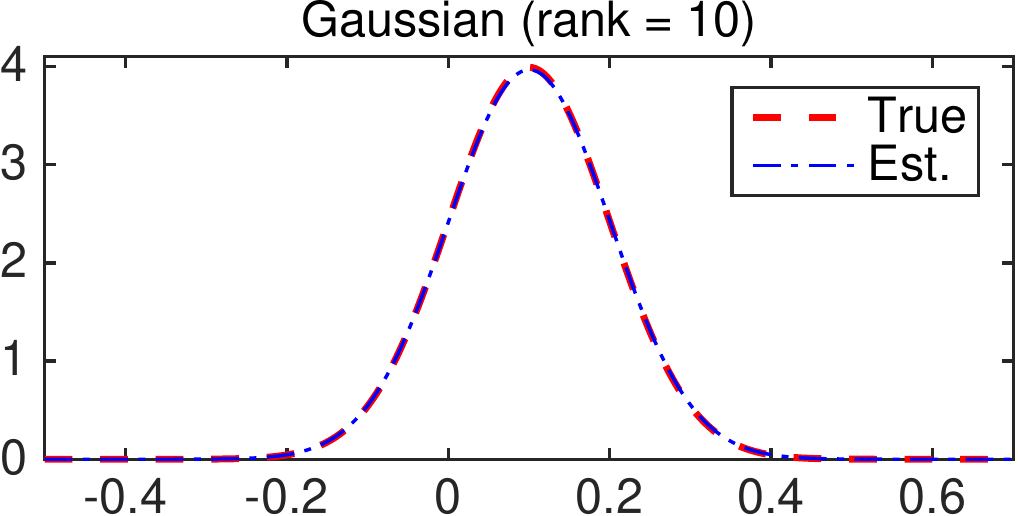}
\hspace{-0.15cm}
\includegraphics[height=0.65in,width=1.37in,angle=0]{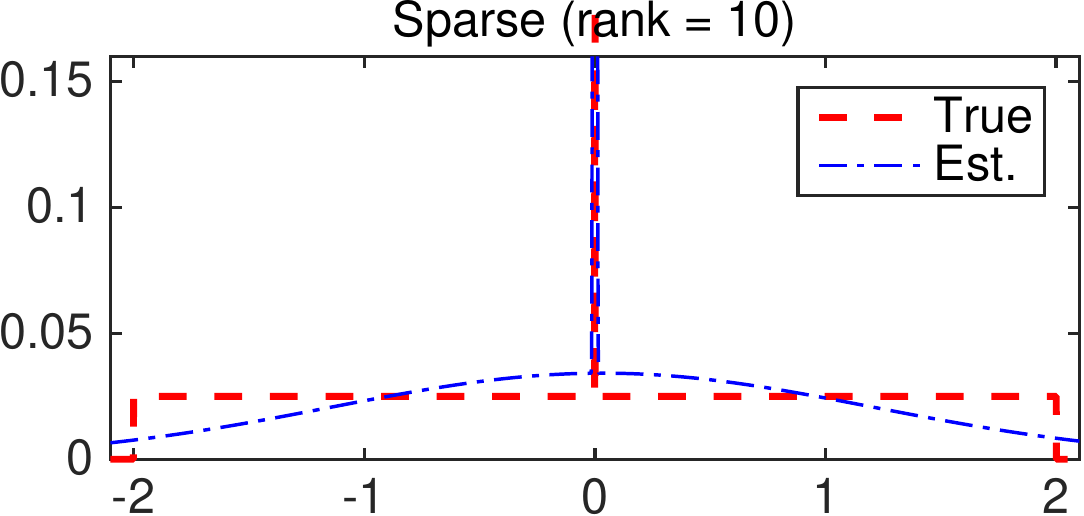}
\hspace{-0.15cm}
\includegraphics[height=0.65in,width=1.37in,angle=0]{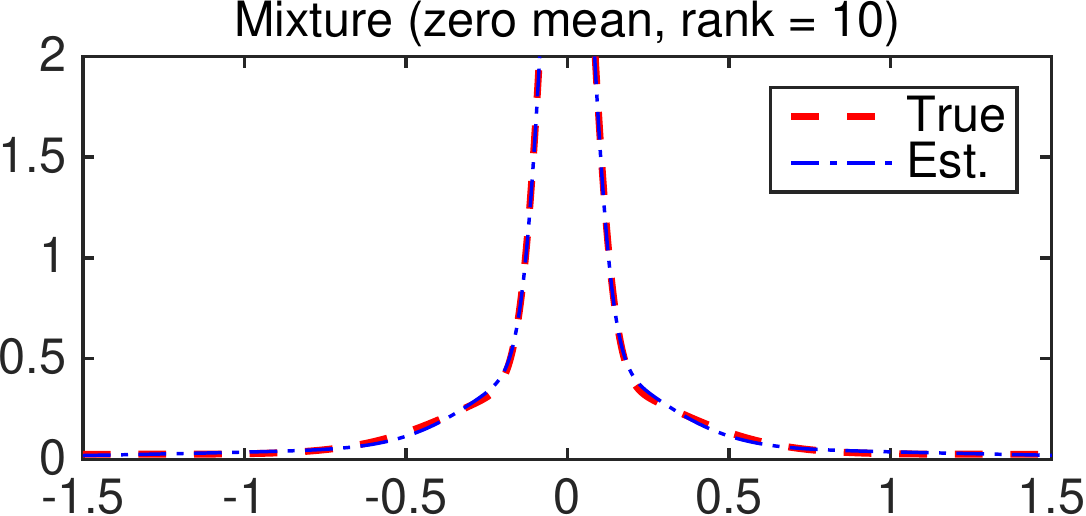}
\hspace{-0.15cm}
\includegraphics[height=0.65in,width=1.37in,angle=0]{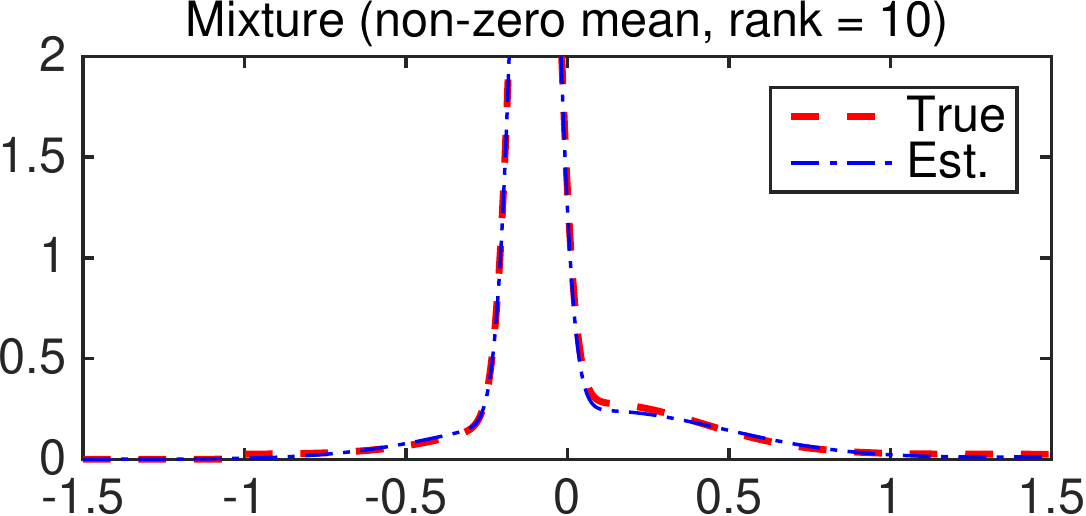}
\end{center}
\vspace{-0.1cm}
\caption{Visual comparison between true PDFs of $\mathcal{E}$ (denoted 'True') and those estimated ones (denote 'Est.').}
\vspace{-0.1cm}
\label{fig:mtpi}
\end{figure*}

\subsection{Tensor completion on synthetic data}\label{subsection:syntetic}
We generate the latent tensor $\mathcal{L}$ of size $30 \times 30 \times 30$ based on Eq.~\eqref{eq:eq1}. The low rank structure $\mathcal{X}$ of ${\rm{rank}}(\mathcal{X}) = 5$ is generated by the CP factorization in Eq.~\eqref{eq:eq7} with factor matrices $\mbox{\boldmath{$U$}}^{(k)} \in {\mathbb{R}^{30 \times 5}}$s and weights vector $\mbox{\boldmath{$\lambda$}} \in{\mathbb{R}^{5\times 1}}$. Entries in $\mbox{\boldmath{$U$}}^{(k)}$ are independently sampled from $\mathcal{N}(0,1)$, while weights in $\mbox{\boldmath{$\lambda$}}$ are uniformly sampled in the range of $[$-$2,2]$. To simulate a wide range of residual components, $5$ types of $\mathcal{E}$ are generated as follows. (1) Zero non-low-rank structure with all zero entries. (2) Gaussian non-low-rank structure with all entries sampled from $\mathcal{N}(0, 0.01)$. (3) Sparse non-low-rank structure with $10\%$ entries sampled uniformly from the range of $[$-$2,2]$. (4) Mixture non-low-rank structure (zero mean) with $10\%$ entries sampled uniformly in the range of $[$-$2,2]$, $30\%$ entries sampled from $\mathcal{N}(0,0.1)$ and $60\%$ entries sampled from $\mathcal{N}(0,0.005)$. (5) Mixture non-low-rank structure (non-zero mean) with $10\%$ entries sampled uniformly in the range of $[$-$1,4]$, $20\%$ entries sampled from $\mathcal{N}(0.1,0.1)$ and $70\%$ entries sampled from $\mathcal{N}($-$0.1,1/300)$. Another $5$ simulations for $\mathcal{L}$ are generated in a similar manner except for setting ${\rm{rank}}(\mathcal{X}) = 10$. To obtain the incomplete observation $\mathcal{Y}_{\Omega}$, we first add noise $\mathcal{M}$ with entries sampled from $\mathcal{N}(0,0.001)$ on $\mathcal{L}$. Then, a certain percentage (i.e., missing ratio) of entries are randomly selected from the noisy $\mathcal{L}$ and set zeros. We choose missing ratio from $70\%$ to $ 90\%$.

\begin{table}\footnotesize
\caption{RRE on synthetic tensor with different missing ratios.}
\renewcommand{\arraystretch}{1.1}
\begin{center}
\begin{tabu} to 0.48\textwidth {X[1.8,l]|X[0.9,c]|X[c]|X[0.9,c]|X[0.9,c]|X[c]|X[0.9,c]}
\hline
 & \multicolumn{3}{c|}{rank = $5$} & \multicolumn{3}{c}{rank = $10$}\\
\hline
Method & $70\%$ & $80\%$ & $90\%$ & $70\%$ & $80\%$ & $90\%$\\
\hline
Ours\_{nx}($6$) & - & - & - & - & - & -\\
Ours\_{nx}($10$) & - & - & - & - & - & -\\
Ours\_{nx}($20$) & - & - & - & - & - & -\\
Ours\_{nx}($50$) & - & - & - & - & - & -\\
Ours\_{ne} & 0.2234 & 0.2242 & 0.2426 & 0.2210 & 0.2372 & 0.2702\\
Ours & \textbf{0.1829} & \textbf{0.1960} & \textbf{0.2038} & \textbf{0.1796} & \textbf{0.1914} & \textbf{0.2055}\\
\hline
\end{tabu}
\end{center}
$^*$'-' denotes recovery failure.
\label{table:effectivenss}
\end{table}

\begin{figure}
\setlength{\abovecaptionskip}{0pt}
\begin{center}
\subfigure[\vspace{-0.3cm}Rank = 5]{\includegraphics[height=1.4in,width=1.7in,angle=0]{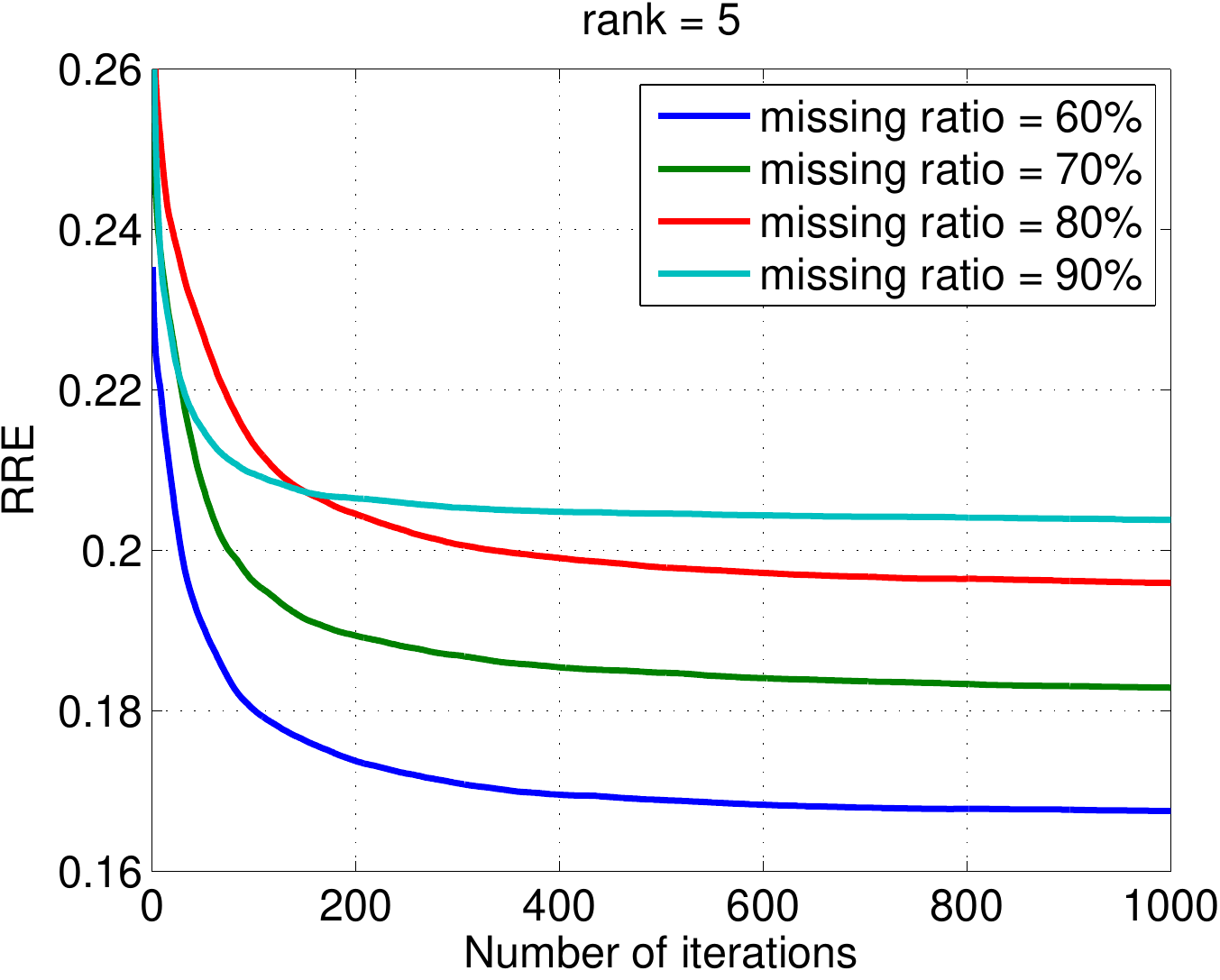}}
\hspace{-0.15cm}
\subfigure[\vspace{-0.3cm}Rank = 10]{\includegraphics[height=1.4in,width=1.7in,angle=0]{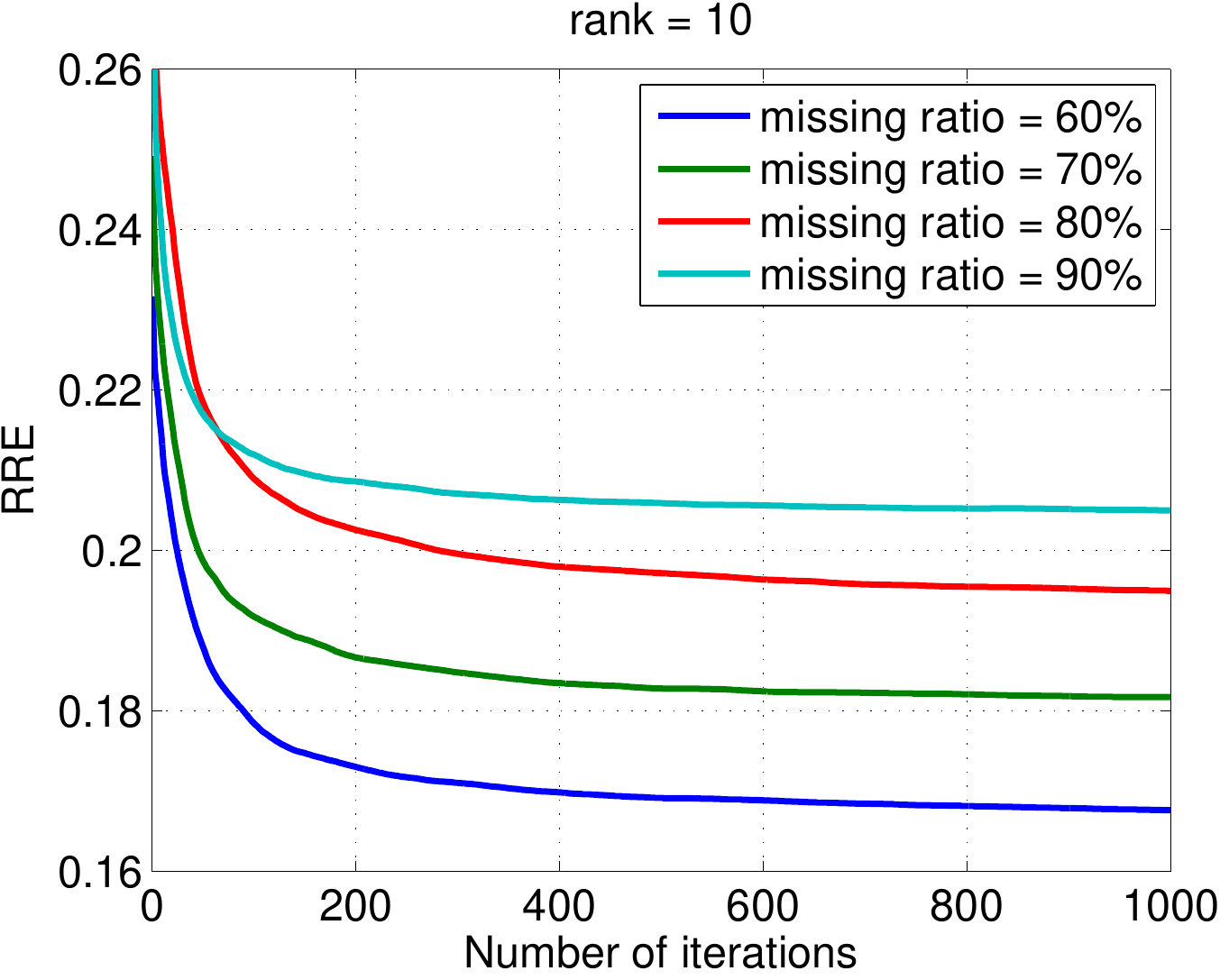}}
\end{center}
\caption{RRE curves versus the iteration number of Gibbs sampling on two synthetic tensors with different missing ratios.}
\label{fig:convergence}
\end{figure}

\subsubsection{Tensor rank determination} 
We evaluate the proposed model on $2$ different synthetic tensors which consist of two low rank $\mathcal{X}$s of ${\rm{rank}}(\mathcal{X}) = 5,10$ but a same mixture non-low-rank structure $\mathcal{E}$ with non-zero mean. Given $\mathcal{Y}_{\Omega}$ with different missing ratios, we infer the sparse weights vector $\mbox{\boldmath{$\lambda$}}$, and give the inferred tensor rank as the number of weights $|\lambda_r| > 10$e-$5$. Since most competitors cannot infer the tensor rank, only the inferred ranks from the proposed model, FBCP and BRTF are shown in Figure~\ref{fig:rank}. It can be seen that when ${\rm{rank}}(\mathcal{X}) = 5$, the proposed model infers the tensor rank exactly even with $90\%$ missing entries, while both FBCP and BRTF fail to estimate the real tensor rank accurately. When ${\rm{rank}}(\mathcal{X}) = 10$, the proposed model only misses the true rank when missing ratio is $90\%$, while FBCP and BRTF miss the true one in most cases. This is mainly owing to the newly defined tensor rank as well as the sparsity induced low-rank model adopted. In addition, modelling the non-low-rank structure also benefits separating the low-rank structure exactly and thus determining the accurate tensor rank. Therefore, we can conclude that the proposed model can accurately estimate the tensor rank, even if the tensor data contains complex non-low-rank structure with numerous missing entries. 

\begin{table}\footnotesize
\caption{RRE on synthetic tensors with different non-low-rank structures and missing ratios.}
\renewcommand{\arraystretch}{1.1}
\vspace{-0.5cm}
\begin{center}
\begin{tabu} to 0.5\textwidth {X[2,l]|X[0.9,c]|X[c]|X[0.9,c]|X[0.9,c]|X[c]|X[0.9,c]}
\hline
\multicolumn{7}{c}{Zero non-low-rank structure $\mathcal{E}$}\\
\hline
 & \multicolumn{3}{c|}{rank = $5$} & \multicolumn{3}{c}{rank = $10$}\\
\hline
Method & $70\%$ & $80\%$ & $90\%$ & $70\%$ & $80\%$ & $90\%$\\
\hline
FaLRTC~\cite{liu2013tensor} & 0.0248 & 0.1990 & 0.6942 & 0.4386 & 0.7057 & 0.9105\\
HaLRTC~\cite{liu2013tensor} & 0.0218 & 0.1974 & 0.6934 & 0.4383 & 0.7056 & 0.9105\\
RPTC$_{\rm{scad}}$~\cite{zhao2015novel} & 0.0063 & 0.1948 & 0.8367 & 0.5245 & 0.7485 & 0.9168\\
TMac~\cite{xu2013parallel} & 0.8354 & 0.8917 & 0.9498 & 0.8352 & 0.8920 & 0.9479\\
STDC~\cite{chen2014simultaneous} & 0.1694 & 0.3685 & 0.8714 & 0.5585 & 0.8377 & 0.9965\\
t-SVD~\cite{zhang2014novel} & 0.3874 & 0.6242 & 0.8687 & 0.6224 & 0.7924 & 0.9498\\
BCPF~\cite{zhao2015bayesian} & \textbf{0.0027} & \textbf{0.0032} & 0.1589 & 0.0404 & 0.0417 & 0.3928\\
BRTF~\cite{zhao2016bayesian}& 0.1279 & 0.1379 & 0.1509 & 0.0352 & 0.1393 & 0.4719\\
Ours & 0.0086 & 0.0041 & \textbf{0.0047} & \textbf{0.0036} & \textbf{0.0047} & \textbf{0.0074}\\
\hline
\multicolumn{7}{c}{Sparse non-low-rank structure $\mathcal{E}$}\\
\hline
 & \multicolumn{3}{c|}{rank = $5$} & \multicolumn{3}{c}{rank = $10$}\\
\hline
Method & $70\%$ & $80\%$ & $90\%$ & $70\%$ & $80\%$ & $90\%$\\
\hline
FaLRTC~\cite{liu2013tensor}& 0.0750 & 0.2170 & 0.6864 & 0.4944 & 0.7159 & 0.9067\\
HaLRTC~\cite{liu2013tensor}& 0.0754 & 0.2179 & 0.6857 & 0.4938 & 0.7154 & 0.9068\\
RPTC$_{\rm{scad}}$~\cite{zhao2015novel}& 0.0297 & 0.2025 & 0.8258 & 0.5564 & 0.7545 & 0.9143\\
TMac~\cite{xu2013parallel}& 0.8355 & 0.8916 & 0.9483 & 0.8450 & 0.8935 & 0.9470\\
STDC~\cite{chen2014simultaneous}& 0.1624 & 0.3947 & 0.8797 & 0.5682 & 0.8330 & -\\
t-SVD~\cite{zhang2014novel}& 0.4044 & 0.6117 & 0.8734 & 0.6439 & 0.7907 & 0.9526\\
FBCP~\cite{zhao2015bayesian}& 0.0462 & 0.0335 & 0.1663 & 0.0612 & 0.0634 & 0.3104\\
BRTF~\cite{zhao2016bayesian}& 0.0889 & 0.1437 & 0.1587 & 0.1338 & 0.1449 & 0.4732\\
Ours& \textbf{0.0285} & \textbf{0.0307} & \textbf{0.0351} & \textbf{0.0292} & \textbf{0.0320} & \textbf{0.0389}\\
\hline
\multicolumn{7}{c}{Mixture non-low-rank structure $\mathcal{E}$ (on-zero mean)}\\
\hline
 & \multicolumn{3}{c|}{rank = $5$} & \multicolumn{3}{c}{rank = $10$}\\
\hline
Method & $70\%$ & $80\%$ & $90\%$ & $70\%$ & $80\%$ & $90\%$\\
\hline
FaLRTC~\cite{liu2013tensor} & 0.2902 & 0.4363 & 0.7834 & 0.5494 & 0.7413 & 0.9240\\
HaLRTC~\cite{liu2013tensor} & 0.2901 & 0.4362 & 0.7842 & 0.5495 & 0.7415 & 0.9240\\
RPTC$_{\rm{scad}}$~\cite{zhao2015novel} & 0.2903 & 0.4906 & 0.8711 & 0.5907 & 0.7704 & 0.9288\\
TMac~\cite{xu2013parallel} & 0.8340 & 0.9006 & 0.9522 & 0.8374 & 0.8923 & 0.9511\\
STDC~\cite{chen2014simultaneous} & 0.3233 & 0.5527 & 0.8995 & 0.6512 & 0.8688 & 0.9423\\
t-SVD~\cite{zhang2014novel} & 0.5158 & 0.7116 & 0.9245 & 0.6759 & 0.8110 & 0.9754\\
FBCP~\cite{zhao2015bayesian} & 0.2199 & 0.2416 & 0.2746 & 0.2247 & 0.2316 & 0.2996\\
BRTF~\cite{zhao2016bayesian} & 0.1976 & 0.2326 & 0.2547 & 0.1837 & 0.2340 & 0.5642\\
Ours & \textbf{0.1829} & \textbf{0.1960} & \textbf{0.2038} & \textbf{0.1796} & \textbf{0.1914} & \textbf{0.2055}\\
\hline
\end{tabu}
\end{center}
$^*$'-' denotes recovery failure.
\label{table:performance}
\end{table}

\subsubsection{Ability of fitting a wide range of $\mathcal{E}$s}
Given $\mathcal{Y}_{\Omega}$ with $70\%$ missing entries, the proposed model can estimate the probability density function (PDF) (i.e., distribution) of $\mathcal{E}$ by inferring the posterior mean of parameters $\mbox{\boldmath{$\mu$}}_e$, $\mbox{\boldmath{$\tau$}}_e$ and $\mbox{\boldmath{$\pi$}}$ in tensor completion. The quantitative comparison between the estimated parameters in $5$ different PDFs of $\mathcal{E}$ and the corresponding ground truth are given in Table~\ref{table:Fittness}. It can be seen that the proposed model produces accurate estimations to these parameters, viz., the proposed model is able to well fit $\mathcal{E}$ with different PDFs. To make this point more clear, we plot the estimated PDFs of $\mathcal{E}$ and the corresponding ground truth in Figure~\ref{fig:mtpi}. We can find that the estimated PDFs comply well with the ground truth, especially in the cases of complicated mixture $\mathcal{E}$. Moreover, the proposed model performs stably with different low rank structures (e.g., ${\rm{rank}}(\mathcal{X}) = 5, 10$ in Table~\ref{table:Fittness} and Figure~\ref{fig:mtpi}) in terms of fitting $\mathcal{E}$. Therefore, we can conclude that the proposed model is able to fit a wide range of $\mathcal{E}$, which is consistent with the theoretical analysis in Remark~\ref{remark:remark1}.

\subsubsection{Effectiveness of data-adaptive tensor model}\label{subsubsec: Effect}
In the proposed data-adaptive tensor model, we model the low-rank structure and the complex non-low-rank structure separately. In this part, we concentrate on validating the effectiveness of each part in the model. To this end, we compare the proposed model with its two variants, namely 'Ous\_{nx}' and 'Ours\_{ne}'. In 'Ours\_{nx}', we remove the low-rank structure $\mathcal{X}$ and model the whole tensor with mixture of Gaussians, while 'Ours\_{ne}' sets the no-low-rank structure $\mathcal{E} = 0$. In the experiments, we employ these three methods to recover two synthetic tensors with different missing ratios, which consists of two respective $\mathcal{X}$s of ${\rm{rank}}(\mathcal{X}) = 5,10$ and a non-zero mean mixture $\mathcal{E}$. Since 'Ours\_{nx}' only utilizes the mixture of Gaussians to model the whole tensor, different $D$ (i.e., the number of Gaussian components) are adopted to well fit the complex tensor structure. Their recovery performance is measured by the relative reconstruction error (RRE), i.e., $\|\mathcal{L} - \hat{\mathcal{L}}\|_F / \|\mathcal{L}\|_F$, where $\hat{\mathcal{L}}$ is the estimation of the true $\mathcal{L}$. The comparison results is presented in Table~\ref{table:effectivenss}.

Without modelling the non-low-rank structure, the performance of 'Ours\_{ne}' drops obviously compared with the proposed model. This demonstrates that modelling the non-low-rank structure separately can benefit performance improvement. In addition, ‘Ours\_{nx}’ fails to recover the tensor in most cases. The reason is intuitive. To solve the ill-posed tensor completion problem, we have to well represent the low-rank structure as regularization. However, mixture of Gaussians is too flexible to regularize the ill-posed problem and thus causes failure completion, especially when the tensor is highly structured (of lower rank = $5$), it fails in all cases. Hence, we turn to model the low-rank structure by exploiting the sparsity in the CP factorization, which is crucial for tensor completion. Therefore, we can conclude that the proposed data-adaptive tensor model which exploits the low-rank-structure and the non-low-rank structure separately, is effective in tensor completion.

\subsubsection{Validation of convergence}
It has been shown that samples generated in Gibbs sampling coverage to a stable probabilistic distribution~\cite{bishop2006pattern}. To illustrate this point more intuitively, with two same synthetic tensors adopted in Section~\ref{subsubsec: Effect}, we plot the RRE curves of the proposed model versus the Gibbs sampling iteration number with different missing ratios in Figure~\ref{fig:convergence}. We can find that the proposed model converges in hundreds of iterations in different cases.

\subsubsection{Evaluation on recovery performance}
We compare the proposed model with $8$ competitors in tensor completion. To comprehensively compare them, we conduct experiments on tensors consisting of a low rank $\mathbf{X}$ with rank = $5,10$ and i) a zero non-low-rank $\mathcal{E}$; ii) a sparse non-low-rank $\mathcal{E}$; iii) a mixture non-low-rank $\mathcal{E}$ (non-zero-mean). Their recovery results are reported in Table~\ref{table:performance}. We can find that the proposed method surpass the other competitors in most cases. Moreover, the superiority of the proposed method is enhanced with the increasing missing ratio or a simpler $\mathcal{E}$. For example, when $\mathcal{E}$ comes from a mixture distribution and the missing ratio is $90\%$, the improvement over the second best, i.e., FBCP, is up to $0.09$ in RRE when tensor rank is $10$. When $\mathcal{E}$ is set zero without any other changes, the improvement is up to $0.38$. Besides, the proposed method performs more stably than others with the changing missing ratios. For example, when $\mathcal{E} = 0$ and the tensor rank is $10$, the RRE of the proposed model lies below $0.01$. In contrast, the RRE of FaLRTC increases from $0.4$ to $0.9$. Based on those observed results, we can conclude that the proposed model outperforms the other competitors on the synthetic tensor data.


\begin{figure}
\setlength{\abovecaptionskip}{0pt}
\begin{center}
\includegraphics[height=0.64in,width=0.64in,angle=0]{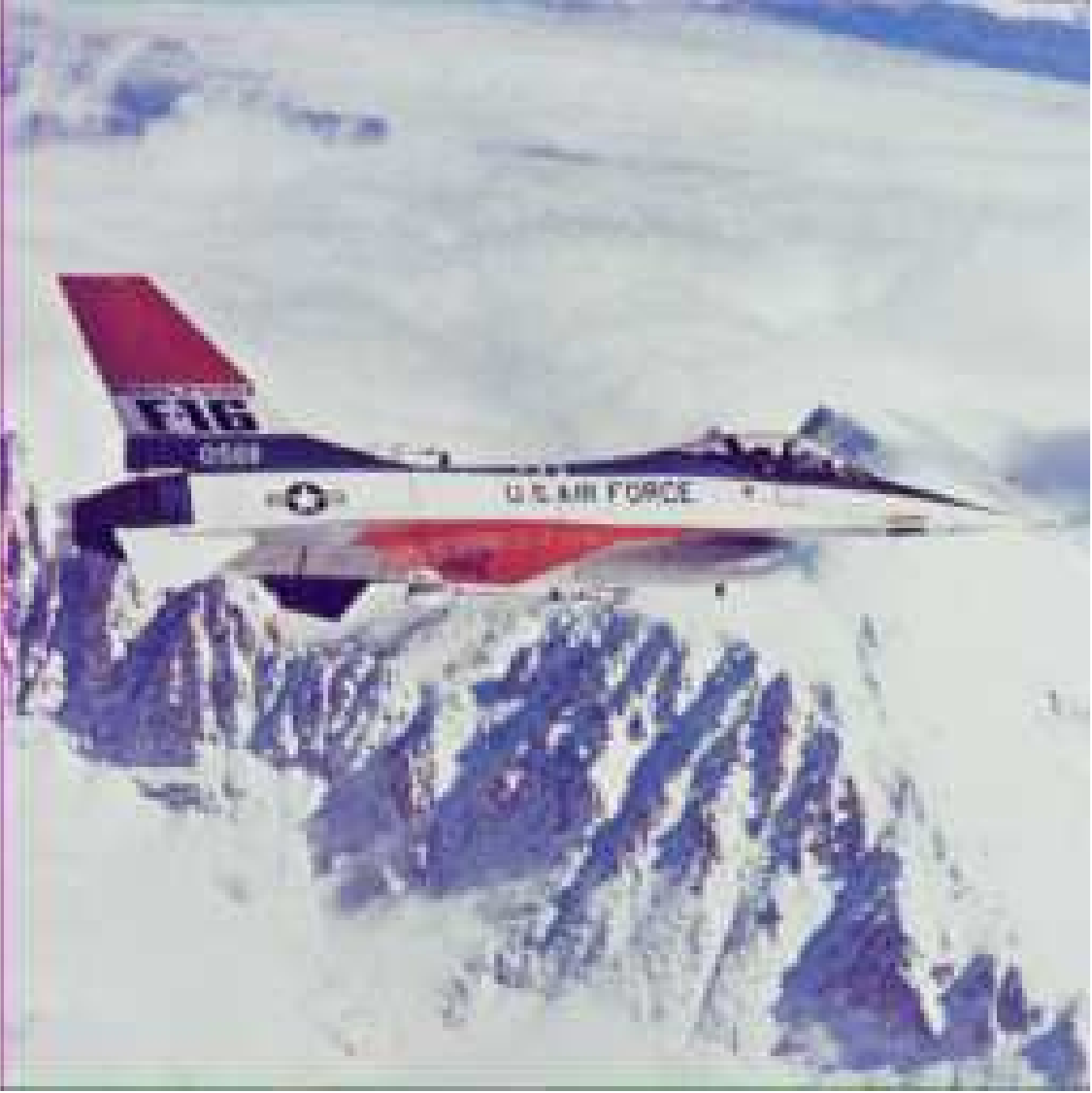}
\hspace{-0.16cm}
\includegraphics[height=0.64in,width=0.64in,angle=0]{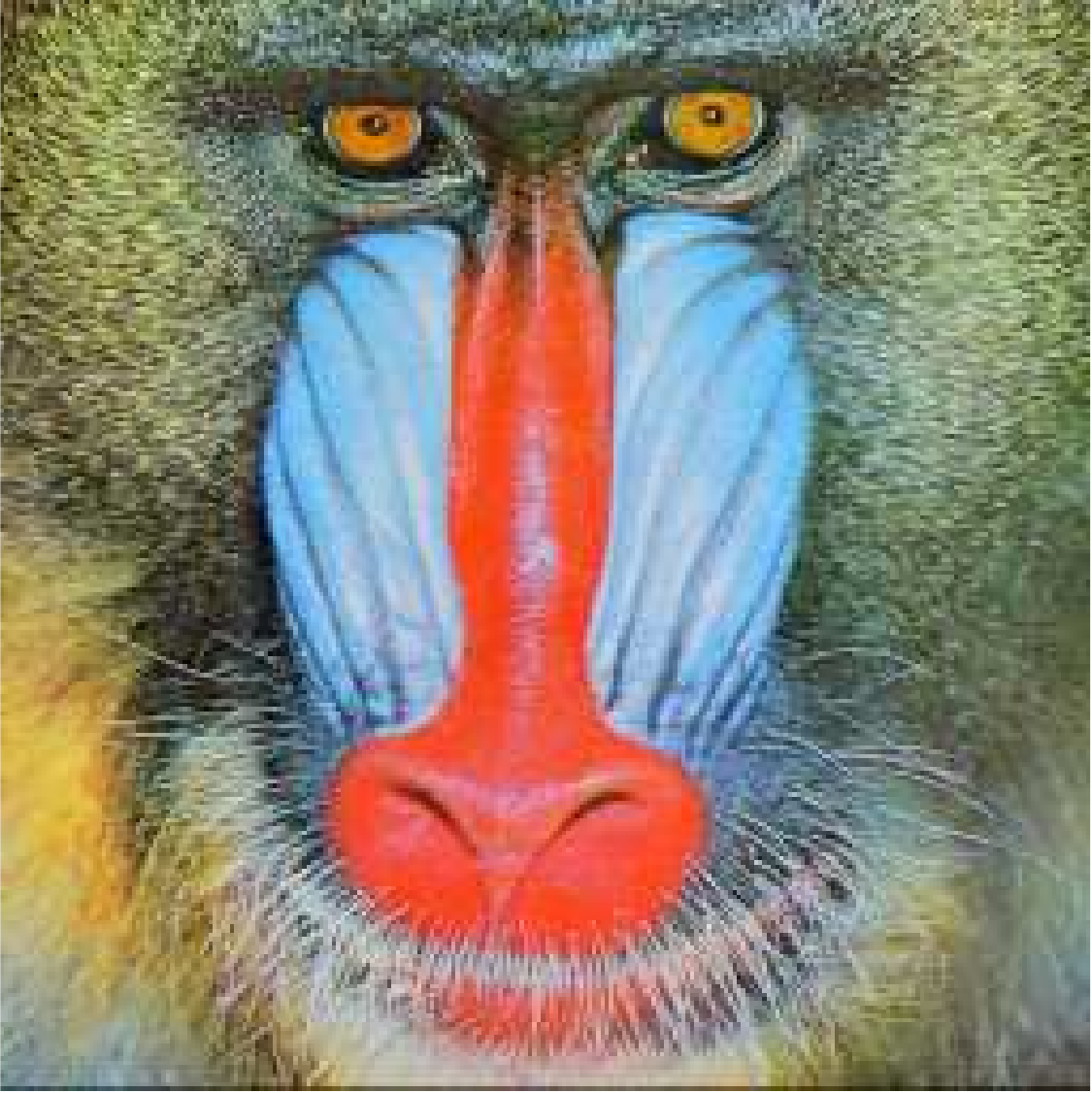}
\hspace{-0.16cm}
\includegraphics[height=0.64in,width=0.64in,angle=0]{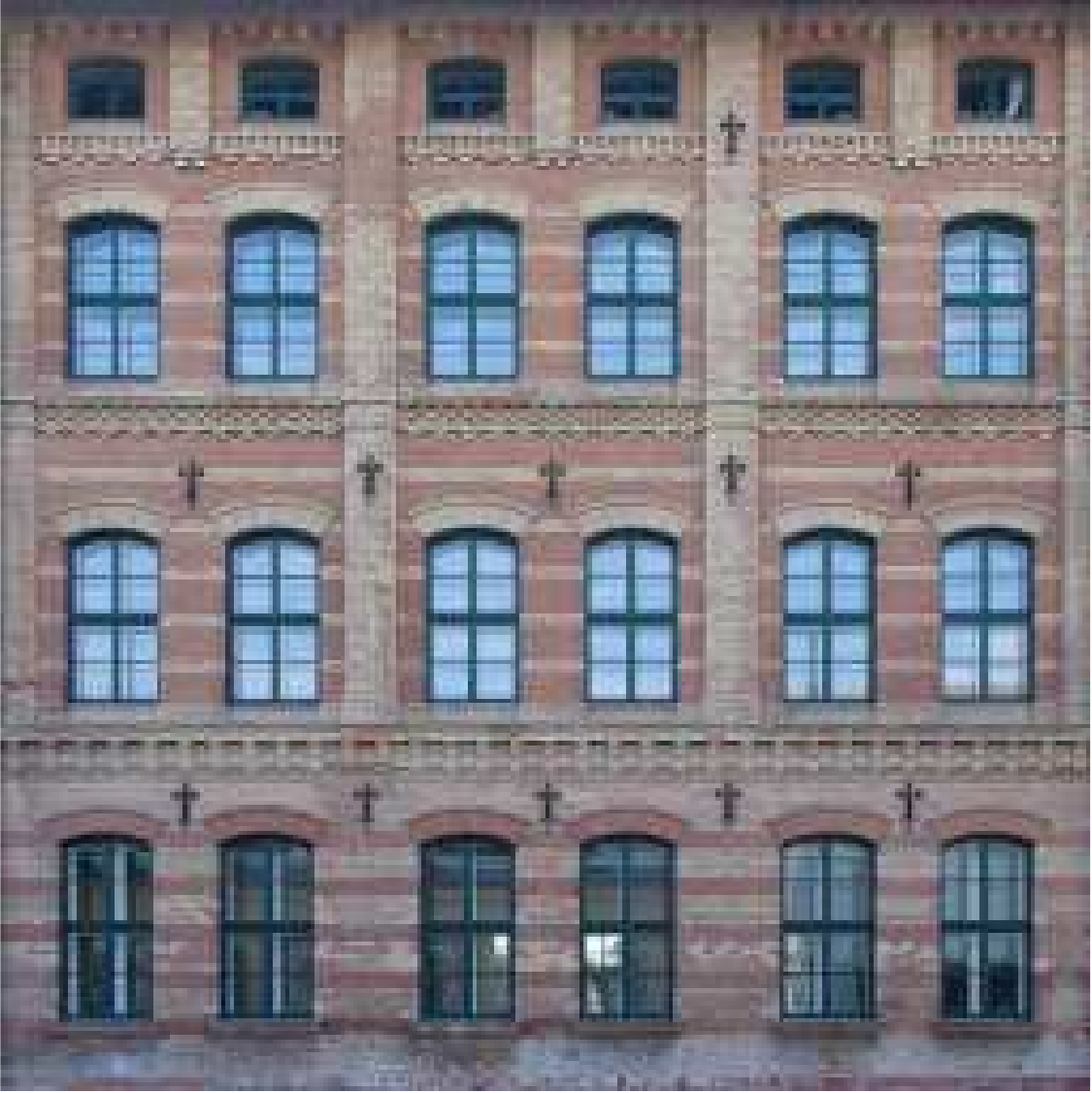}
\hspace{-0.16cm}
\includegraphics[height=0.64in,width=0.64in,angle=0]{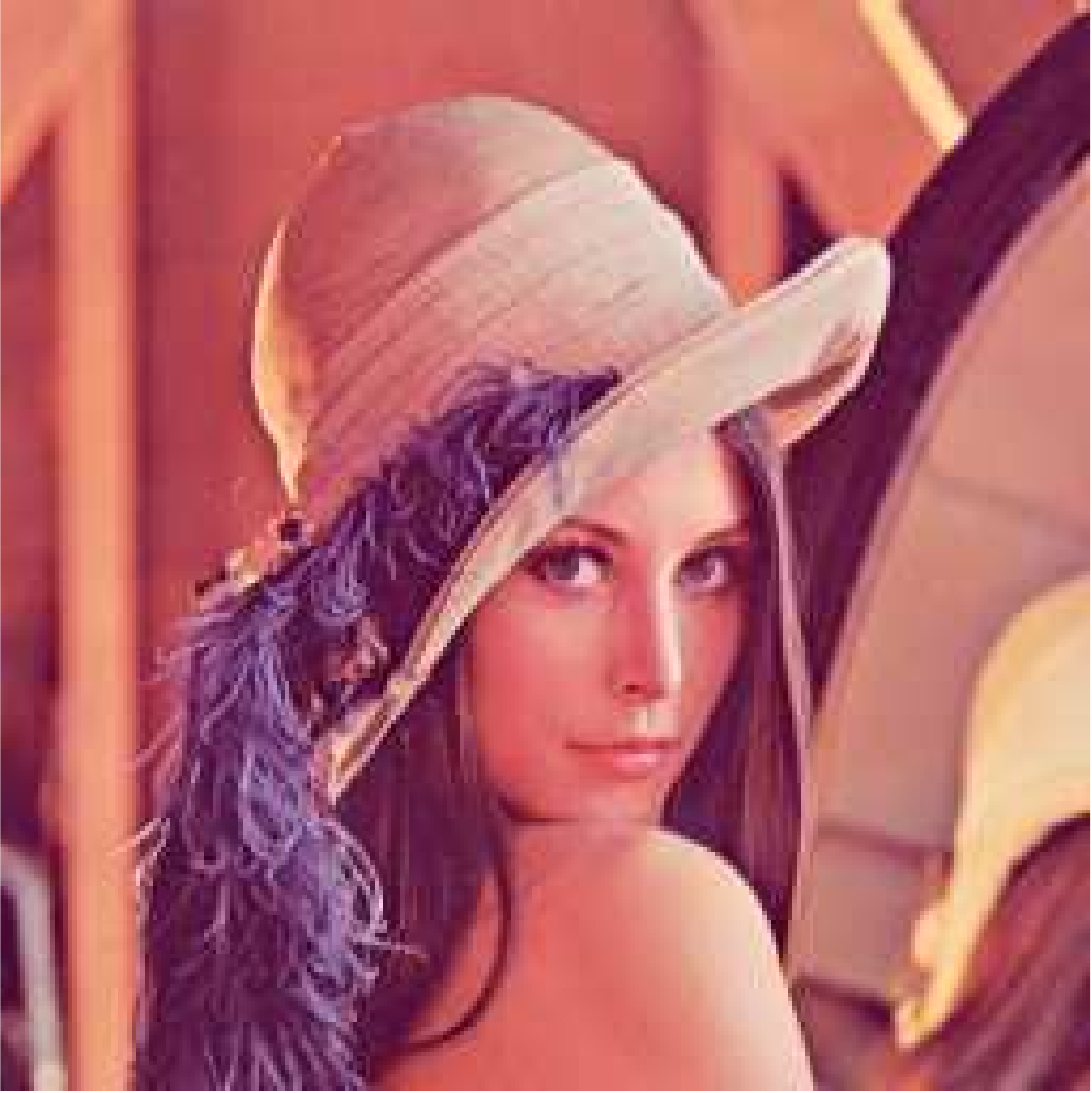}
\hspace{-0.16cm}
\includegraphics[height=0.64in,width=0.64in,angle=0]{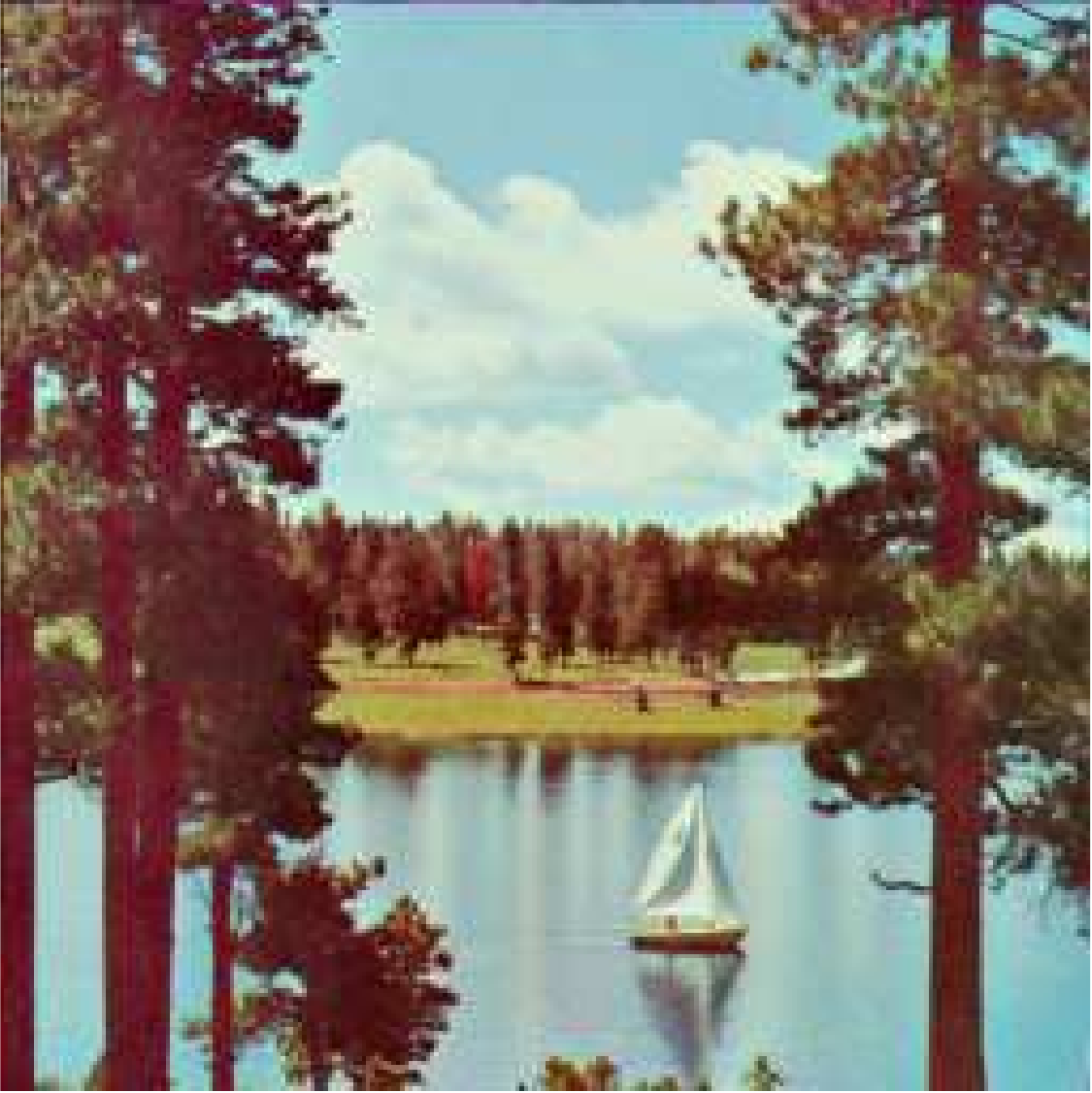}\
\\
\vspace{0.05cm}
\includegraphics[height=0.64in,width=0.64in,angle=0]{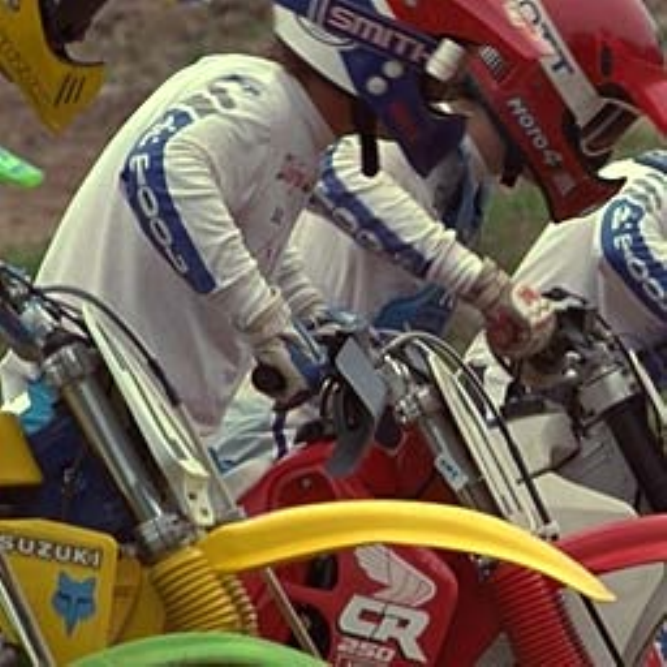}
\hspace{-0.16cm}
\includegraphics[height=0.64in,width=0.64in,angle=0]{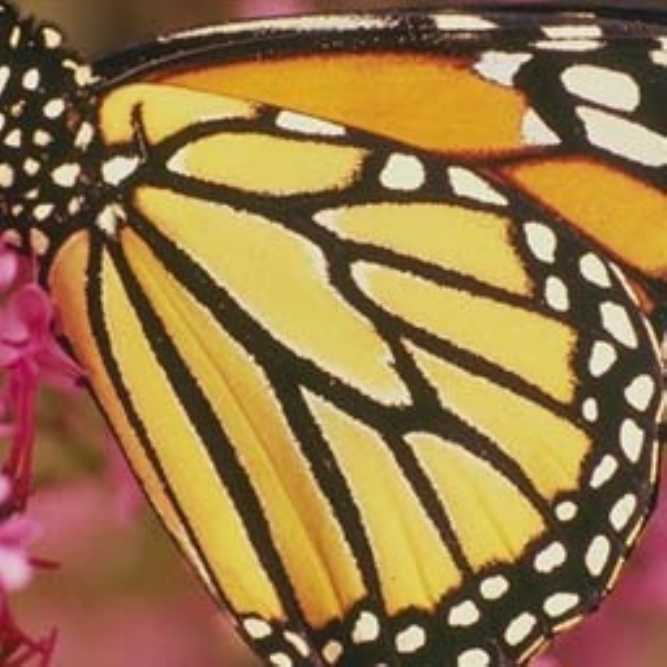}
\hspace{-0.16cm}
\includegraphics[height=0.64in,width=0.64in,angle=0]{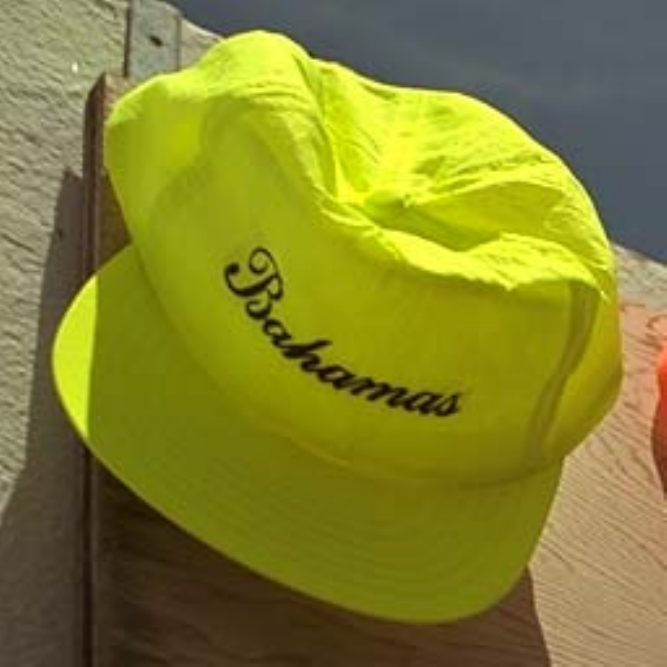}
\hspace{-0.16cm}
\includegraphics[height=0.64in,width=0.64in,angle=0]{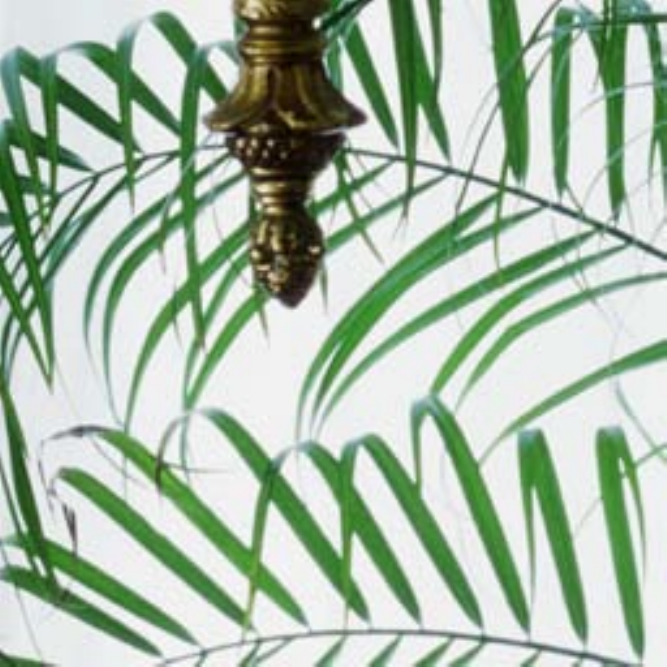}
\hspace{-0.16cm}
\includegraphics[height=0.64in,width=0.64in,angle=0]{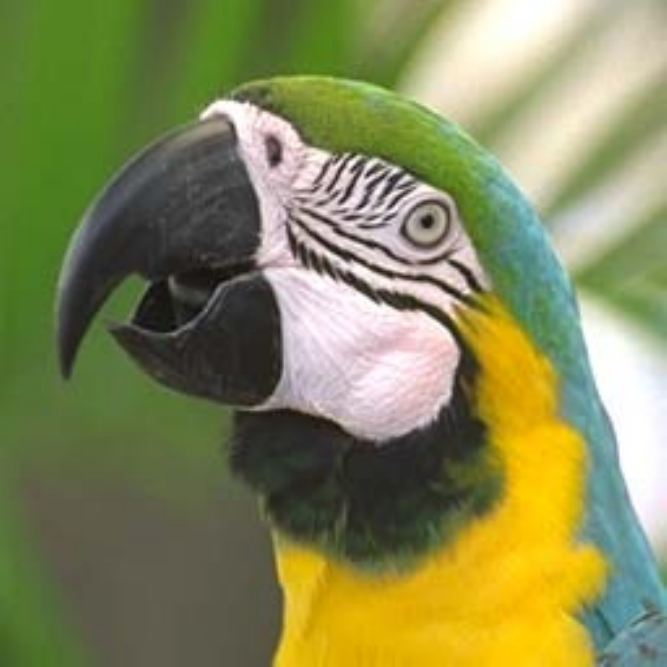}
\end{center}
\caption{Ground truth of 10 benchmark images.}
\label{fig:inpainting}
\end{figure}

\begin{table*}\footnotesize
\caption{Average RRE, PSNR and SSIM on 10 benchmark images with different missing ratios.}
\renewcommand{\arraystretch}{1.1}
\begin{center}
\begin{tabu} to 1\textwidth {X[1.9,l]|X[c]|X[c]|X[c]|X[c]|X[c]|X[c]|X[c]|X[c]|X[c]|X[c]|X[c]|X[c]}
\hline
\multirow{2}{*}{Method} & \multicolumn{3}{c|}{60\%} & \multicolumn{3}{c|}{70\%} & \multicolumn{3}{c|}{80\%} & \multicolumn{3}{c}{90\%}\\
\cline{2-13}
 & RRE & PSNR & SSIM & RRE & PSNR & SSIM & RRE & PSNR & SSIM & RRE & PSNR & SSIM\\
\hline
FaLRTC~\cite{liu2013tensor} & 0.1003 & 25.6337 & 0.7911 & 0.1289 & 23.4559 & 0.7024 & 0.1704 & 21.0291 & 0.5797 & 0.2445 & 17.8378 & 0.4068\\
HaLRTC~\cite{liu2013tensor} & 0.0995 & 25.7348 & 0.7916 & 0.1281 & 23.5329 & 0.7010 & 0.1695 & 21.0939 & 0.5757 & 0.2430 & 17.9096 & 0.3997\\
RPTC$_{\rm{scad}}$~\cite{zhao2015novel} & 0.0859 & 27.0793 & 0.8232 & 0.1336 & 24.1529 & 0.7106 & 0.1582 & 21.9985 & 0.6022 & 0.2210 & 18.9062 & 0.4198\\
TMac~\cite{xu2013parallel} & 0.1408 & 22.7919 & 0.6465 & 0.1517 & 22.1323 & 0.6036 & 0.1700 & 21.1231 & 0.5375 & 0.2940 & 16.1311 & 0.2627\\
STDC~\cite{chen2014simultaneous} & 0.0893 & 26.5850 & 0.8089 & 0.1078 & 25.0486 & 0.7579 & 0.1367 & 22.9630 & \textbf{0.6801} & 0.2326 & 17.9273 & \textbf{0.4913}\\
t-SVD~\cite{zhang2014novel} & 0.0915 & 26.4223 & 0.7832 & 0.1219 & 23.9397 & 0.6826 & 0.1649 & 21.3169 & 0.5472 & 0.2391 & 18.0432 & 0.3580\\
FBCP~\cite{zhao2015bayesian} & 0.0916 & 26.5030 & 0.7601 & 0.1137 & 24.5926 & 0.6866 & 0.1508 & 22.1263 & 0.5727 & 0.2192 & 18.8112 & 0.3878\\
BRTF~\cite{zhao2016bayesian} & 0.2685 & 16.7725 & 0.4421 & 0.2892 & 16.1276 & 0.4018 & 0.3089 & 15.5556 & 0.3611 & 0.3291 & 15.0054 & 0.3108\\
Ours & \textbf{0.0755} & \textbf{28.0223} & \textbf{0.8381} & \textbf{0.0966} & \textbf{25.9144} & \textbf{0.7676} & \textbf{0.1302} & \textbf{23.3615} & 0.6549 & \textbf{0.1970} & \textbf{19.8088} & 0.4549\\
\hline
\end{tabu}
\end{center}
\label{table:avginpainting}
\end{table*}

\begin{figure*}[htp]
\setlength{\abovecaptionskip}{0pt}
\begin{center}
\subfigure[Incomplete image]{\includegraphics[height=1.35in,width=1.35in,angle=0]{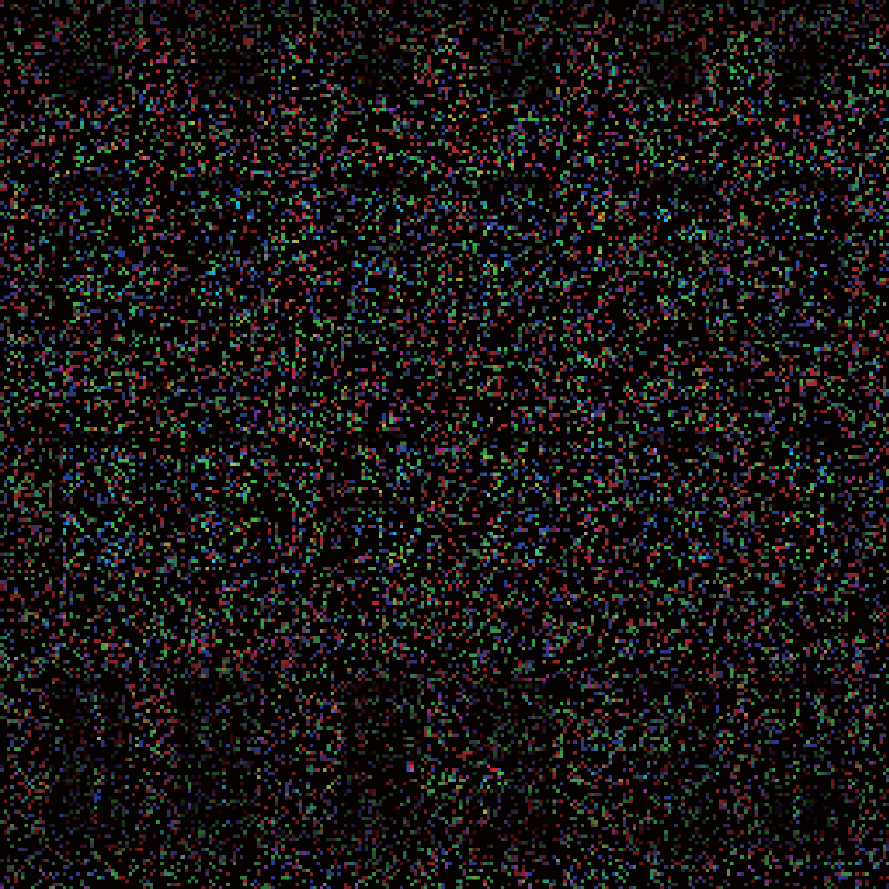}}
\hspace{-0.15cm}
\subfigure[FaLRTC~\cite{liu2013tensor}]{\includegraphics[height=1.35in,width=1.35in,angle=0]{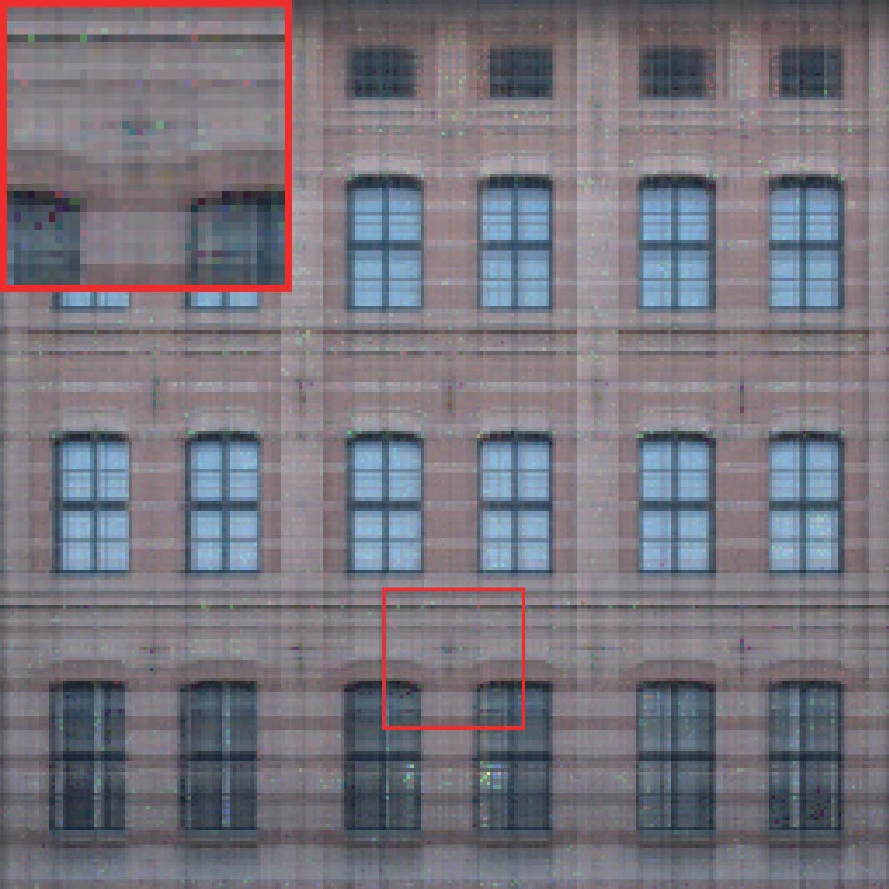}}
\hspace{-0.15cm}
\subfigure[HaLRTC~\cite{liu2013tensor}]{\includegraphics[height=1.35in,width=1.35in,angle=0]{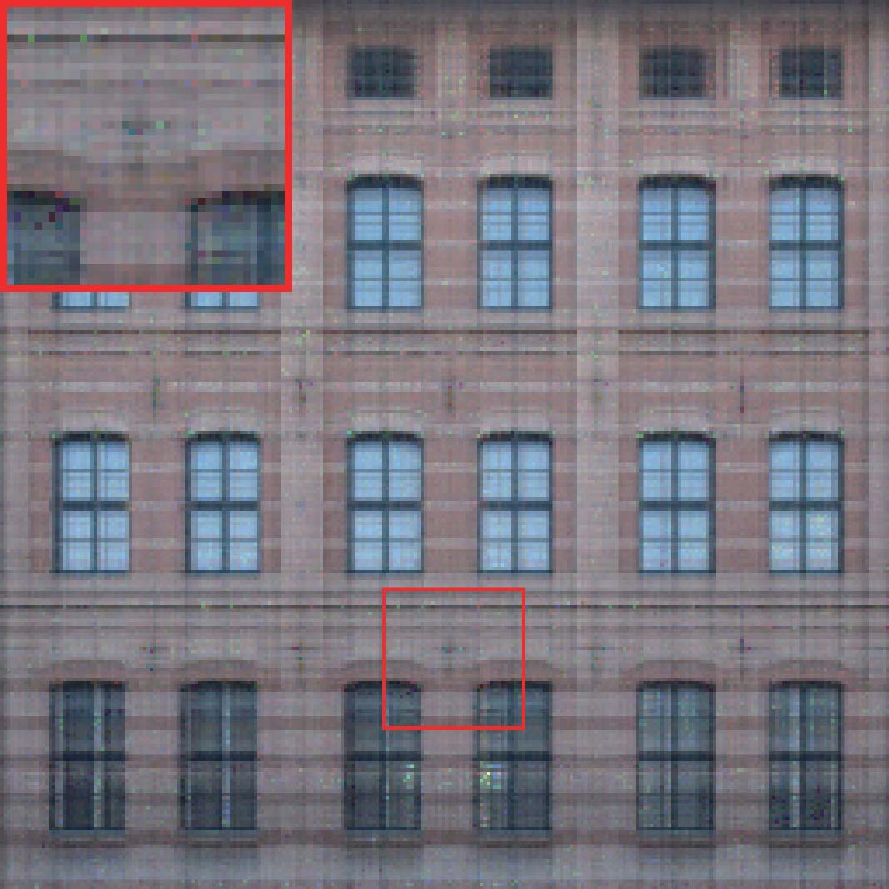}}
\hspace{-0.15cm}
\subfigure[RPTC$_{\rm{scad}}$~\cite{zhao2015novel}]{\includegraphics[height=1.35in,width=1.35in,angle=0]{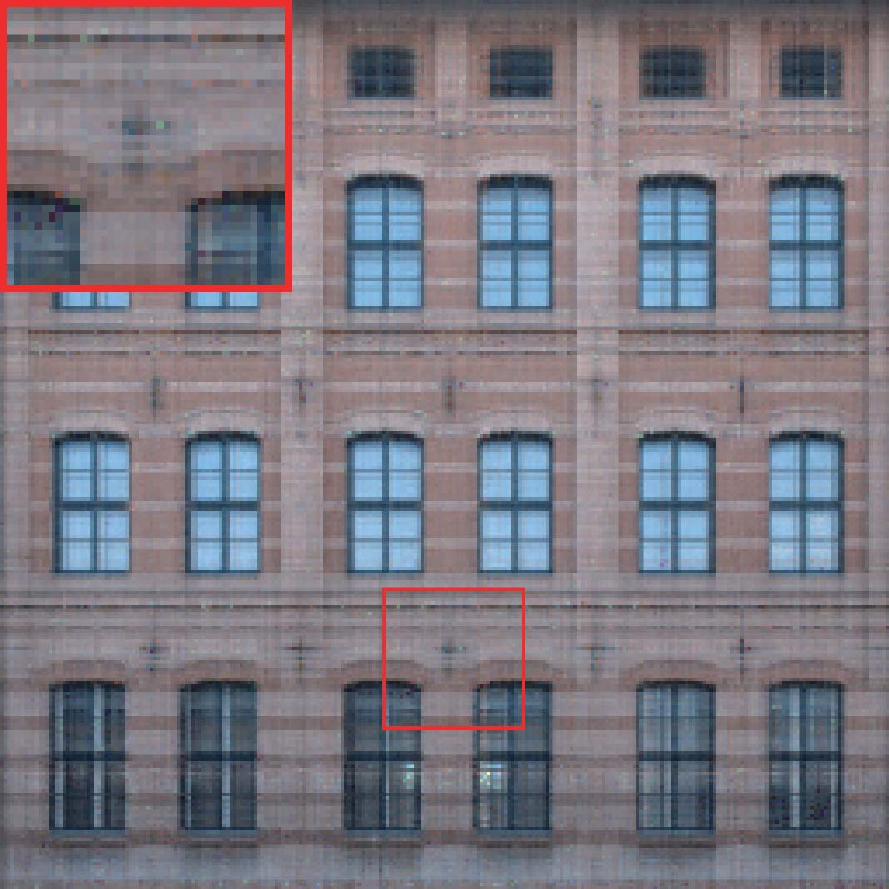}}
\hspace{-0.15cm}
\subfigure[TMac~\cite{xu2013parallel}]{\includegraphics[height=1.35in,width=1.35in,angle=0]{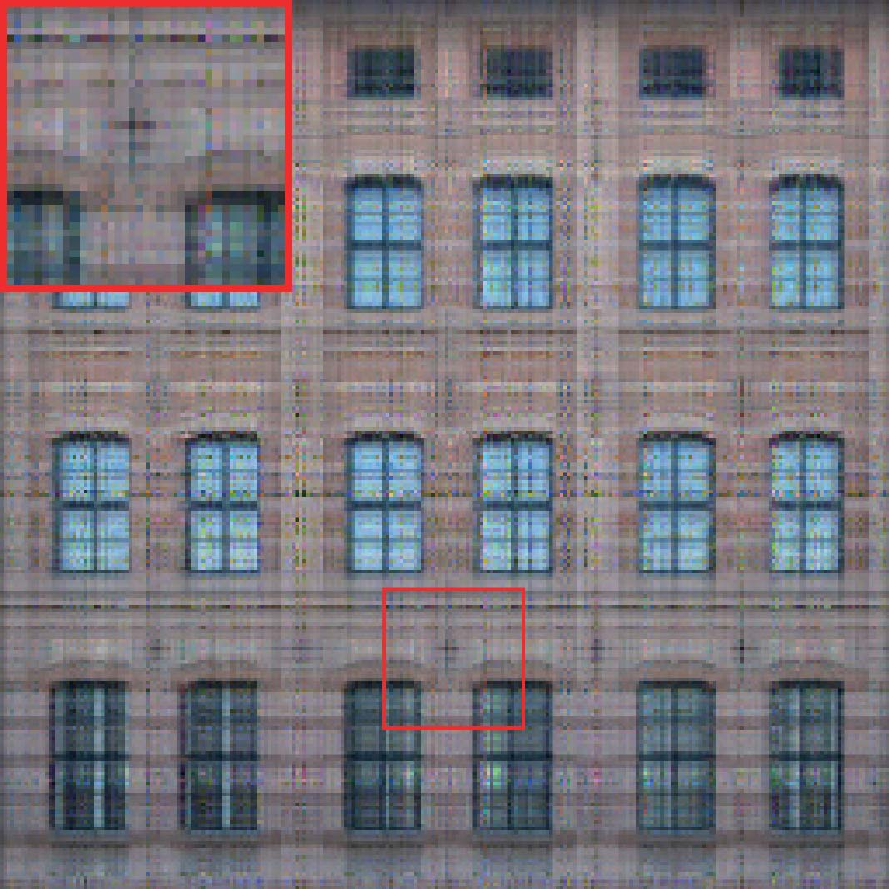}}
\\
\subfigure[STDC~\cite{chen2014simultaneous}]{\includegraphics[height=1.35in,width=1.35in,angle=0]{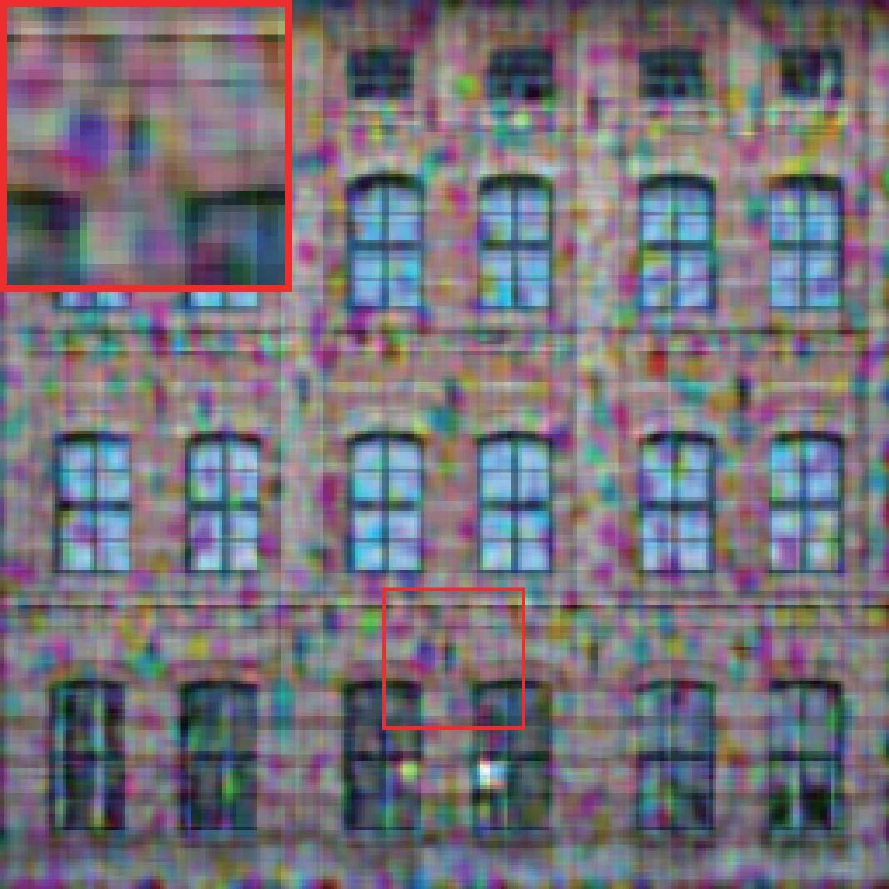}}
\hspace{-0.15cm}
\subfigure[t-SVD~\cite{zhang2014novel}]{\includegraphics[height=1.35in,width=1.35in,angle=0]{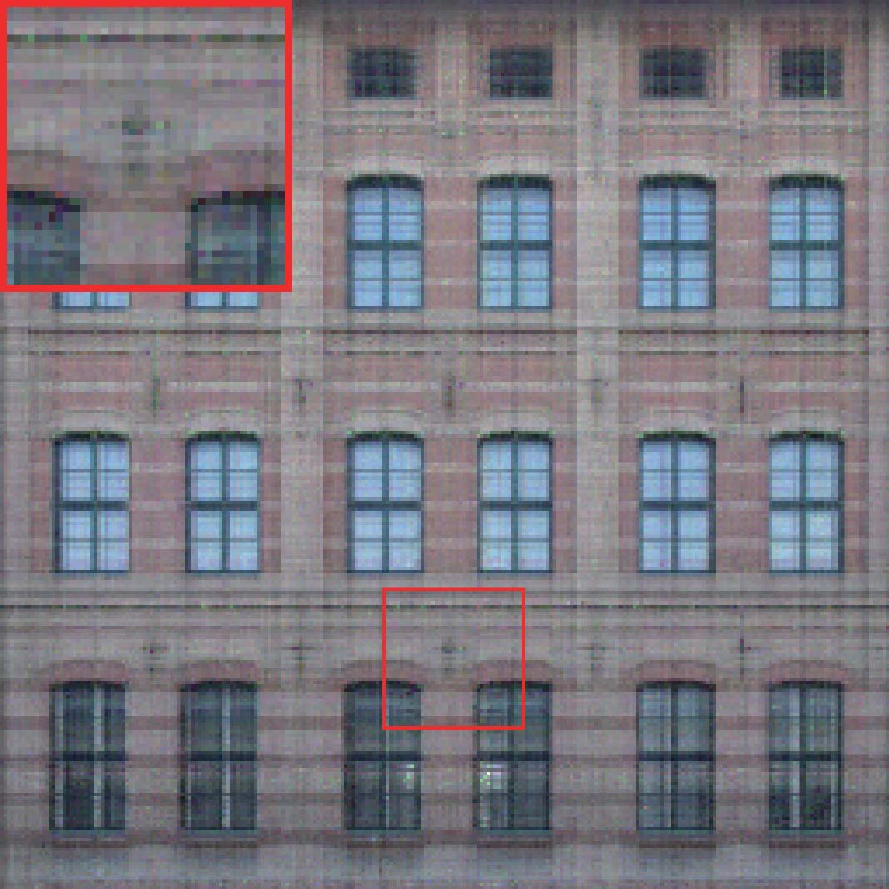}}
\hspace{-0.15cm}
\subfigure[FBCP~\cite{zhao2015bayesian}]{\includegraphics[height=1.35in,width=1.35in,angle=0]{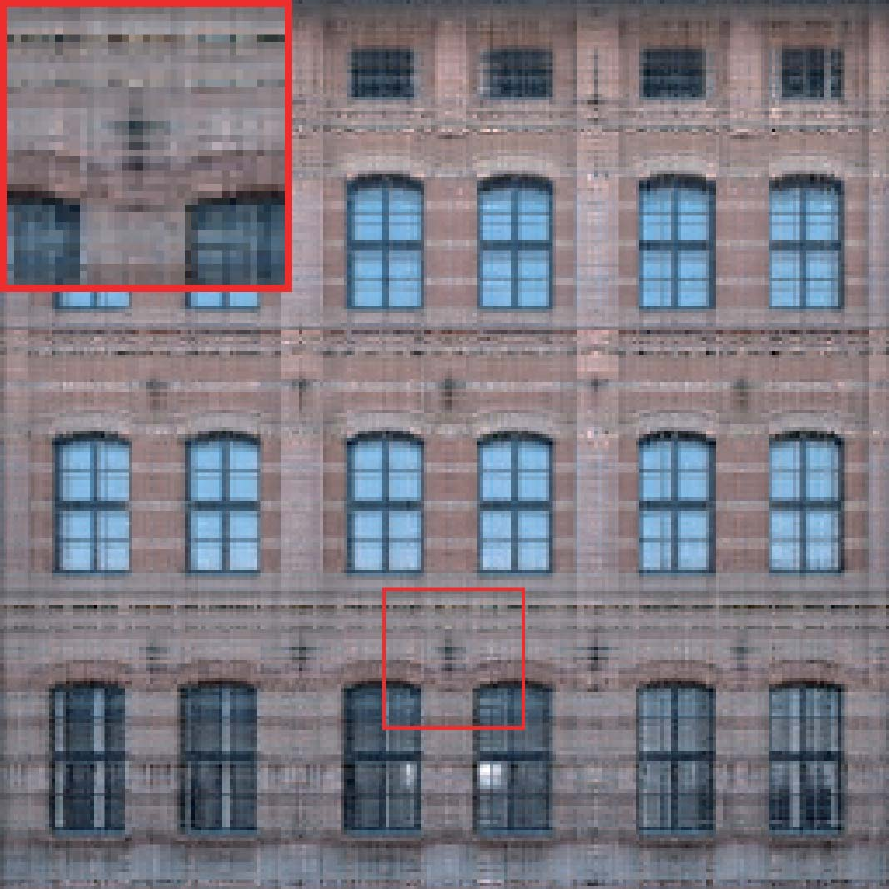}}
\hspace{-0.15cm}
\subfigure[BRTF~\cite{zhao2016bayesian}]{\includegraphics[height=1.35in,width=1.35in,angle=0]{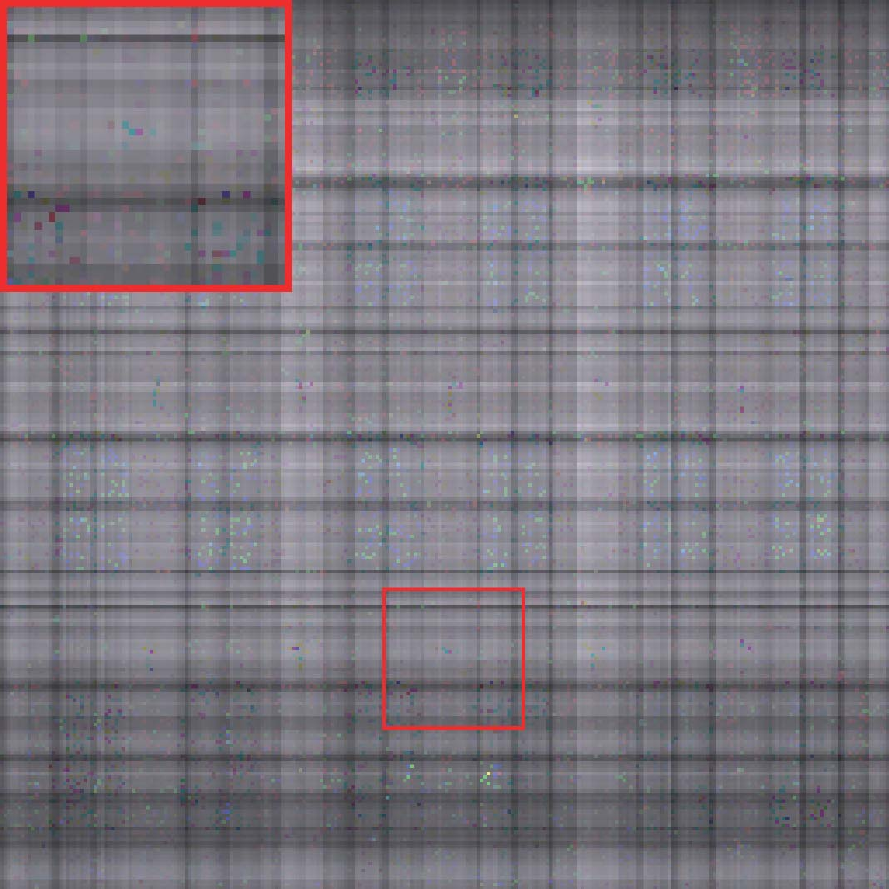}}
\hspace{-0.15cm}
\subfigure[Ours]{\includegraphics[height=1.35in,width=1.35in,angle=0]{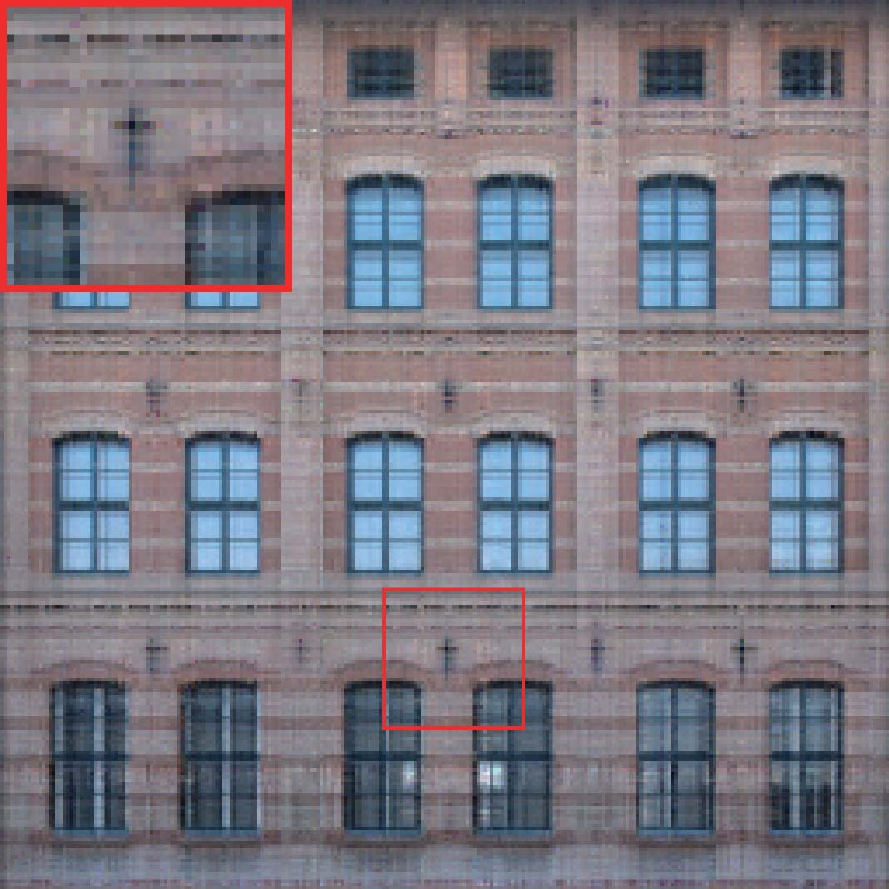}}\\
\end{center}
\vspace{-0.3cm}
\caption{Visual results of the 'facade' image, when missing ratio is $90\%$.}
\vspace{-0.3cm}
\label{fig:facade}
\end{figure*}

\begin{figure*}
\setlength{\abovecaptionskip}{0pt}
\begin{center}
\subfigure[Incomplete image]{\includegraphics[height=1.35in,width=1.35in,angle=0]{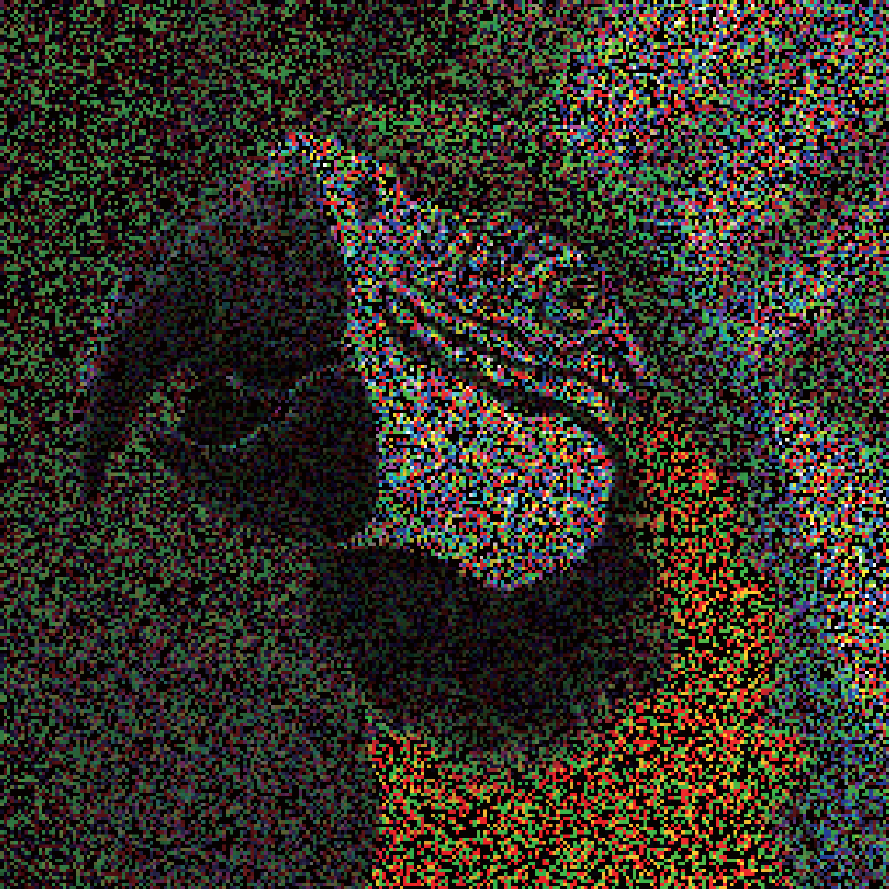}}
\hspace{-0.15cm}
\subfigure[FaLRTC~\cite{liu2013tensor}]{\includegraphics[height=1.35in,width=1.35in,angle=0]{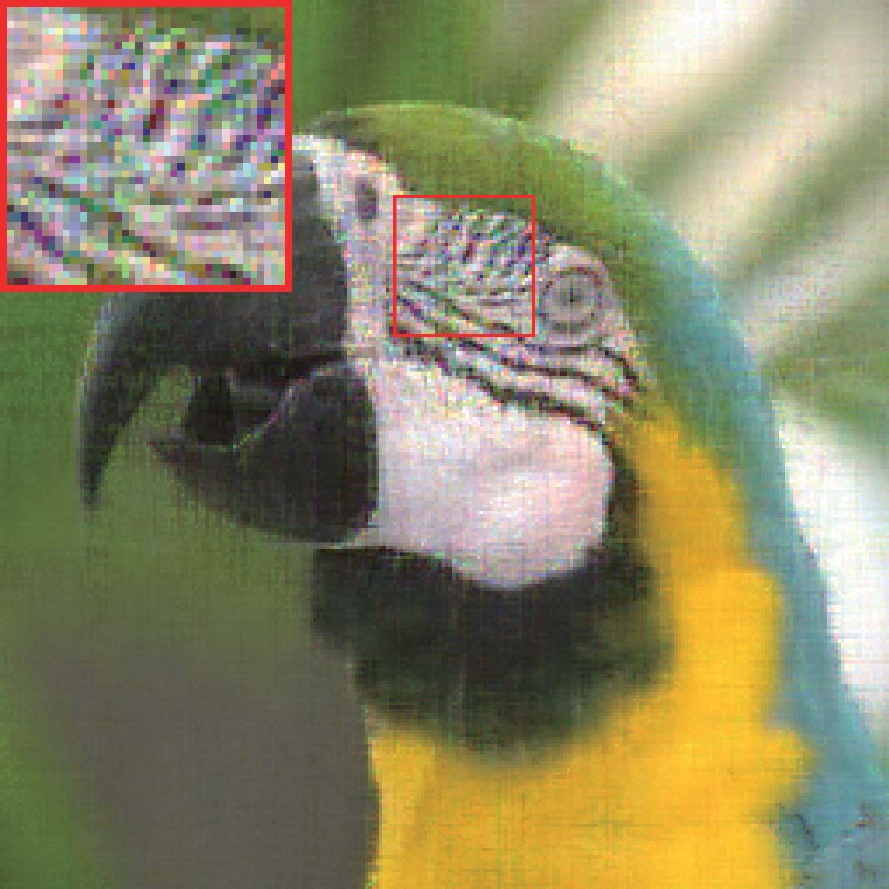}}
\hspace{-0.15cm}
\subfigure[HaLRTC~\cite{liu2013tensor}]{\includegraphics[height=1.35in,width=1.35in,angle=0]{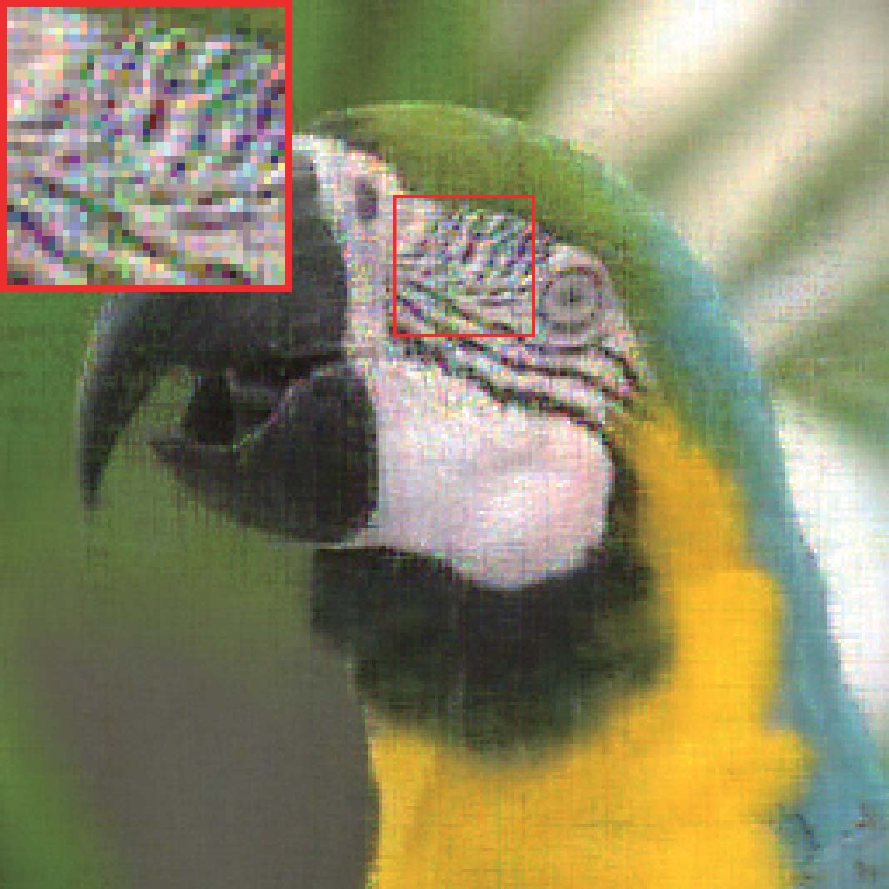}}
\hspace{-0.15cm}
\subfigure[RPTC$_{\rm{scad}}$~\cite{zhao2015novel}]{\includegraphics[height=1.35in,width=1.35in,angle=0]{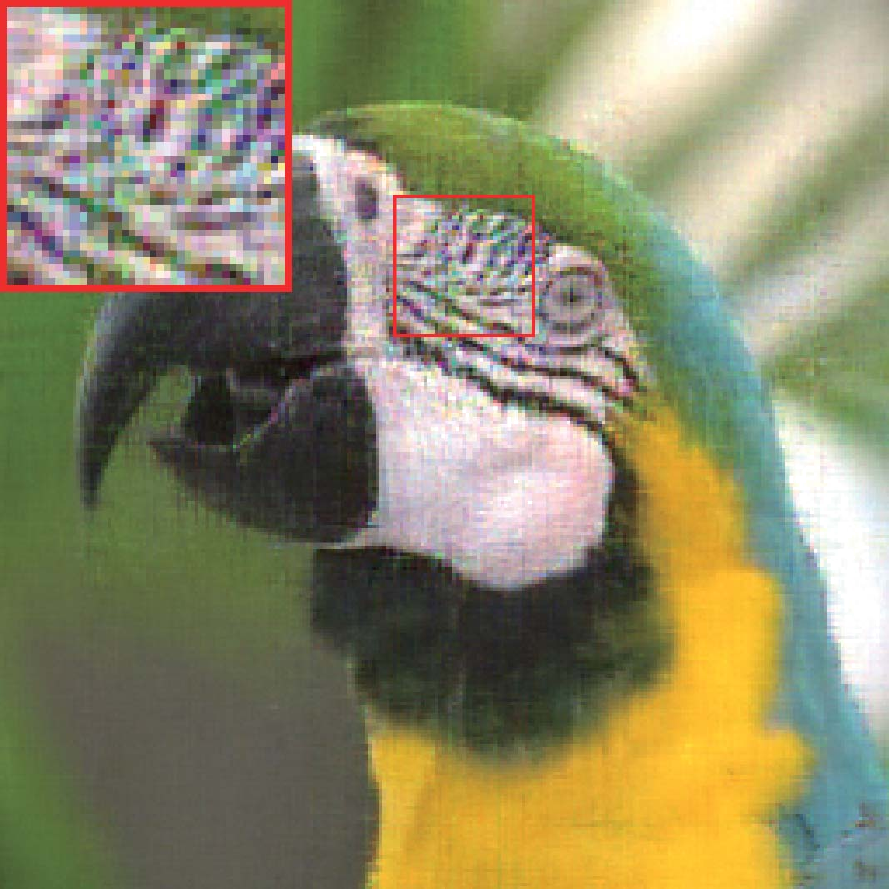}}
\subfigure[TMac~\cite{xu2013parallel}]{\includegraphics[height=1.35in,width=1.35in,angle=0]{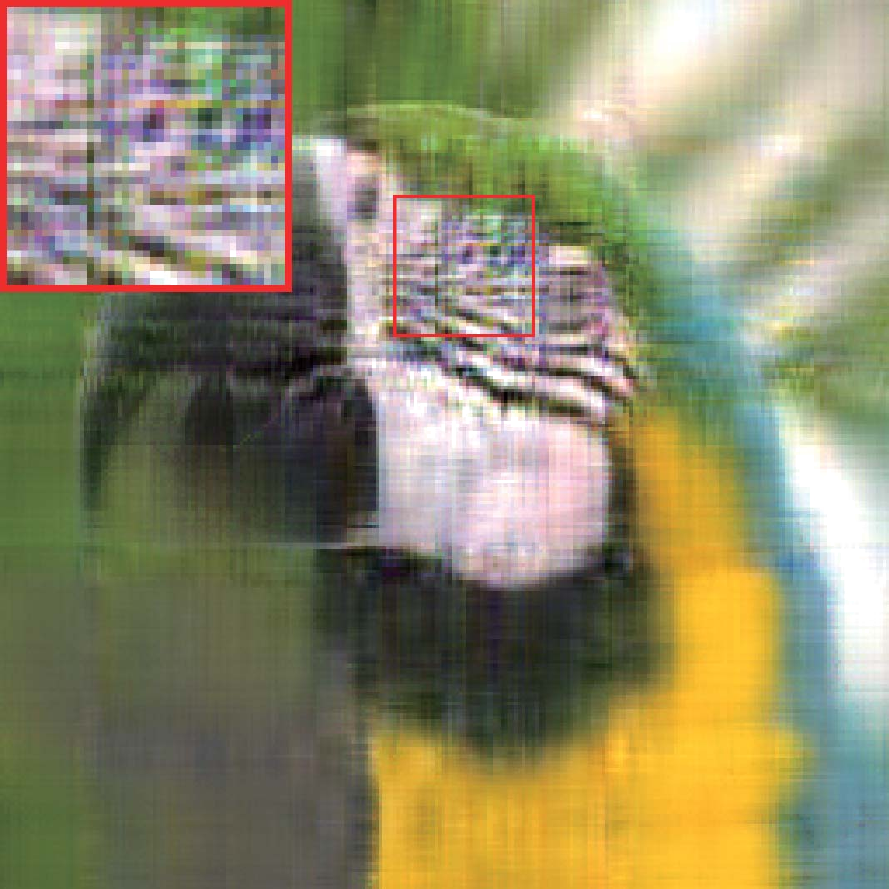}}
\hspace{-0.15cm}
\\
\subfigure[STDC~\cite{chen2014simultaneous}]{\includegraphics[height=1.35in,width=1.35in,angle=0]{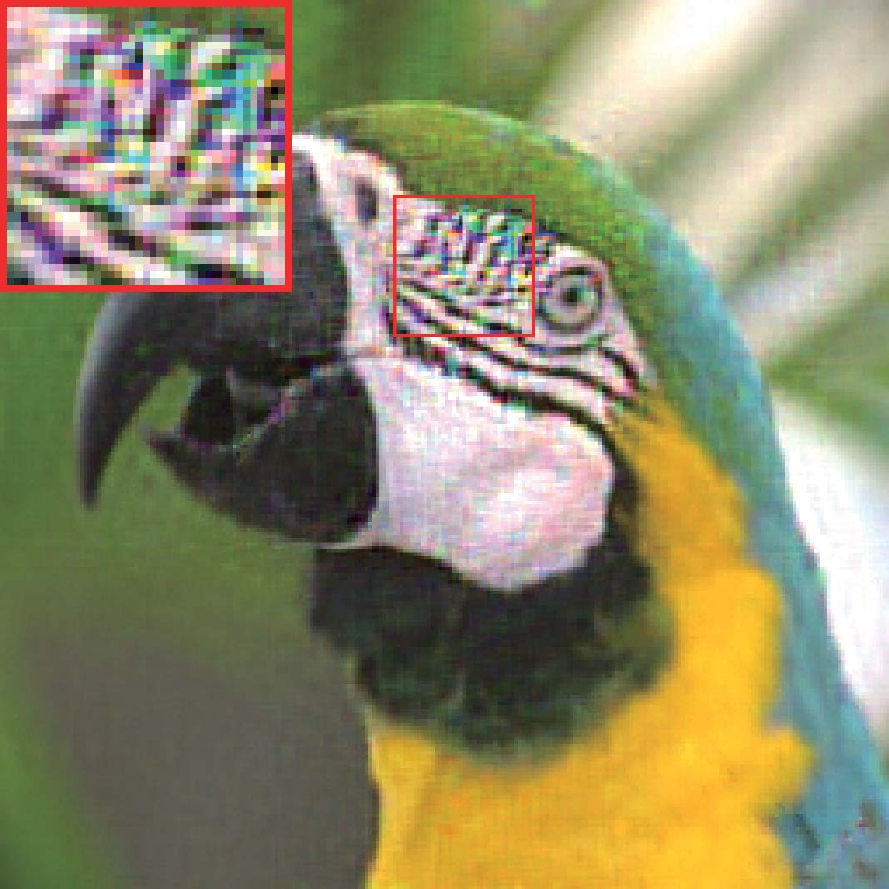}}
\hspace{-0.15cm}
\subfigure[t-SVD~\cite{zhang2014novel}]{\includegraphics[height=1.35in,width=1.35in,angle=0]{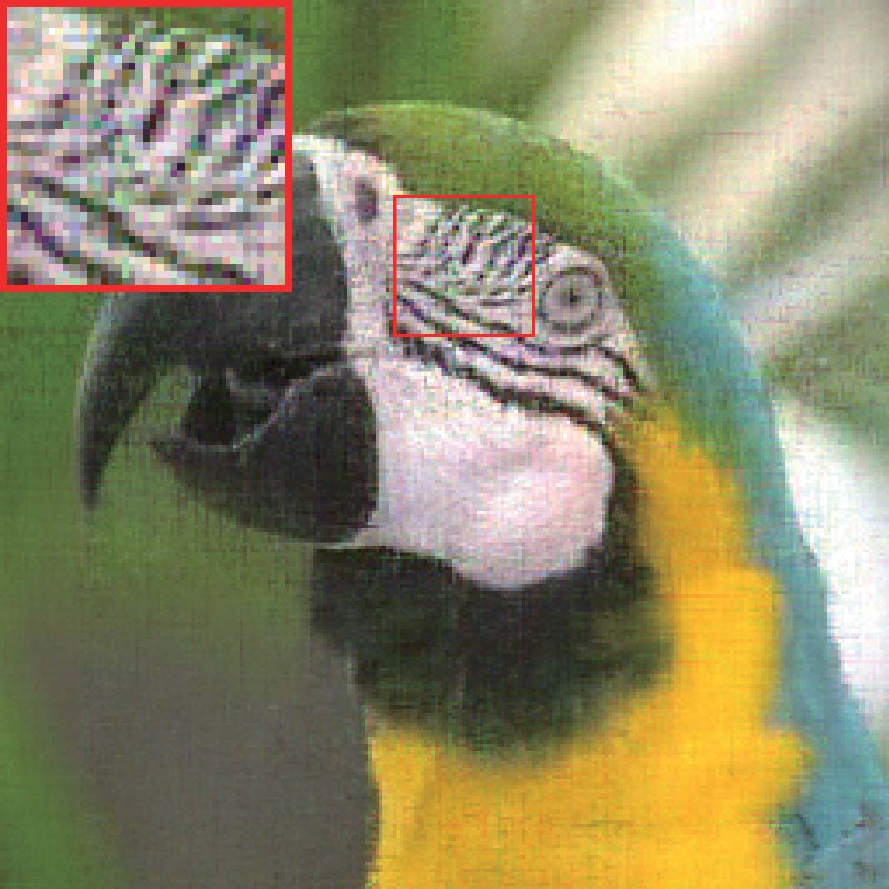}}
\hspace{-0.15cm}
\subfigure[FBCP~\cite{zhao2015bayesian}]{\includegraphics[height=1.35in,width=1.35in,angle=0]{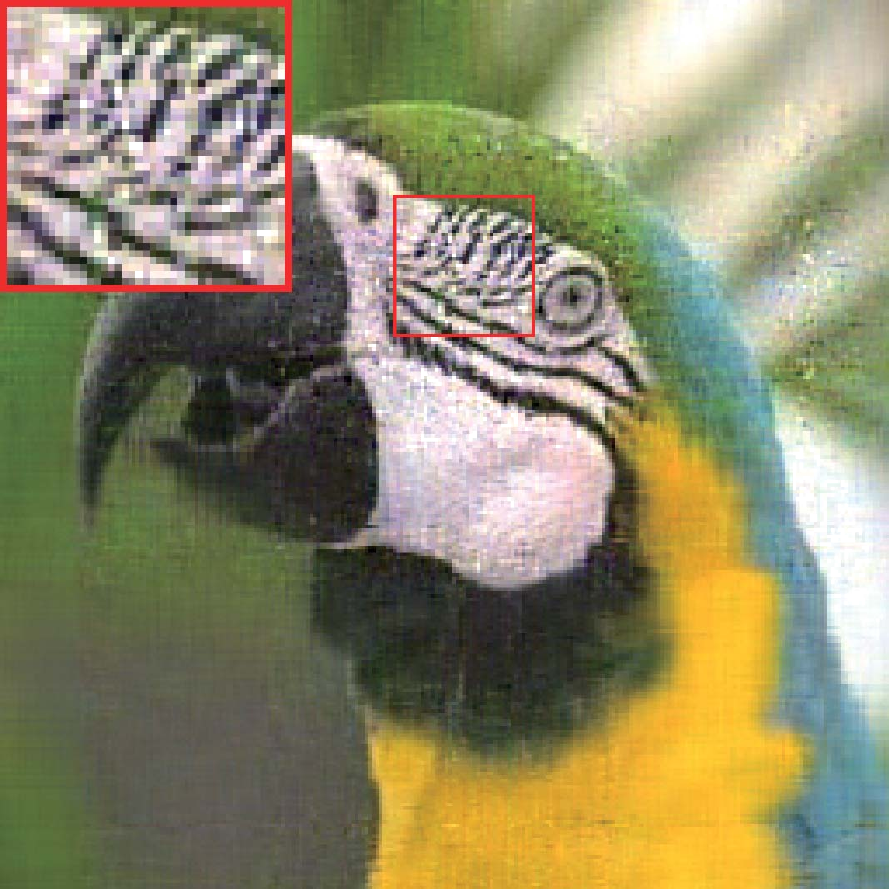}}
\hspace{-0.15cm}
\subfigure[BRTF~\cite{zhao2016bayesian}]{\includegraphics[height=1.35in,width=1.35in,angle=0]{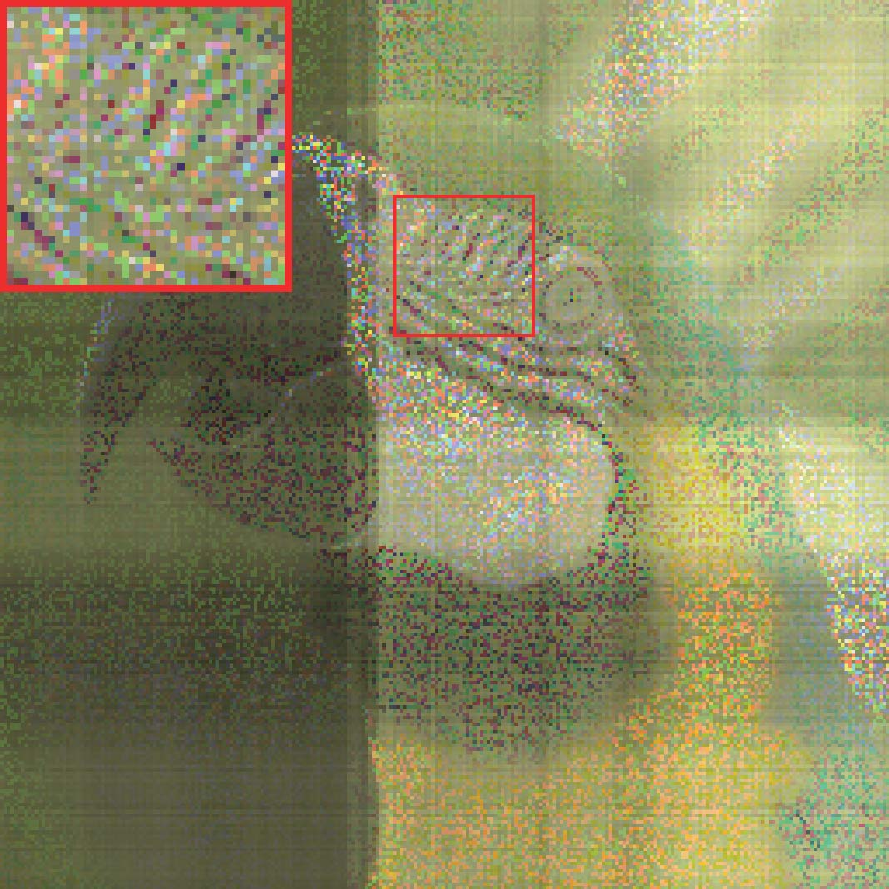}}
\hspace{-0.15cm}
\subfigure[Ours]{\includegraphics[height=1.35in,width=1.35in,angle=0]{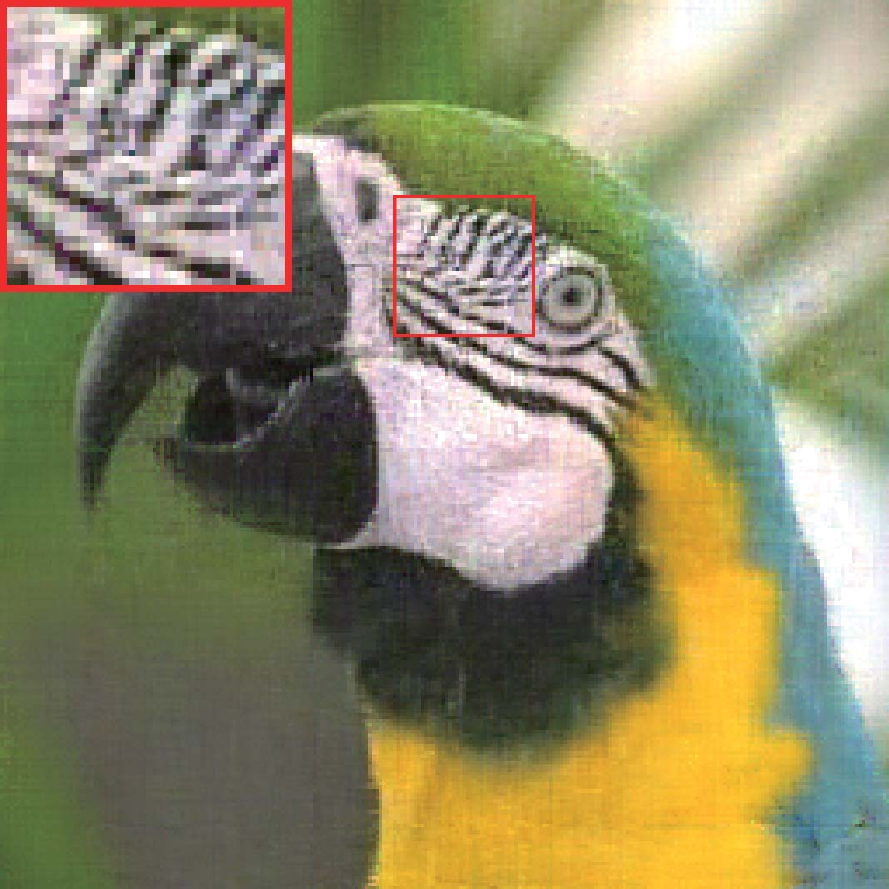}}\\
\end{center}
\caption{Visual results of the 'parrot' image from the top four methods, when missing ratio is $70\%$.}
\label{fig:parrot}
\end{figure*}

\begin{table}\footnotesize
\caption{RRE, PSNR and SSIM on two videos.}
\renewcommand{\arraystretch}{1.1}
\begin{center}
\begin{tabu} to 0.5\textwidth {X[2,l]|X[0.9,c]|X[c]|X[0.9,c]|X[0.9,c]|X[c]|X[0.9,c]}
\hline
\multirow{2}{*}{Method} & \multicolumn{3}{c|}{suzie} & \multicolumn{3}{c}{foreman}\\
\cline{2-7}
 & RRE & PSNR & SSIM & RRE & PSNR & SSIM\\
\hline
FaLRTC~\cite{liu2013tensor} & 0.0746 & 29.7584 & 0.7108 & 0.0537 & 29.1385 & 0.7484\\
HaLRTC~\cite{liu2013tensor} & 0.0746 & 29.7585 & 0.7108 & 0.0537 & 29.1385 & 0.7484\\
RPTC$_{\rm{scad}}$~\cite{zhao2015novel} & 0.0726 & 29.9943 & 0.7158 & 0.0493 & 29.8815 & 0.7635\\
TMac~\cite{xu2013parallel} & 0.0615 & 31.4253 & 0.8588 & 0.0702 & 26.8153 & 0.8103\\
STDC~\cite{chen2014simultaneous} & 0.0747 & 29.7401 & 0.7091 & 0.0543 & 29.0494 & 0.7431\\
FBCP~\cite{zhao2015bayesian} & 0.0426 & 34.6221 & 0.9103 & 0.0570 & 28.6269 & 0.8205\\
BRTF~\cite{zhao2016bayesian} & 0.0951 & 27.6432 & 0.6967 & 0.0975 & 23.9687 & 0.7117\\
Ours & \textbf{0.0381} & \textbf{35.6013} & \textbf{0.9182} & \textbf{0.0422} & \textbf{31.2435} & \textbf{0.8512}\\
\hline
\end{tabu}
\end{center}
\label{table:video}
\end{table}

\begin{figure*}[htp]
\setlength{\abovecaptionskip}{0pt}
\begin{center}
\subfigure[Original frame]{\includegraphics[height=1.1in,width=1.35in,angle=0]{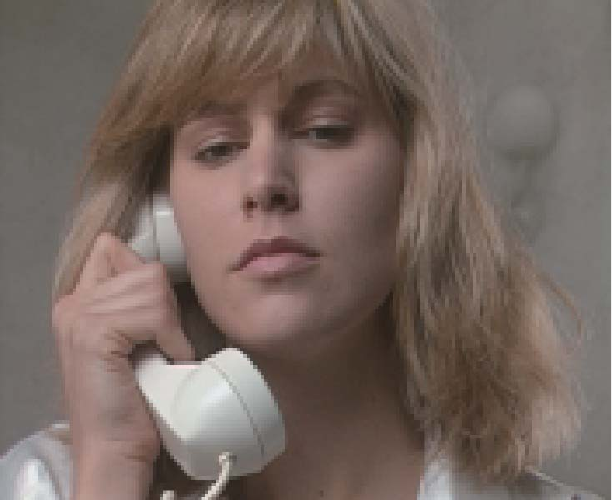}}
\hspace{-0.15cm}
\subfigure[Incomplete frame]{\includegraphics[height=1.1in,width=1.35in,angle=0]{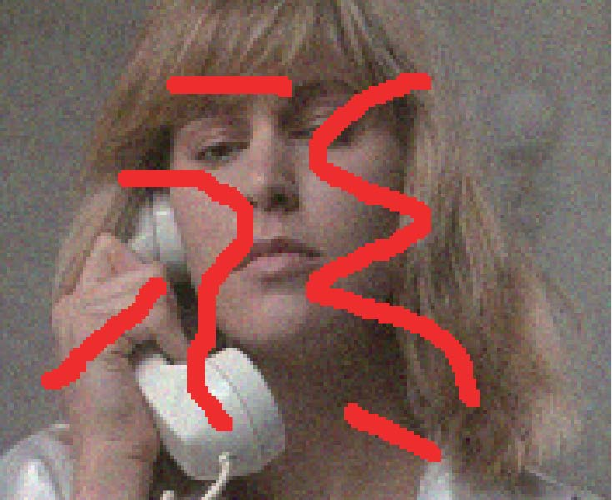}}
\hspace{-0.15cm}
\subfigure[FaLRTC~\cite{liu2013tensor}]{\includegraphics[height=1.1in,width=1.35in,angle=0]{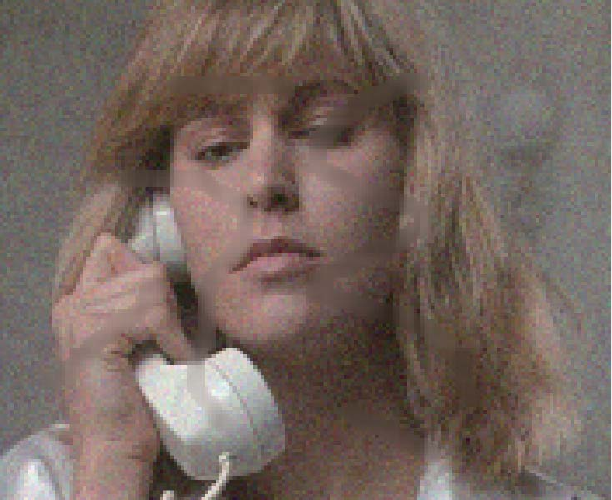}}
\hspace{-0.15cm}
\subfigure[HaLRTC~\cite{liu2013tensor}]{\includegraphics[height=1.1in,width=1.35in,angle=0]{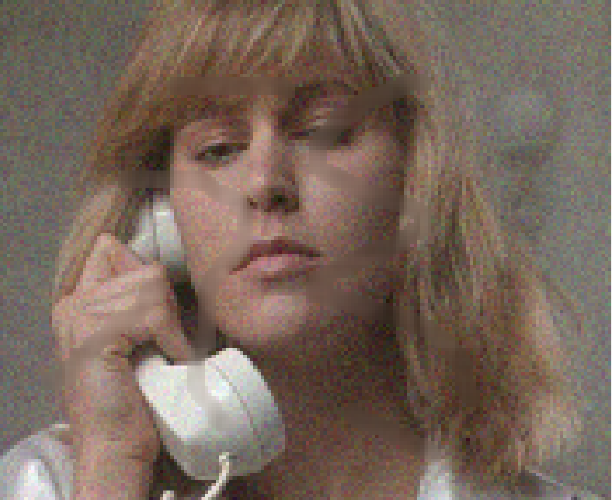}}
\hspace{-0.15cm}
\subfigure[RPTC$_{\rm{scad}}$~\cite{zhao2015novel}]{\includegraphics[height=1.1in,width=1.35in,angle=0]{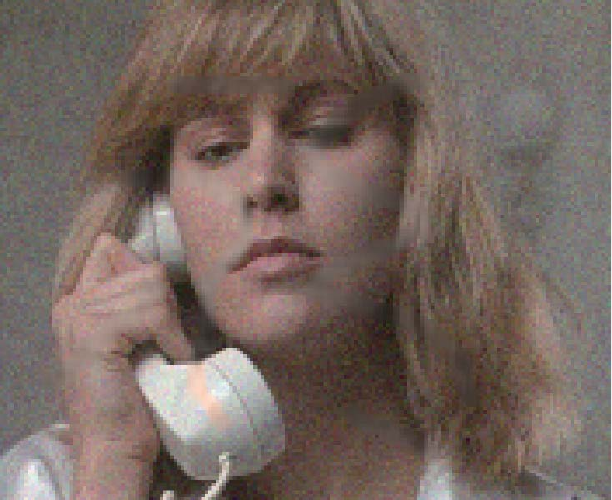}}
\\
\subfigure[TMac~\cite{xu2013parallel}]{\includegraphics[height=1.1in,width=1.35in,angle=0]{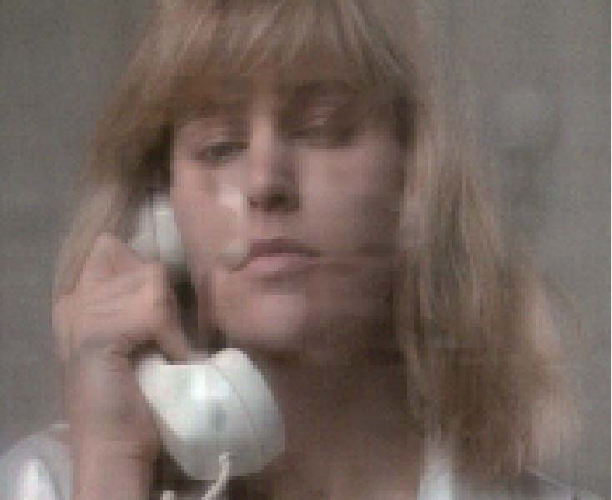}}
\hspace{-0.15cm}
\subfigure[STDC~\cite{chen2014simultaneous}]{\includegraphics[height=1.1in,width=1.35in,angle=0]{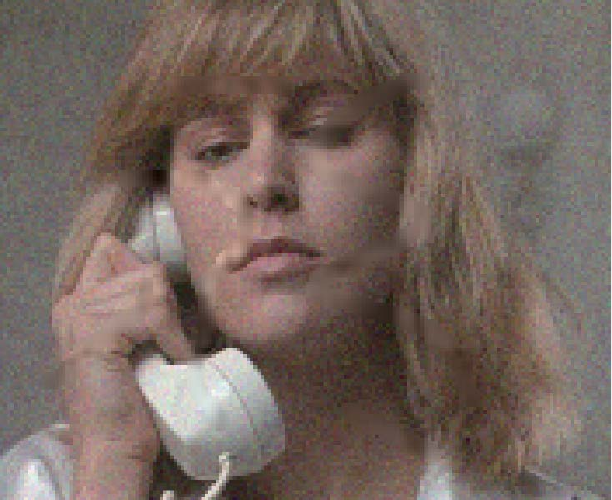}}\hspace{-0.15cm}
\subfigure[FBCP~\cite{zhao2015bayesian}]{\includegraphics[height=1.1in,width=1.35in,angle=0]{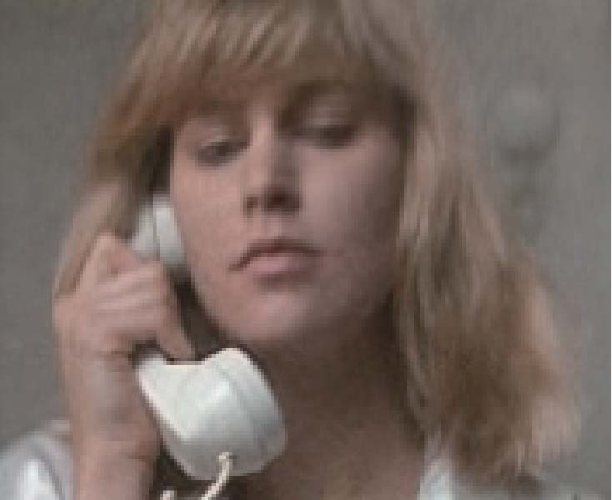}}
\hspace{-0.15cm}
\subfigure[BRTF~\cite{zhao2016bayesian}]{\includegraphics[height=1.1in,width=1.35in,angle=0]{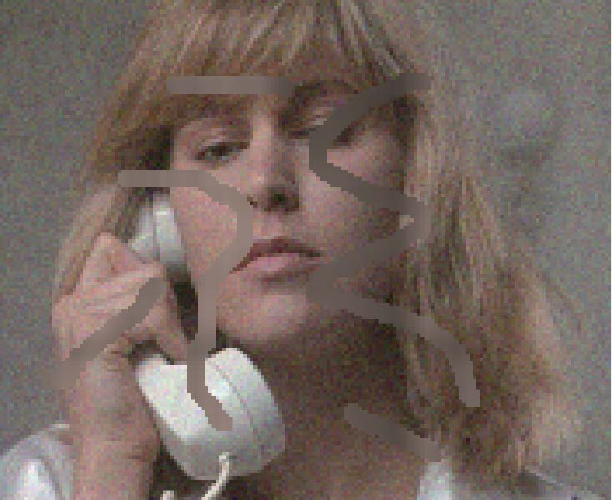}}
\hspace{-0.15cm}
\subfigure[Ours]{\includegraphics[height=1.1in,width=1.35in,angle=0]{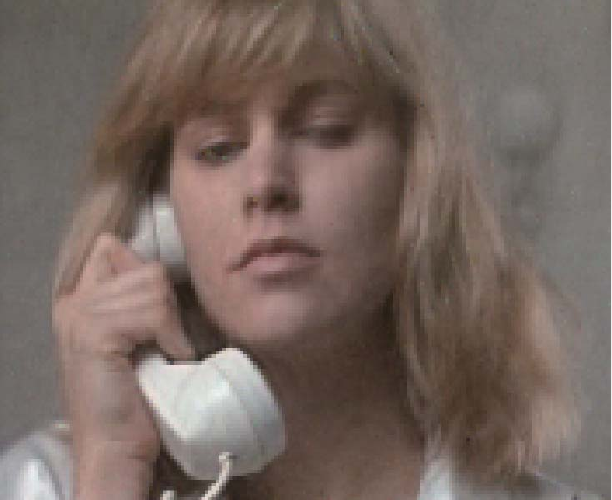}}\\
\end{center}
\vspace{-0.1cm}
\caption{Visual results of the $11$-th frame in the 'suzie' video.}
\vspace{-0.1cm}
\label{fig:suzie}
\end{figure*}

\begin{figure*}[htp]
\setlength{\abovecaptionskip}{0pt}
\begin{center}
\subfigure[Original frame]{\includegraphics[height=1.1in,width=1.35in,angle=0]{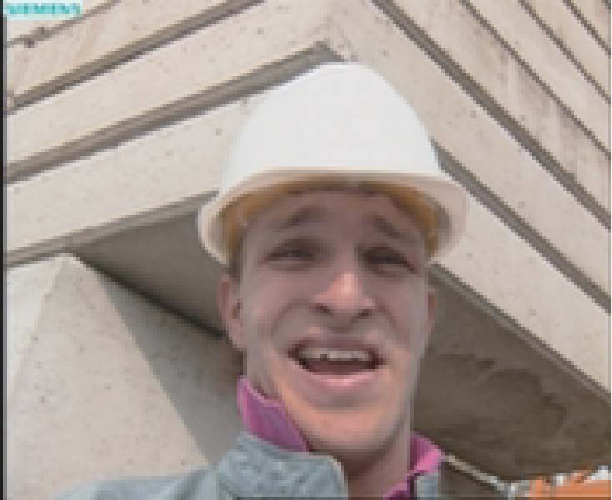}}
\hspace{-0.15cm}
\subfigure[Incomplete frame]{\includegraphics[height=1.1in,width=1.35in,angle=0]{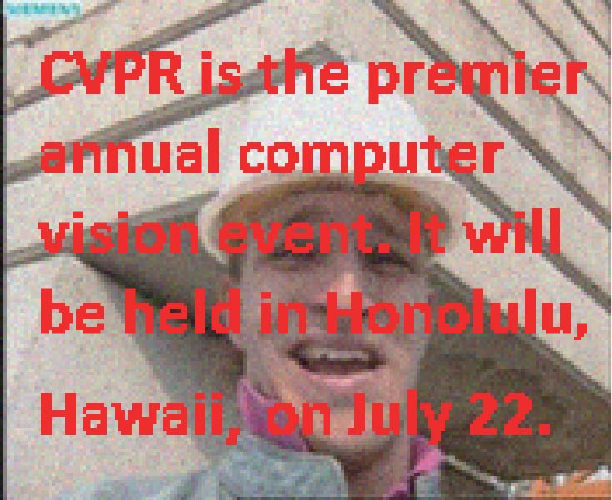}}
\hspace{-0.15cm}
\subfigure[FaLRTC~\cite{liu2013tensor}]{\includegraphics[height=1.1in,width=1.35in,angle=0]{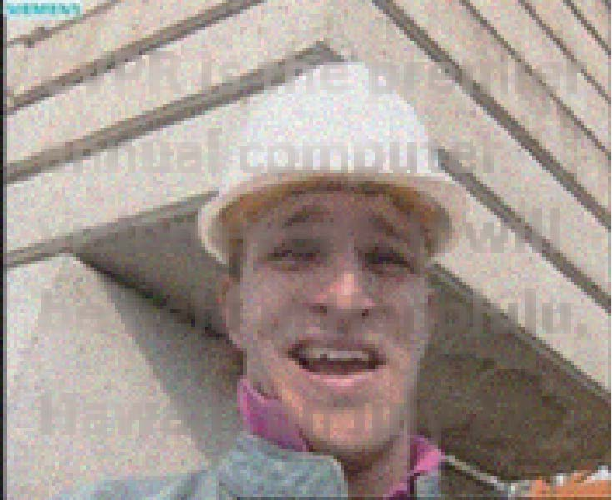}}
\hspace{-0.15cm}
\subfigure[HaLRTC~\cite{liu2013tensor}]{\includegraphics[height=1.1in,width=1.35in,angle=0]{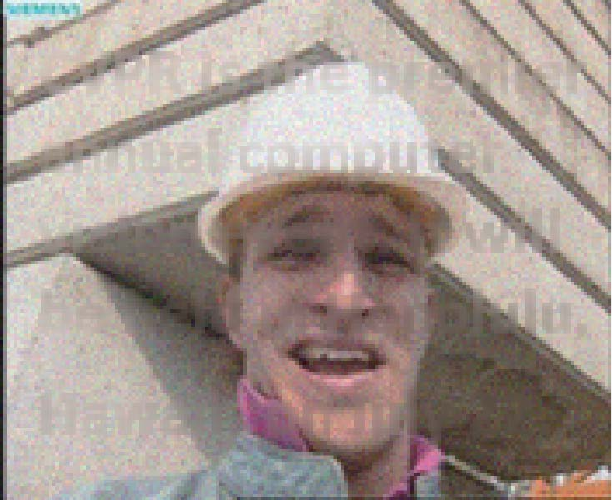}}
\hspace{-0.15cm}
\subfigure[RPTC$_{\rm{scad}}$~\cite{zhao2015novel}]{\includegraphics[height=1.1in,width=1.35in,angle=0]{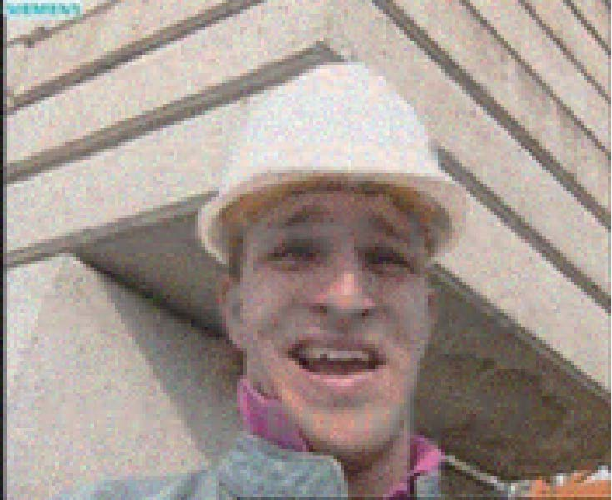}}
\\
\subfigure[TMac~\cite{xu2013parallel}]{\includegraphics[height=1.1in,width=1.35in,angle=0]{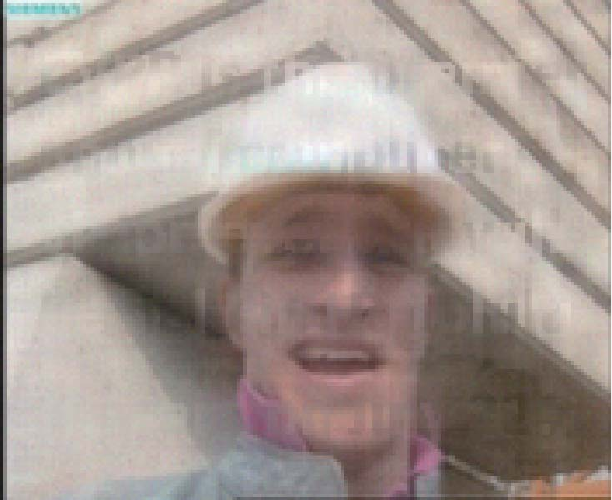}}
\hspace{-0.15cm}
\subfigure[STDC~\cite{chen2014simultaneous}]{\includegraphics[height=1.1in,width=1.35in,angle=0]{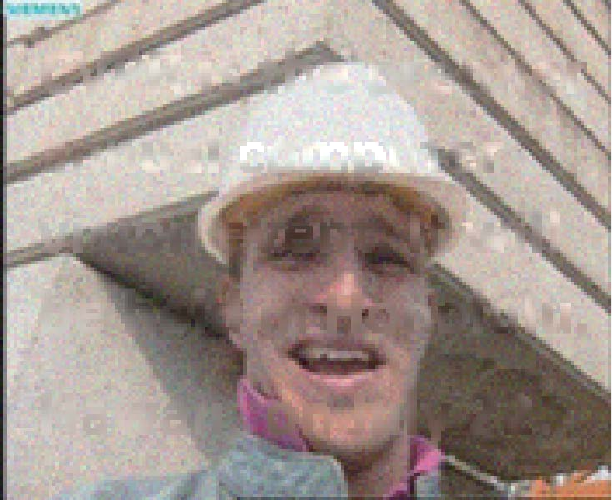}}
\hspace{-0.15cm}
\subfigure[FBCP~\cite{zhao2015bayesian}]{\includegraphics[height=1.1in,width=1.35in,angle=0]{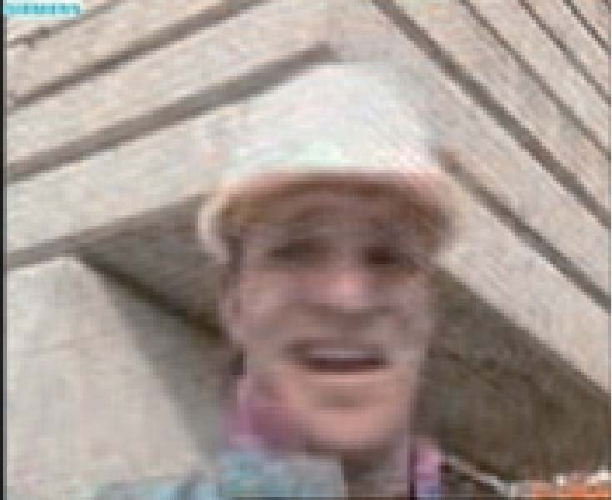}}
\hspace{-0.15cm}
\subfigure[BRTF~\cite{zhao2016bayesian}]{\includegraphics[height=1.1in,width=1.35in,angle=0]{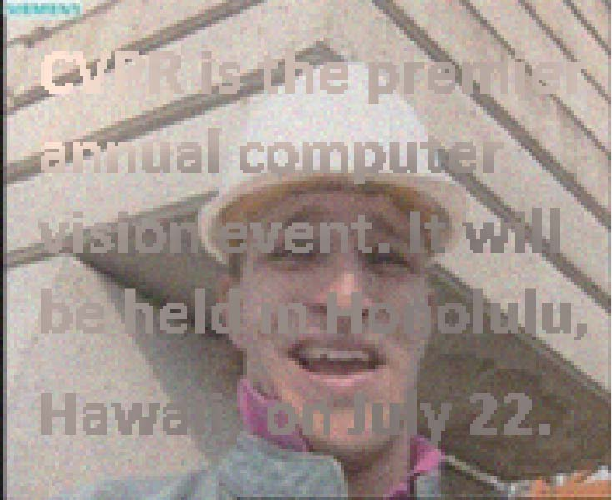}}
\hspace{-0.15cm}
\subfigure[Ours]{\includegraphics[height=1.1in,width=1.35in,angle=0]{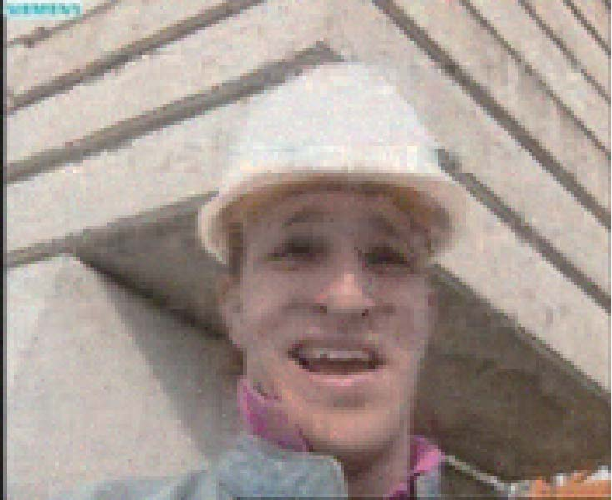}}\\
\end{center}
\vspace{-0.1cm}
\caption{Visual results of the $19$-th frame in the 'foreman' video.}
\vspace{-0.1cm}
\label{fig:foreman}
\end{figure*}

\subsection{Image inpainting}
In this subsection, we test the proposed model on 10 benchmark images (see Figure~\ref{fig:inpainting}) for image inpainting. Each image is of size $256\times 256 \times 3$ and rescaled into $[0,1]$. We generate the incomplete observation by randomly selecting a certain percentage of missing entries in each image. Given the incomplete observation, all methods are utilized to recover the latent image. Besides RRE, peak signal-to-noise ratio (PSNR) and structural similarity (SSIM), two conventional image quality indices are also adopted to measure the recovery accuracy. The average RRE, PSNR and SSIM on 10 benchmark images are given in Table~\ref{table:avginpainting}. It can be seen that the proposed model performs better than other methods in most cases. To make this clear, we show the visual comparison of the recovery results of the 'facade' image in Figure~\ref{fig:facade}, when missing ratio is $90\%$ as well as the 'parrot' image with missing ratio $=70\%$ in Figure~\ref{fig:parrot}. In each image, an interest area is zoomed for details comparison. Due to 'facade' is highly structured, most of the methods can recover the major structure of the image with even $90\%$ missing entries. However, most of the competitors fail to well recover the fine details, e.g., the zoomed areas in Figure~\ref{fig:parrot} (b)-(i). In contrast, these details are well recovered in the results of the proposed model shown as the zoomed area in Figure~\ref{fig:parrot} (j). The reasons come from two aspects. First, we propose an appropriate sparsity induced low-rank model for tensor data, which is totally different from the matrix rank norm induced low rank model in FaLRTC, HaLRTC and TMac. Second, we adopt a flexible mixture of Gaussians to model the complex non-low-rank structure, which, however, is not considered in other competitors. These can be further validated by the results on 'Parrot' image in Figure~\ref{fig:parrot} which contains complex non-low-rank structure according to Figure~\ref{fig:dis}. It can be seen that the proposed model give more clear and natural results than others, especially on the zoomed details.

\begin{table}\footnotesize
\caption{RRE on CMU-PIE dataset with different missing ratios.}
\renewcommand{\arraystretch}{1.1}
\begin{center}
\begin{tabu} to 0.45\textwidth {X[1.6,l]|X[c]|X[c]|X[c]|X[c]}
\hline
Method & $60\%$ & $70\%$ & $80\%$ & $90\%$\\
\hline
FaLRTC~\cite{liu2013tensor} & 0.3741 & 0.5058 & 0.7392 & 0.9103\\
HaLRTC~\cite{liu2013tensor} & 0.3694 & 0.5021 & 0.7388 & 0.9086\\
RPTC$_{\rm{scad}}$~\cite{zhao2015novel} & 0.2228 & 0.2599 & 0.3284 & 0.4166\\
TMac~\cite{xu2013parallel} & 0.2629 & 0.4226 & 0.6574 & 0.9385\\
STDC~\cite{chen2014simultaneous} & 0.2602 & 0.2606 & 0.3316 & 0.5154\\
FBCP~\cite{zhao2015bayesian} & 0.1410 & 0.1911 & 0.2590 & 0.3746\\
Ours & \textbf{0.1172} & \textbf{0.1560} & \textbf{0.2274} & \textbf{0.3372}\\
\hline
\end{tabu}
\end{center}
\label{table:face}
\end{table}

\begin{figure*}
\begin{center}
\subfigure[Ground truth]{\includegraphics[height=1.6in,width=1.6in,angle=0]{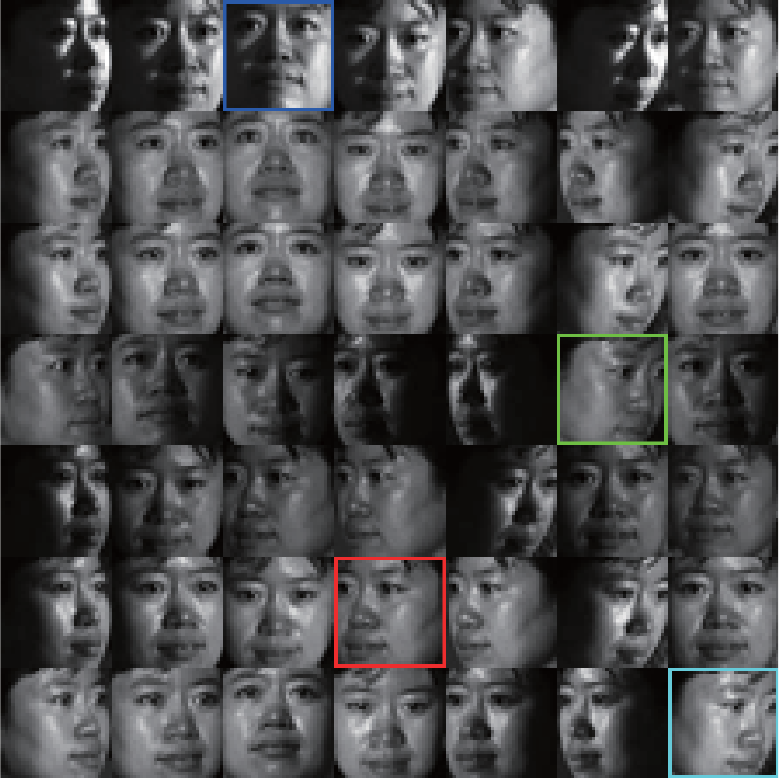}}
\hspace{-0.15cm}
\subfigure[FaLRTC~\cite{liu2013tensor}]{\includegraphics[height=1.6in,width=1.6in,angle=0]{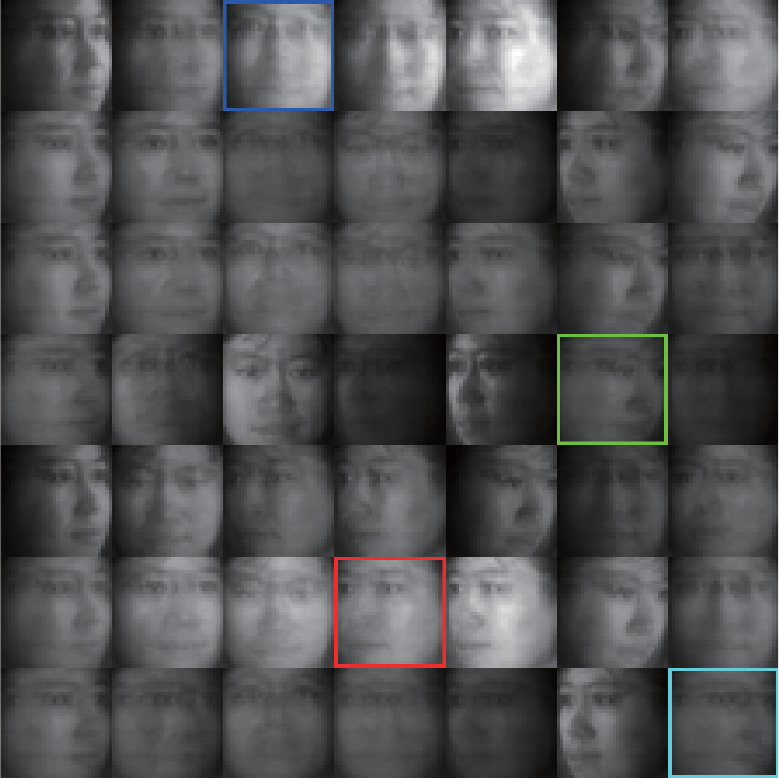}}
\hspace{-0.15cm}
\subfigure[HaLRTC~\cite{liu2013tensor}]{\includegraphics[height=1.6in,width=1.6in,angle=0]{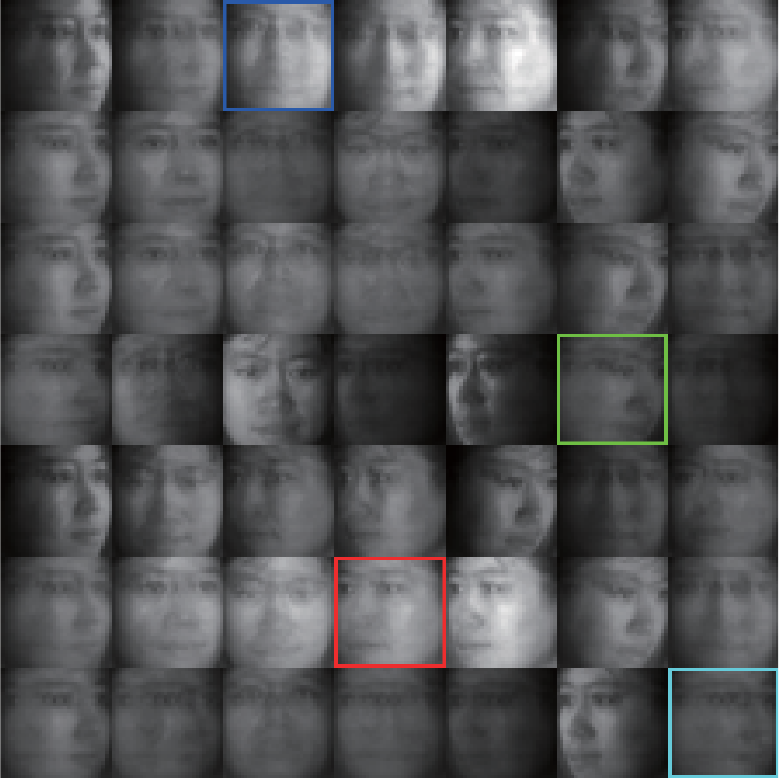}}
\hspace{-0.15cm}
\subfigure[RPTC$_{\rm{scad}}$~\cite{zhao2015novel}]{\includegraphics[height=1.6in,width=1.6in,angle=0]{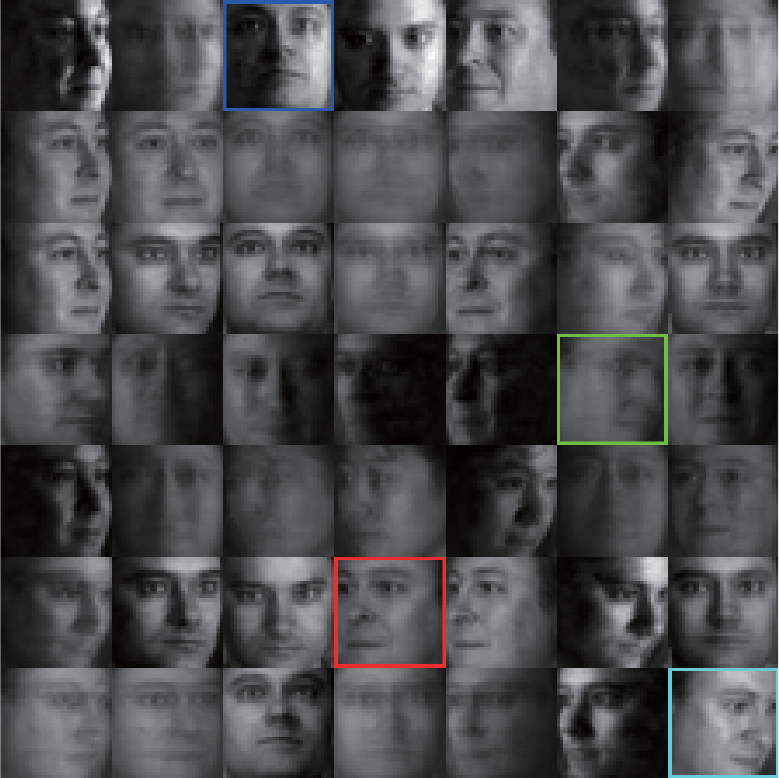}}
\\
\subfigure[TMac~\cite{xu2013parallel}]{\includegraphics[height=1.6in,width=1.6in,angle=0]{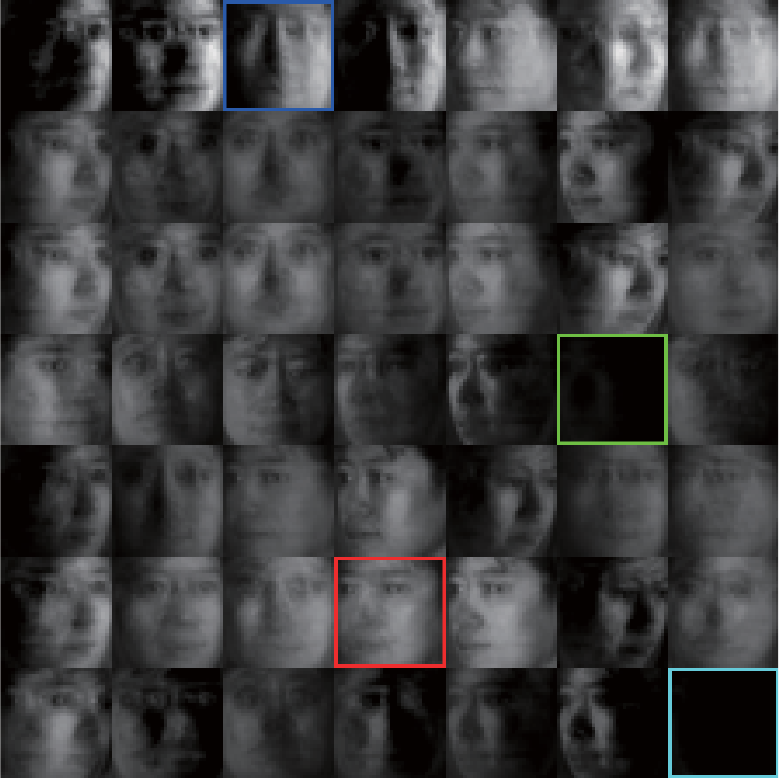}}
\hspace{-0.15cm}
\subfigure[STDC~\cite{chen2014simultaneous}]{\includegraphics[height=1.6in,width=1.6in,angle=0]{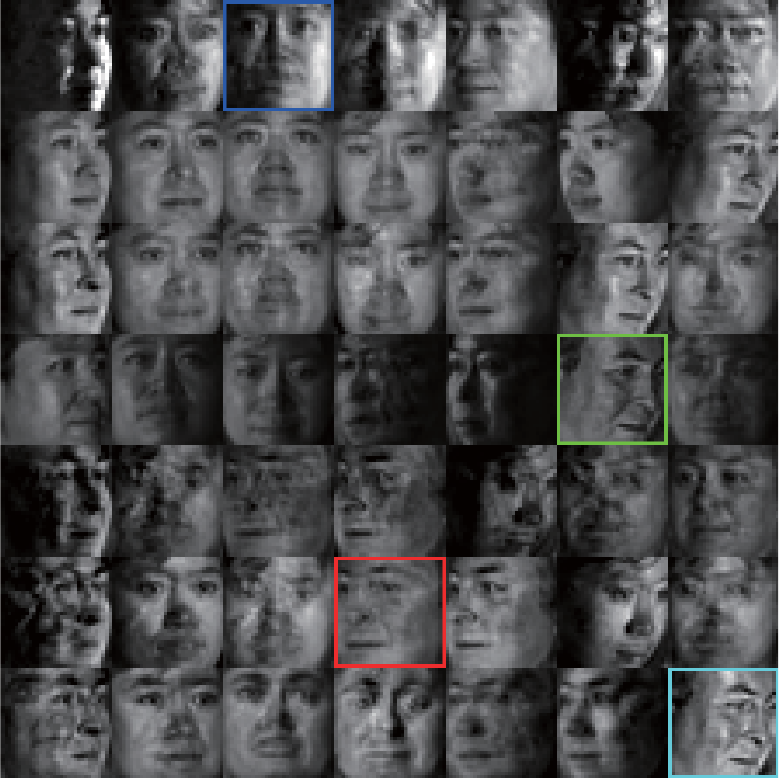}}
\hspace{-0.15cm}
\subfigure[FBCP~\cite{zhao2015bayesian}]{\includegraphics[height=1.6in,width=1.6in,angle=0]{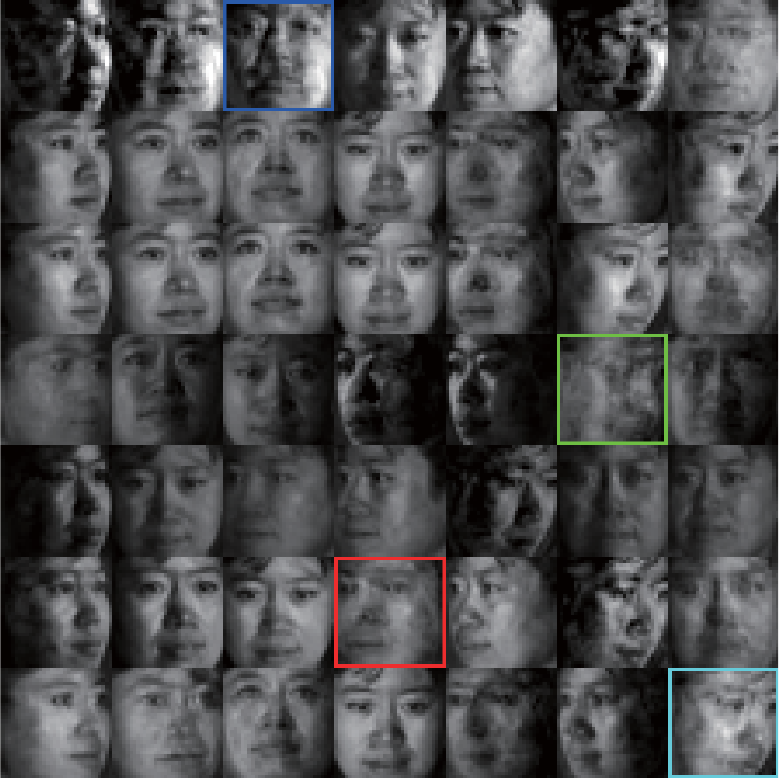}}
\hspace{-0.15cm}
\subfigure[Ours]{\includegraphics[height=1.6in,width=1.6in,angle=0]{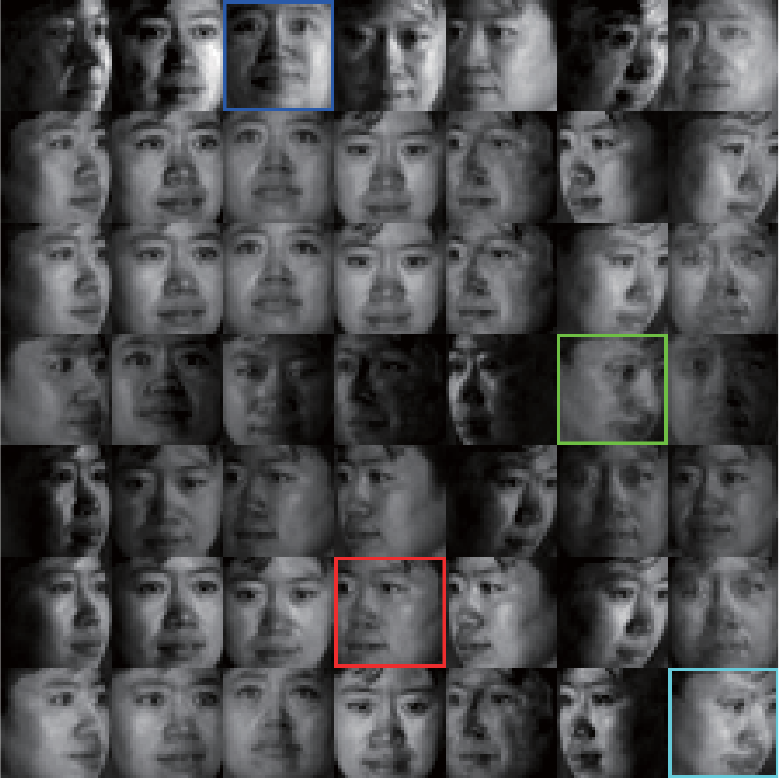}}\\
\end{center}
\caption{Visual results on the $49$ selected missing faces, when missing ratio is $70\%$.}
\label{fig:faceimage}
\end{figure*}

\subsection{Video completion}
We further evaluate the proposed model on two famous videos, 'suzie' and 'foreman'\footnote{https://media.xiph.org/video/derf/} for video completion. In each video, $20$ consecutive frames are extracted as the experimental data of size $144 \times 176 \times 3 \times 20$, which is then rescaled into $[0,1]$. The incomplete observation for 'susie' is a scrabbled version\footnote{For every two consecutive frames in the video, there is one corrupted with the same scrabbles or superimposed texts.} shown in Figure~\ref{fig:suzie} (b), while the incomplete observation for 'foreman' is corrupted by superimposed texts{\color{red}{\footnotemark[3]}} shown as Figure~\ref{fig:foreman} (b). The observations are also corrupted by noise sampled from $\mathcal{N}(0,0.001)$. Due to t-SVD cannot be applied to 4-mode tensor~\cite{zhang2014novel}, here we only employ other $8$ methods to recover the latent video. The RRE, PSNR and SSIM on two videos are give in Table~\ref{table:video}. The proposed model gives higher PSNR, SSIM and lower RRE than others, which demonstrates the superiority of the proposed method over others in video completion. To further clarify this, visual comparison results of two corrupted frames are shown in Figure~\ref{fig:suzie},~\ref{fig:foreman}. Compared with others, the proposed model produces sharper and more clear results.

\subsection{Facial image synthesis}
Motivated by the fact that a complete training set is often not available in real face recognition applications~\cite{geng2011face}, facial image synthesis is studied for recovering missing faces in an incomplete dataset. In this subsection, we use the CMU-PIE face dataset~\cite{sim2003cmu} for facial image synthesis. Specifically, we select images from the first $3$ subjects with $11$ positions and $21$ illumination changes to construct a $4$-mode tensor $\mathcal{L} \in{\mathbb{R}^{3\times 11 \times 21 \times 1024}}$, which is then rescaled into $[0,1]$. For the incomplete observation, a certain percentage of missing faces are randomly selected. Given the observation, all methods except t-SVD and BRTF~\footnote{BRTF always fails on this dataset as rank becomes 0.} are employed to recover $\mathcal{L}$. Table~\ref{table:face} gives RRE for each method. It can be seen that the proposed model gives the lowest RRE with different missing ratios. When the missing ratio is $70\%$, visual comparison results for $49$ selected missing faces generated from all methods are shown in Figure~\ref{fig:faceimage}. Compared with other competitors, the results of the proposed model recovers more details and produces less artifacts.

\begin{figure*}
\begin{center}
\subfigure[RRE]{\includegraphics[height=0.9in,width=1.15in,angle=0]{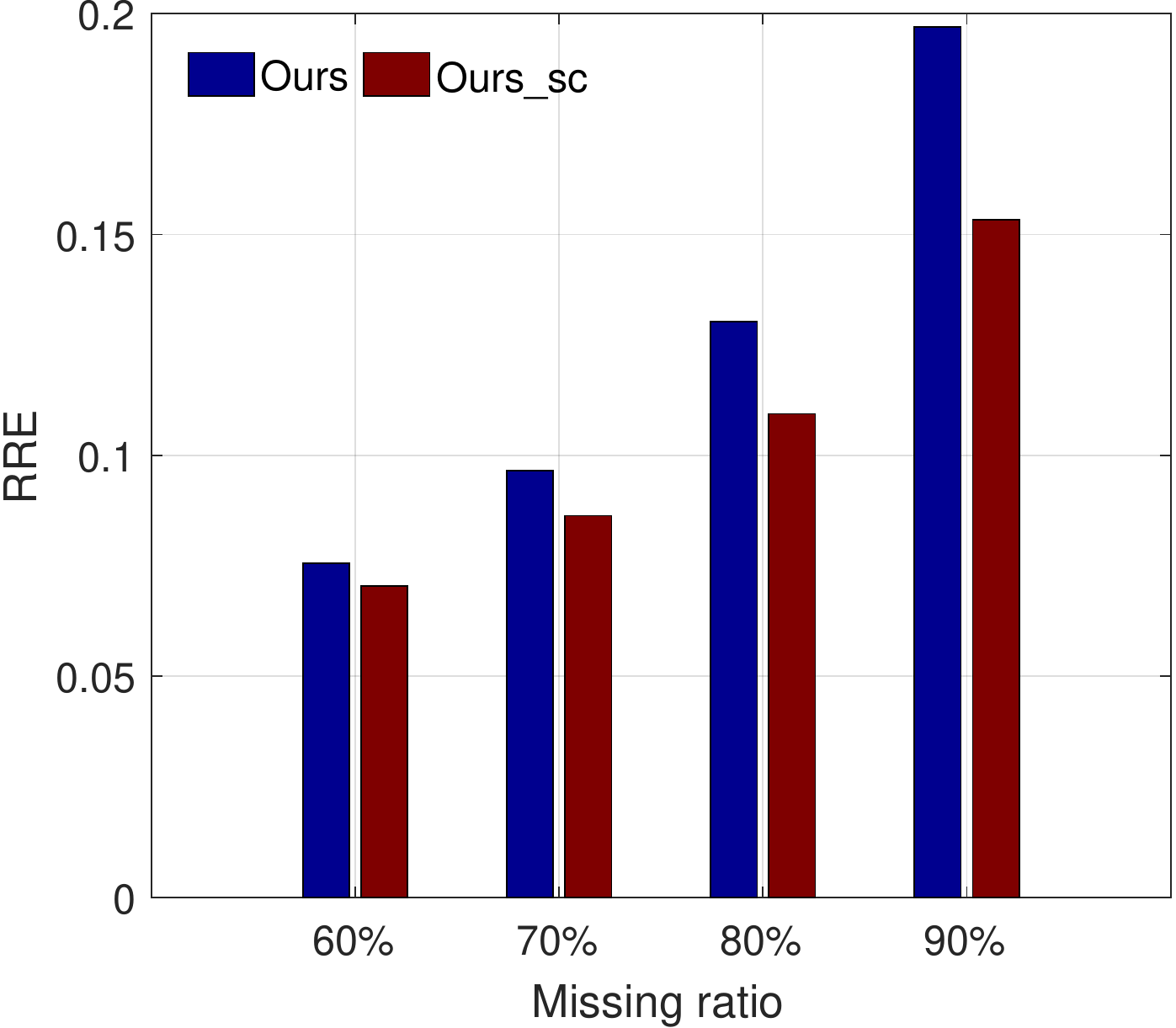}}
\hspace{-0.15cm}
\subfigure[PSNR]{\includegraphics[height=0.9in,width=1.15in,angle=0]{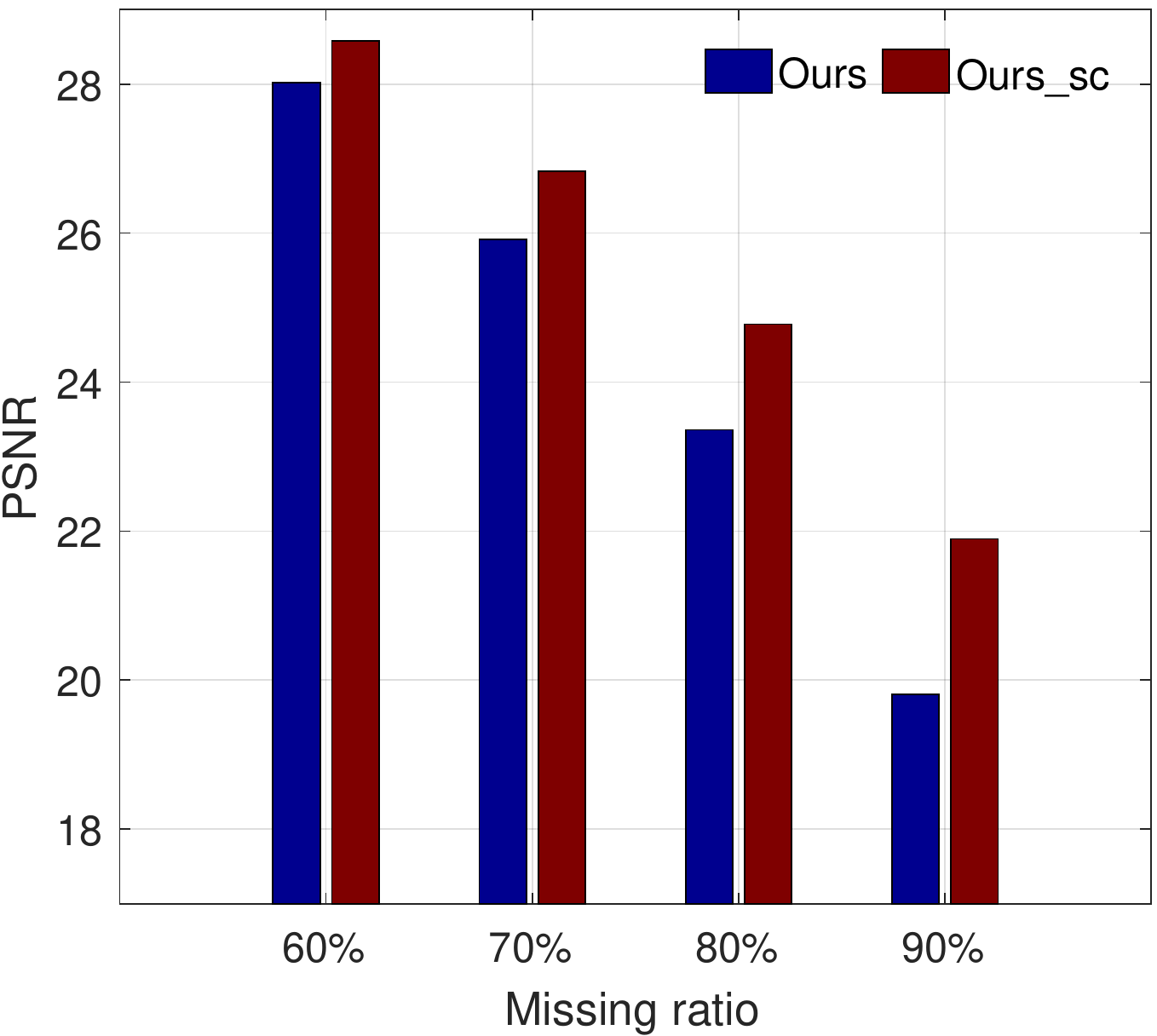}}
\hspace{-0.15cm}
\subfigure[SSIM]{\includegraphics[height=0.9in,width=1.15in,angle=0]{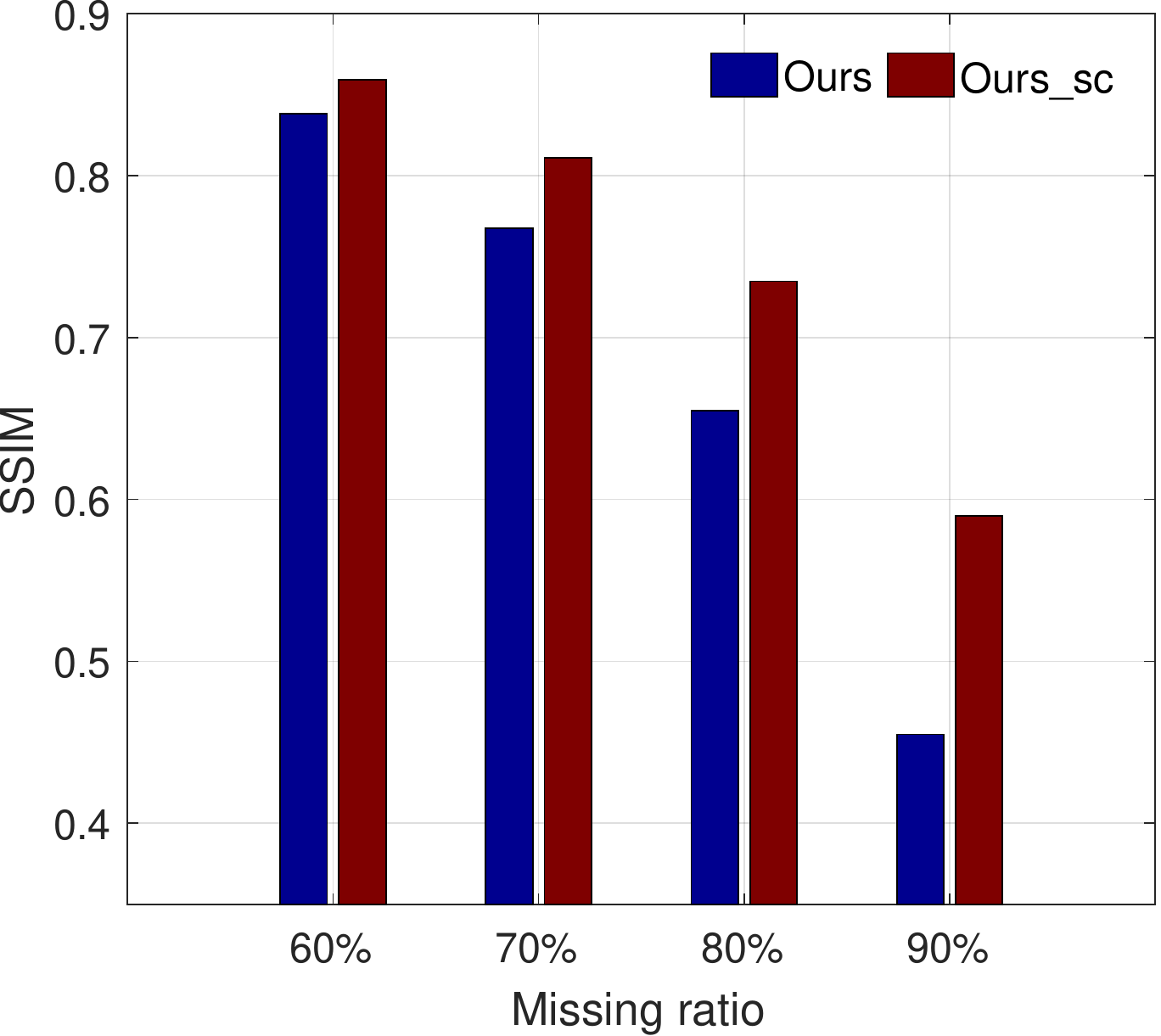}}
\subfigure[RRE]{\includegraphics[height=0.9in,width=1.15in,angle=0]{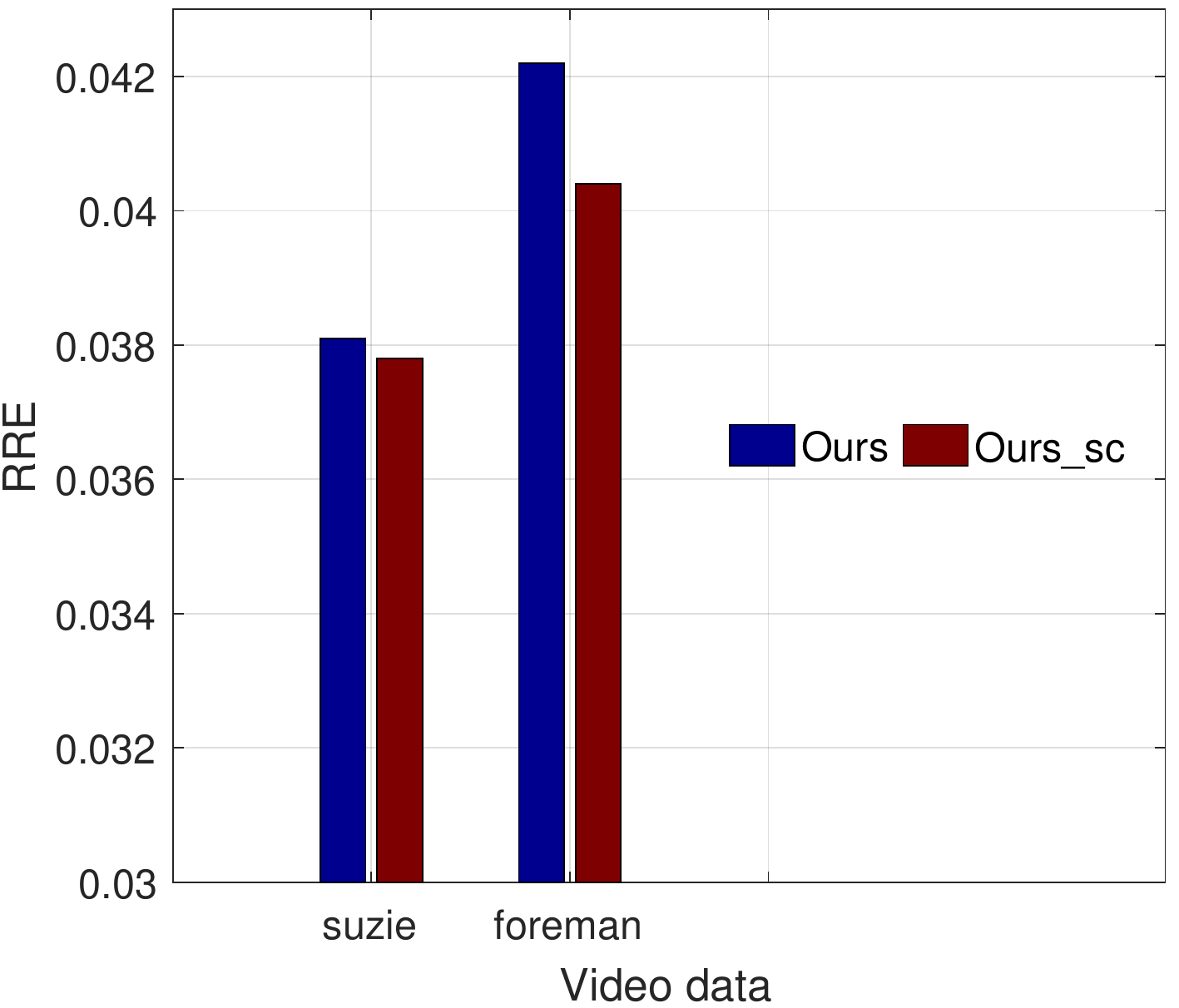}}
\hspace{-0.15cm}
\subfigure[PSNR]{\includegraphics[height=0.9in,width=1.15in,angle=0]{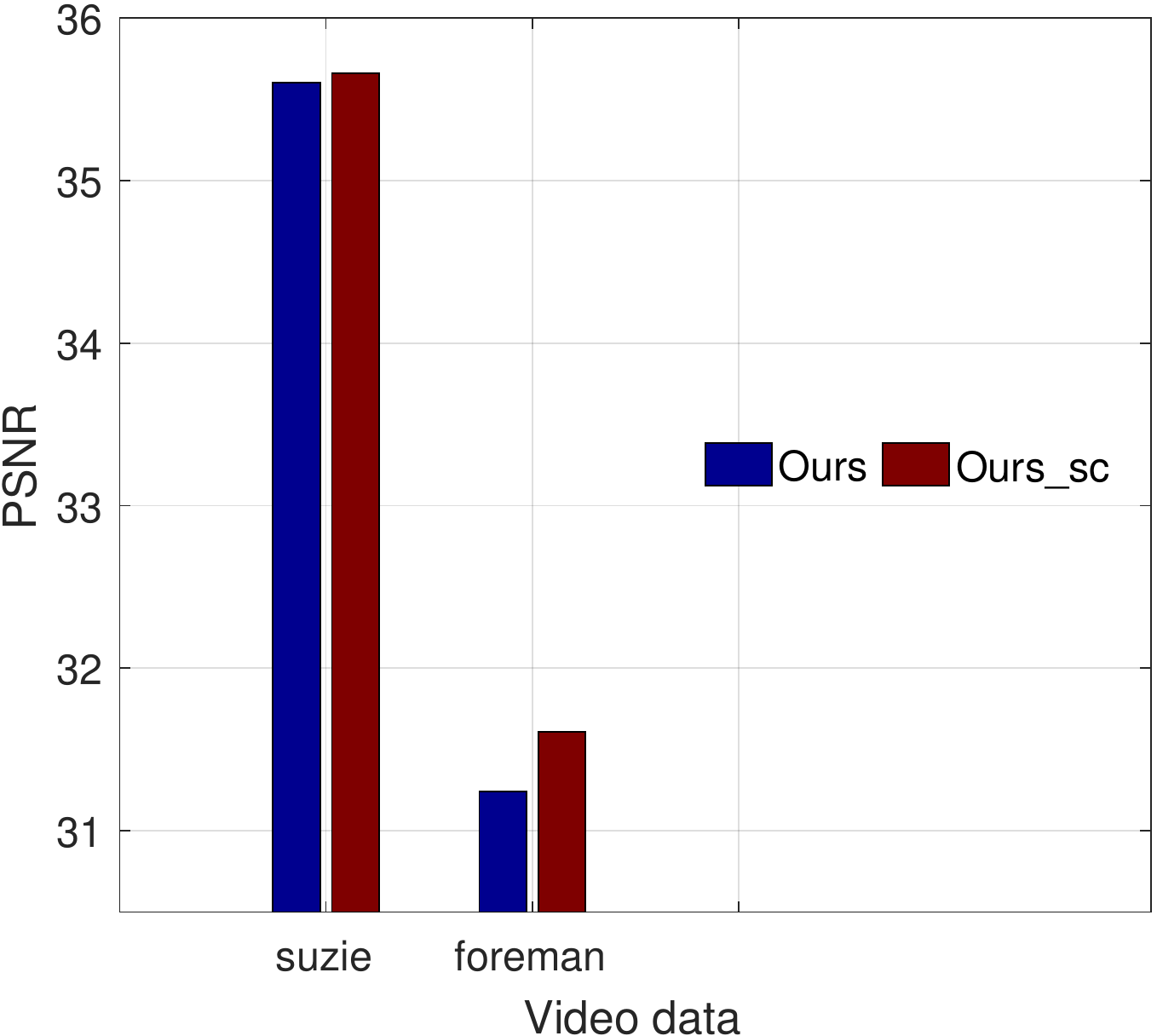}}
\hspace{-0.15cm}
\subfigure[SSIM]{\includegraphics[height=0.9in,width=1.15in,angle=0]{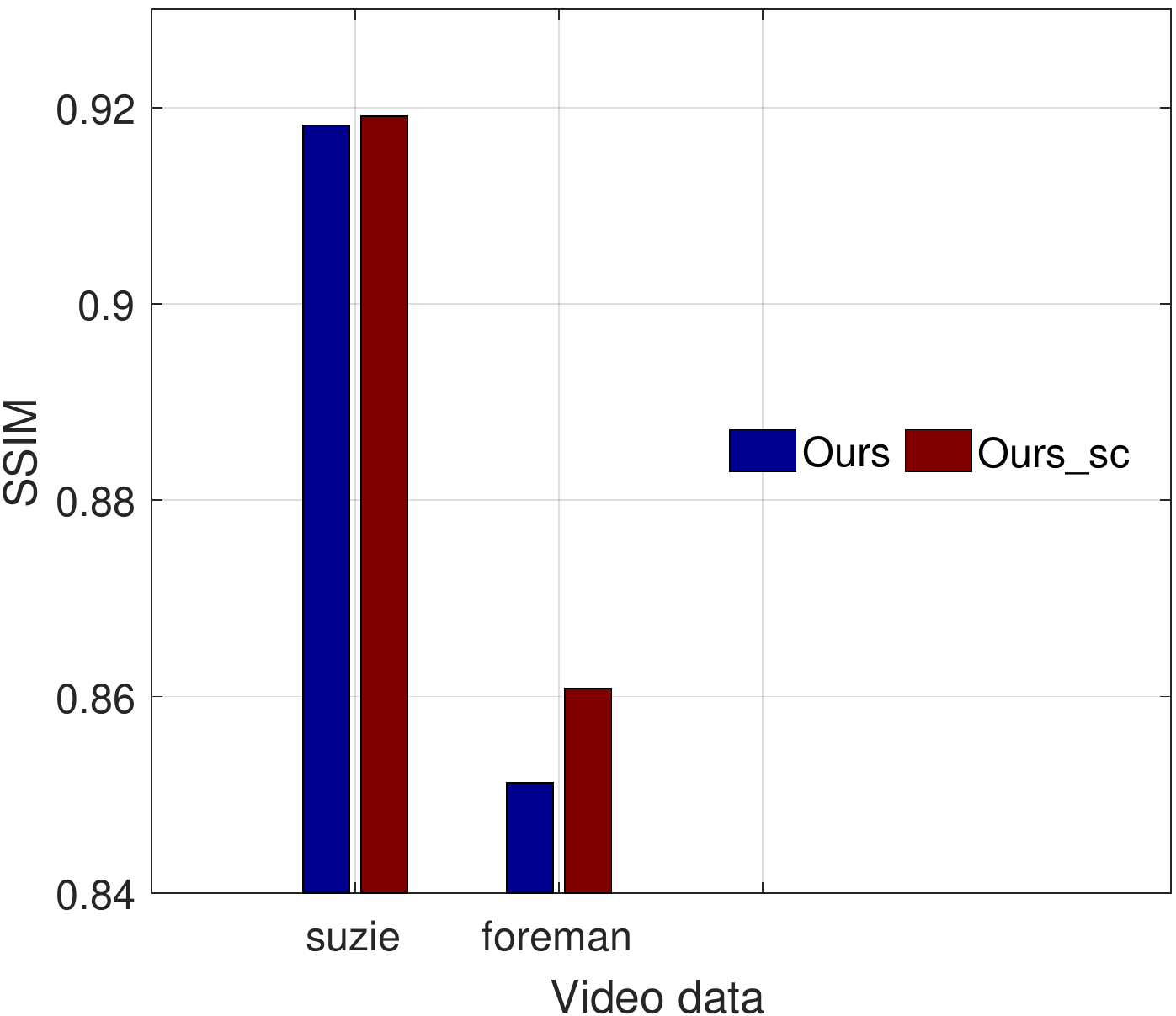}}
\end{center}
\caption{Performance comparison between the proposed method and its variant that considers the spatial similarity in two applications. (a)-(c) denote the results in image inpainting, and (d)-(f) are the results in video completion.}
\label{fig:SpatialC}
\end{figure*}

\section{Discussion}
In this section, we will further discuss the spatial similarity constraint in visual tensors. Specifically, visual tensor data (e.g., image or video) often shows similarity in spatial domain, which can provide extra prior information for the ill-posed completion task, thus improves the recovery accuracy especially when the missing ratio is high. Such kind of similarity can be modelled by exploiting the correlation among different rows of the factor matrices $\matb{U}^{(k)}$. For example, we consider a color image as $3$-mode tensor, the correlation among different rows in each $\matb{U}^{(k)}$ represent the correlation among rows, columns and channels of the image, respectively. Considering the obviously local similarity in visual data, we assume neighbouring rows in each $\matb{U}^{(k)}$ to be similar and introduce the following prior distribution
\begin{equation}\label{eq:eq50}
\begin{aligned}
p(u^{(k)}_{ir}) \sim \mathcal{N}(u^{(k)}_{ir}|\mu^{(k)}, \tau^{(k)-1})\mathcal{N}(u^{(k)}_{ir}|\matb{w}^T_i\vecb{u}^{(k)}_{r},\eta^{-1}_0)
\end{aligned}
\end{equation}
where $\vecb{w}_i = [w_{1i},...,w_{n_ki}]^T$ is a weight vector with $\sum\nolimits^{n_k}_{j=1}w_{ji} = 1$ and $w_{ii} = 0$. This prior suggests each row of $\matb{U}^{(k)}$ not only can be reconstructed by other rows, but also complies with the $\ell_2$ norm constraint on each entry. To this end, a large enough $\eta_0$ is adopted, e.g., $10^3$. In addition, we set $w_{ji} = \exp(- 2 * \rho * |i - j|^2)$ to implies the major contribution of neighbouring rows in the reconstruction. The parameter $\rho = Nz / N$, where $N_z$ denotes the number of observed entries in $\mathcal{Y}$, thus $1 - \rho$ equals the missing ratio. With this prior, $u^{(k)}_{ir}$ thus can be drawn from a Gaussian distribution $\mathcal{N}(\tilde{\mu}_{u^{(k)}_{ir}}, \tilde{\tau}_{u^{(k)}_{ir}})$ as Eq.~\eqref{eq:eq23} with parameters
\begin{equation}\label{eq:eq51}
\begin{aligned}
&\tilde{\tau}_{u^{(k)}_{ir}} = \sum\nolimits_{{\mbox{\boldmath{$i$}}}: i_k = i}\tau_0{o_{\mbox{\boldmath{$i$}}}}\tilde{c}^{rk2}_{\mbox{\boldmath{$i$}}}+ \tau^{(k)} + \eta_0,\\
&\tilde{\mu}_{u^{(k)}_{ir}} = \tau^{-1}_{u^{(k)}_{ir}}\left(\sum\nolimits_{{\mbox{\boldmath{$i$}}}: i_k = i}\tau_0{o_{\mbox{\boldmath{$i$}}}}\tilde{y}^r_{\mbox{\boldmath{$i$}}}\tilde{c}^{rk}_{\mbox{\boldmath{$i$}}} + \tau^{(k)}\mu^{(k)} + \eta_0\matb{w}^T_i\vecb{u}^{(k)}_{r}\right).
\end{aligned}
\end{equation}

To demonstrate the effectiveness of the prior in Eq.~\eqref{eq:eq50}, we compare the proposed model with its variant 'Ours\_sc' where  $u^{(k)}_{ir}$ is sampled with Eq.~\eqref{eq:eq51}, in above image inpainting and video completion applications. The comparison of their numerical results is shown in bar charts in Figure~\ref{fig:SpatialC}. It can be seen that 'Ours\_sc' surpasses the proposed model in both applications. Moreover, the superiority is more obvious when the missing ratio is high. This is intuitively because high missing ratio results in a worsened ill-posed problem which requires more prior information to regularize the infinite solution space.

\section{Conclusion}
We have presented a novel data-adaptive tensor completion model. The model explicitly decomposes the latent tensor into the low-rank structure and the non-low-rank one. The low-rank prior rests upon a new definition of tensor rank, which forms the basis of automatic tensor rank determination with exploiting sparsity in CP factorization from an incomplete set of observations. The prior for the non-low-rank structure is a mixture of Gaussians which has shown to be flexible enough to reflect a variety of real tensor data. These two priors allow the development of an MMSE method to estimate not only the posterior mean of missing entries, but also their uncertainty using Gibbs sampling. In addition, the proposed model has been shown to outperform its competitors in terms of tensor recovery.


\ifCLASSOPTIONcaptionsoff
  \newpage
\fi



\bibliographystyle{IEEEtran}
\bibliography{egbib}
\end{document}